\documentclass{article}

\PassOptionsToPackage{numbers, compress}{natbib}

\usepackage[preprint]{arxiv}

\usepackage[utf8]{inputenc} \usepackage[T1]{fontenc}    \usepackage{hyperref}       \usepackage{url}            \usepackage{booktabs}       \usepackage{amsfonts}       \usepackage{nicefrac}       \usepackage{microtype}      \usepackage{xcolor}
\usepackage{float}
\usepackage{lipsum,xcolor}
\usepackage{enumitem}
\usepackage{amsmath,amsthm,amssymb,amsfonts,mathtools}
\usepackage{todonotes}
\usepackage{algorithmic}
\usepackage{hyperref}
\usepackage{graphicx}
\usepackage{subfigure}
\usepackage[ruled,vlined, algo2e, resetcount]{algorithm2e}
\usepackage{wrapfig}

\newcommand{\shubhada}[1]{\textcolor{orange}{Shubhada: #1}}

\newcommand{\arun}[1]{\textcolor{teal}{Arun: #1}}
\newcommand{\KS}[1]{\textcolor{blue}{KS: #1}}

\newcommand{\lrset}[1]{\left\{{#1}\right\}}
\newcommand{\lrp}[1]{\left({#1}\right)}
\newcommand{\lrs}[1]{\left[{#1}\right]}
\newcommand{\E}{\mathbb{E}}
\newcommand{\KL}{\operatorname{KL}}
\newcommand{\argmax}{\arg\!\max}
\newcommand{\KLinfL}{\operatorname{KL^{L}_{inf}}}
\newcommand{\KLinfU}{\operatorname{KL^{U}_{inf}}}

\newtheorem{theorem}{Theorem}[section]
\newtheorem{lemma}[theorem]{Lemma}

\newtheorem{proposition}[theorem]{Proposition}

\newtheorem{definition}[theorem]{Definition}

\title{Optimal Best-Arm Identification in Bandits with Access to Offline Data}

\author{Shubhada Agrawal\thanks{A significant part of this work was done as a student researcher at Google Research, India, when the author was a PhD student at TIFR Mumbai, India.} \\
  Georgia Institute of Technology, USA\\
  \texttt{shubhadaiitd@gmail.com} \\
\And
  Sandeep Juneja \\
  TIFR Mumbai, India \\
  Visiting Researcher, Google Research, India\\
\texttt{juneja@tifr.res.in} \\
  \AND
  Karthikeyan Shanmugam\\
  Google Research, India \\
\texttt{karthikeyanvs@google.com} \\
  \And
  Arun Sai Suggala\\
  Google Research, India \\
\texttt{arunss@google.com} \\
}

\usepackage{setspace}
\begin{document}
\maketitle
\begin{abstract}
Learning paradigms based purely on offline data as well as those based solely on sequential online learning have been well-studied in the literature. In this paper, we consider combining offline data with online learning, an area less studied but of obvious practical importance. We consider the stochastic $K$-armed bandit problem, where our goal is to identify the arm with the highest mean in the presence of relevant offline data, with confidence $1-\delta$. We conduct a lower bound analysis on policies that provide such $1-\delta$ probabilistic correctness guarantees. We develop algorithms that match the lower bound on sample complexity when $\delta$ is small. Our algorithms are computationally efficient with an average per-sample acquisition cost of $\tilde{O}(K)$, and rely on a careful characterization of the optimality conditions of the lower bound problem.

\end{abstract}

\section{Introduction}

Bandit optimization (BO) and reinforcement learning (RL) are general frameworks for sequential decision-making when the dynamics of the underlying environment are a priori unknown and have shown great empirical success in areas including games and recommendation systems~\cite{li2010contextual, silver2017mastering, lattimore2020bandit}. Despite their success, they have found limited applications in critical areas such as healthcare and robotics because repeated interactions with the environment is often expensive. To overcome this drawback, recent works have developed the frameworks of offline RL that totally avoid online interactions with the environment~\cite{levine2020offline, nguyen2021offline}. They instead rely on offline data collected in the past to learn good policies due to access to significant amounts of offline data. However, a major drawback is that naively using offline data may be bottle necked by quality of offline data. For example, if the data is  from a less exploratory policy, then the learned policies have poor performance~\cite{liu2020provably}. 

In this work, we consider a learning paradigm which is in between the purely online and purely offline learning paradigms. To be precise, we consider sequential decision-making problems where the learner can use both online interactions and offline data to come up with good policies.  This paradigm has the potential to achieve the best of both worlds: if the quantity of offline data is good and is representative of the current environment, then the learner can learn good policies with zero or few online interactions. On the other hand, if the offline data is stale, distributionally unrepresentative, inadequate or of poor quality, then the learner can still learn good policies but with more online samples. This combined offline-online (o-o) paradigm has received little attention in the literature. A few recent works have considered this, but as detailed in a later section, they are mostly empirical in nature and do not provide (instance) optimal algorithms as we do in our work. 

We focus on the problem of best-arm identification (BAI) for stochastic multi-armed bandits (MAB) in the offline-online (o-o) paradigm. This is a simple yet powerful setting that captures several practically-important scenarios arising in fields such as healthcare and recommendation systems. In this problem, the decision-maker or the learner is presented with $K$ unknown and independent probability distributions or arms from which it can generate samples. On selecting an arm, the learner observes a sample drawn independently from the underlying distribution. The learner is also provided with side information in the form of offline data, which is generated by an unknown policy. The goal of the learner is to identify the best arm (\emph{i.e.,} arm with the largest mean) in as few online rounds as possible, with at least $1-\delta$ confidence, while making optimal use of the offline data. This problem has been well-studied in the purely online setting where the learner doesn't have access to offline data~\cite{kaufmann2016complexity, garivier2016optimal, russo2016simple, jourdan2022top}. However, to the best of our knowledge, ours is the first work to formally study this problem in the o-o paradigm.

\noindent {\bf Contributions:} We first study the fundamental limits of this problem. In particular, we derive a lower bound on the expected number of online samples generated by an algorithm with access to offline data under the assumption that the offline data comes from an unknown policy which generates samples from the same bandit instance as the one in the online phase, and the online algorithm limits the error probability due to randomness in both and offline and online data to a specified $\delta $. In this framework, we provide a computationally-efficient algorithm that achieves the optimal sample complexity as $\delta \rightarrow 0$  (such asymptotic guarantees are standard in BAI literature). 
Our algorithm is a track-and-stop algorithm that pulls arms in proportions suggested by the lower bound. It requires repeatedly solving an optimization problem associated with the lower bound. We bring out key structural properties of this optimization problem, show that it has a unique solution, and use the associated properties to solve it efficiently. The proposed algorithm is computationally-efficient with an amortized computational complexity of $\tilde{O}(K)$. 

The technical analysis in the o-o version is significantly harder
compared to the purely-online setting. A key technical challenge  is that the proportions suggested by the lower bound optimization problem are a function of the amount of offline data available and can 
possibly converge to $0$ in the limit of small $\delta$. 
This requires delicate sample complexity analysis. 
Our analysis also clearly brings out the benefits  of offline data.
We also highlight this in our numerical experiments.

We also briefly discuss an alternative lower bound formulation where the $\delta$-correct guarantee is
sought along every
offline sample path (more accurately, along a set of probability 1). Even along the rare paths where the samples can be highly  misleading. 
For such formulation
we need to consider guarantees on data generating probability measures conditioned on offline data.
We conclude that such guarantees
require us to ignore offline data and
our lower bound on online samples is identical to that in the pure online setting. 
At a high level, the intuitive rationale is that
the best arm identification problems are composite hypothesis
testing problems. An algorithm with probabilistic guarantees needs to generate sufficient data  to rule out probability measures that suggest an alternative hypothesis. 
However, since all probability measures conditioned on
offline data agree
on the conditioned data, this data is not useful in
separating them. The supporting analysis
is given in  Appendix~\ref{sec:cond_correct}.
This analysis  supports the framework that
we pursue in this paper where both offline and online data 
are assumed to be generated from a common probability measure, and we seek $\delta$-correctness guarantees 
under such joint probability measures.

\noindent{\bf Related work:}
The literature on bandit optimization is vast. Here, we  primarily focus on MAB in the stochastic setting and review literature that is relevant to our work.\vspace{0.05in}\\
\noindent \textit{Purely-online learning.} MAB was first studied by \cite{thompson1933likelihood} in connection with designing adaptive clinical trials. Since then its variants have been extensively studied and are being used in practice, for example, recommendation systems, advertisement placement, routing over networks, resource allocation, etc. (see, \cite{li2010contextual, bubeck2012regret,slivkins2019introduction, lattimore2020bandit}). MAB is usually studied with one of the following three metrics: regret minimization~\cite{lai1985asymptotically, agrawal1995sample,burnetas1996optimal, auer2002finite, honda2010asymptotically,honda2009asymptotically}, BAI with fixed budget~\cite{audibert2010best,bubeck2009pure,barrier2022best}, and BAI with fixed confidence~\cite{even2006action, kalyanakrishnan2012pac,kaufmann2013information, kaufmann2016complexity,degenne2019non}.  Lower bounds and optimal algorithms for these metrics (except for BAI with fixed budget) have been designed under a variety of settings. 

For regret minimization, $\KL$-UCB \cite{garivier2011kl,cappe2013kullback}, ${\KL}_{\operatorname{inf}}$-UCB \cite{agrawal2021regret} are asymptotically optimal for parametric and heavy-tailed distributions, respectively. Thompson sampling is also shown to be optimal for parametric distributions~\cite{agrawal2012analysis,kaufmann2012thompson,korda2013thompson}. In Appendix \ref{app:regret}, we discuss an o-o version of the classical regret-minimization problem. We derive a lower bound on the regret suffered by any reasonable algorithm and discuss a natural adaptation of the $\KL$-UCB algorithm of \cite{cappe2013kullback} for the o-o problem.

The fixed-confidence BAI problem was first studied by \cite{even2006action, gabillon2012best, kalyanakrishnan2012pac, karnin13, jamieson14}. For this problem, Track-and-Stop (TaS) algorithm is shown to be asymptotically-optimal (as $\delta\rightarrow 0$) for exponential families~\cite{garivier2016optimal} as well as for heavy-tailed and bounded settings~\cite{pmlr-v117-agrawal20a}. References \cite{karnin13, jamieson14, chen2017nearly, chen17b} provide algorithms with finite $\delta$ guarantees, but these algorithms are sub-optimal in $\delta \rightarrow 0$ regime. Moreover, \cite{karnin13, jamieson14} only have high-probability guarantees instead of the stronger expected bound. References \cite{kalyanakrishnan2012pac, kaufmann2013information} propose UCB-like algorithms (LUCB1 and KL-LUCB) for $\delta$-correct BAI giving high-probability as well as expected theoretical guarantees that hold even for finite $\delta$. While these algorithms have non-asymptotic theoretical guarantees, they are known to be both theoretically sub-optimal, and also do not beat the tracking-style algorithms empirically. We discuss a version of the $\KL$-LUCB algorithm in Appendix \ref{app:KL-LUCB} adapted to our o-o setting and show numerically that it under-performs. Hence, we only focus on the optimal tracking-based algorithm in the main text. 

\noindent \textit{Purely-offline learning.} MAB and its variants have also been studied in the purely offline setting~\cite{rashidinejad2021bridging, xiao2021optimality}. These works studied various notions of optimality in the offline paradigm and designed optimal algorithms under these notions. However, the quality of the learned policies are severely bottle-necked by the quality of the offline data. For example, the minimax rates for best-arm estimation in \cite{rashidinejad2021bridging} depend on the fraction of pulls of the best arm in the offline data. These rates become trivial as the fraction of pulls of best arm approaches $0$.

\textit{Offline+online learning.} A few recent works have studied MAB in the offline-online paradigm ~\cite{shivaswamy2012multi, bouneffouf2019optimal, ye2020combining, oetomo2021cutting, banerjee2022artificial}. These works have primarily focused on regret minimization. While these works derive worst-case regret bounds for the proposed algorithms, they do not talk about optimality of these bounds either in a minimax sense or instance dependent sense. Reference \cite{bu2020online} considers the o-o regret problem in a 1-D linear bandit setting. In Section~\ref{sec:alg}, we show that the \emph{artificial replay} algorithm proposed by \cite{banerjee2022artificial} is sub-optimal for BAI. The classical 
regret minimization is another practically relevant 
problem to consider in the o-o framework. In this case,
the generalization from the online setting is straightforward. 
We  discuss the o-o regret-minimization problem in Appendix \ref{app:regret}, where we develop a lower bound as well as an algorithm.

\section{Problem setting and lower bound}\label{sec:setup_lb}
\noindent{\bf Distributional assumption:} As is common in the MAB literature, we assume that the arm distributions belong to a known \emph{canonical single-parameter exponential family} of distributions (SPEF).
Then each distribution can be indexed by its mean. This  considerably simplifies
the technicalities and allows simpler illustration of ideas.  
While SPEF are discussed in greater detail in Appendix~\ref{app:SPEF}, here
we let $\cal S \subset \Re$ denote the possible means of the SPEF under consideration.

\noindent {\bf MAB setup and objective:}
The algorithm is presented with $K$ unknown probability distributions or \emph{arms}, $\mu = (\mu_1, \dots, \mu_K)$, where each  $\mu_i \in \cal S $ denotes the mean of the distribution from the SPEF under consideration (we refer to each $\mu_i$ interchangeably as a distribution as well as its mean in the SPEF context). As is standard in this framework (of exact fixed-confidence BAI), we assume that there is a unique arm with the largest mean. If there are $2$ arms tied for the largest mean, the algorithm will stop only with a small probability as it will need to statistically separate the tied arms. One way around it is to look for an $\epsilon$-best arm (an arm whose mean is within $\epsilon$ of the best arm). However, that is technically a significantly more demanding problem. We choose to make this assumption to bring out more simply the insights associated with the utility of offline data. Without loss of generality, let $\mu_1 > \mu_{a}$ for all $a \ne 1$.

In addition to the unknown bandit instance $\mu$ from which the algorithm can generate samples, it has access to $\tau_1$  samples from the arms in $\mu$. We refer to these historical samples as the offline data and denote
it by  ${H} = \{(a_{t}, x_{t})\}_{t = 1 }^{\tau_1}$. The arm choice $a_{t}$ in ${H}$ depends on the  unknown offline policy $\pi^{0}$ that collected the data and is a function of the past data $\{(a_{t'}, x_{t'})\}_{t' < t}$. Reward $x_t$ is a realization  from distribution $\mu_{a_t}$, that is assumed to be generated independent of the past data. Next, recall that a policy specifies the data driven conditional probabilities that at each time $t+1$ determine the next action as a function of the history of realized rewards and arms pulled till time $t$. The online phase of the algorithm starts from $\tau_1+1$ and sequentially continues till some random stopping time $\tau_1 + \tau_\delta$. 

\textit{$\delta-$correct policy/algorithm:} Given $\delta > 0$, 
we are interested in an online policy that stops at time $\tau_1 + \tau_\delta>0$ and outputs a best arm estimate $k_{\tau_1 + \tau_\delta}$ such that $\mathbb{P}(k_{\tau_1 + \tau_\delta} \neq 1 ) \leq \delta $, i.e., it identifies the arm with the maximum mean with probability at least $1-\delta$. Here, the probability is with respect to joint distribution over the offline data as well as online sampling.

We are interested in such $\delta$-correct policies that minimize $\mathbb{E}[\tau_{\delta}]$. Here, the expectation is computed over the randomness in the offline and online policies as well as offline and online samples. 
This is the typical \emph{fixed-confidence} setting of the best-arm identification (BAI) framework of the MAB problem, extensively studied in literature (see, \cite{lattimore2020bandit} for a survey of known results). However, access to the offline data can significantly reduce the number of online samples ($\tau_\delta$) that need to be generated by a $\delta$-correct algorithm. We show this in Section~\ref{Sec:lb.prop} and discuss more concretely after Theorem~\ref{thm:stop_bound1}. A nuance in our analysis is that the amount of offline data, $\tau_1$, is allowed to be a function of $\delta$.

\subsection{Lower bound}\label{sec:lb}

    We now present a lower bound on the expected number of online samples that any $\delta$-correct algorithm will need to generate, while also using the available offline samples. Let $\tau_\delta$ denote the total number of online samples generated. Note that $\tau_\delta$ may depend on the observations generated by the algorithm. Hence, it is a stopping time. Let the total number of samples from arm $a$ in the offline data be denoted by $N^o_a$, and that from online sampling till time $\tau_\delta$ be $N_a$. Recall
    that  $\mu_1 > \mu_{a}$ for all $a \ne 1$. With a view to deriving the lower bound, let us define the following optimization problem ($\textbf{P1}$):
    \begin{align*}
    & \min~~~ \sum\nolimits_{a=1}^K N_a \qquad \text{s.t.} ~~~~ S_{1,a}({\bf N}) \ge \log\frac{1}{2.4 \delta}  \quad\forall a \ne 1,~
    N_a \ge 0,
\end{align*}
where ${\bf N}= (N_a: a \in [K])$, $S_{1,a}$ is defined below in~\eqref{eq:Sab} with $\KL(\nu,x)$ denoting the Kullback-Leibler divergence between two distributions in ${\cal S}$ with means $\nu$ and $x$.
\begin{align} 
    S_{1,a}({\bf N}) = \inf\limits_{x} ~ & \left\{(\mathbb {E}[N^o_{1}] + N_{1})\KL(\mu_{1}, x) + (\mathbb{E}[N^o_{a}] + N_{a})\KL(\mu_{a}, x) \right\}. \label{eq:Sab}
\end{align}
 It is known that for the SPEF family, the infimum in $S_{1,a}$ is given by (see Lemma \ref{lem:optima_SPEF}): 
\begin{align}\label{eq:unique_inf}
x_{1,a} = \frac{ (\mathbb {E}[N^o_{1}]+N_1)\mu_1 + (\mathbb{E}[N^o_a] + N_a ) \mu_a }{\mathbb {E}[N^o_{1}] + N_1 + \mathbb {E}[N^o_{a}] + N_a}.
\end{align}
Let $T^{*}$ denote the optimal value of the problem ({\bf P1}). Then, we have the following result that lower bounds the stopping time for any $\delta$-correct algorithm.

\begin{theorem}\label{thm:Lower_bound}
  For all $\mu \in {\cal S}^K$, $\mu_1 > \mu_{a}$, $a \ne 1$, a $\delta$-correct policy satisfies: $\mathbb{E}[\tau_\delta] \geq T^{*}$.
\end{theorem}

As is common in the bandit literature, the lower bound in the above theorem is obtained by a \emph{change-of-measure} argument, which is captured by the data processing inequality relating the likelihood ratio of observing samples under two different bandit instances to the likelihood of occurrence of an appropriate event under these two bandit instances (see \cite{kaufmann2016complexity, combes2014unimodal, garivier2019explore}). We give a formal proof in  Appendix \ref{app:LB}. 

\subsection{Properties of the lower bound}\label{Sec:lb.prop}

\textit{Intuitive Picture:} When is the offline data sufficient to make the correct decision at $1- \delta$ confidence? What can we qualitatively say about the online samples needed, when is the offline data not sufficient? Below we discuss these issues. 

\begin{wrapfigure}[19]{r}{0.55\textwidth}
\centering
\includegraphics[width=0.54\textwidth]{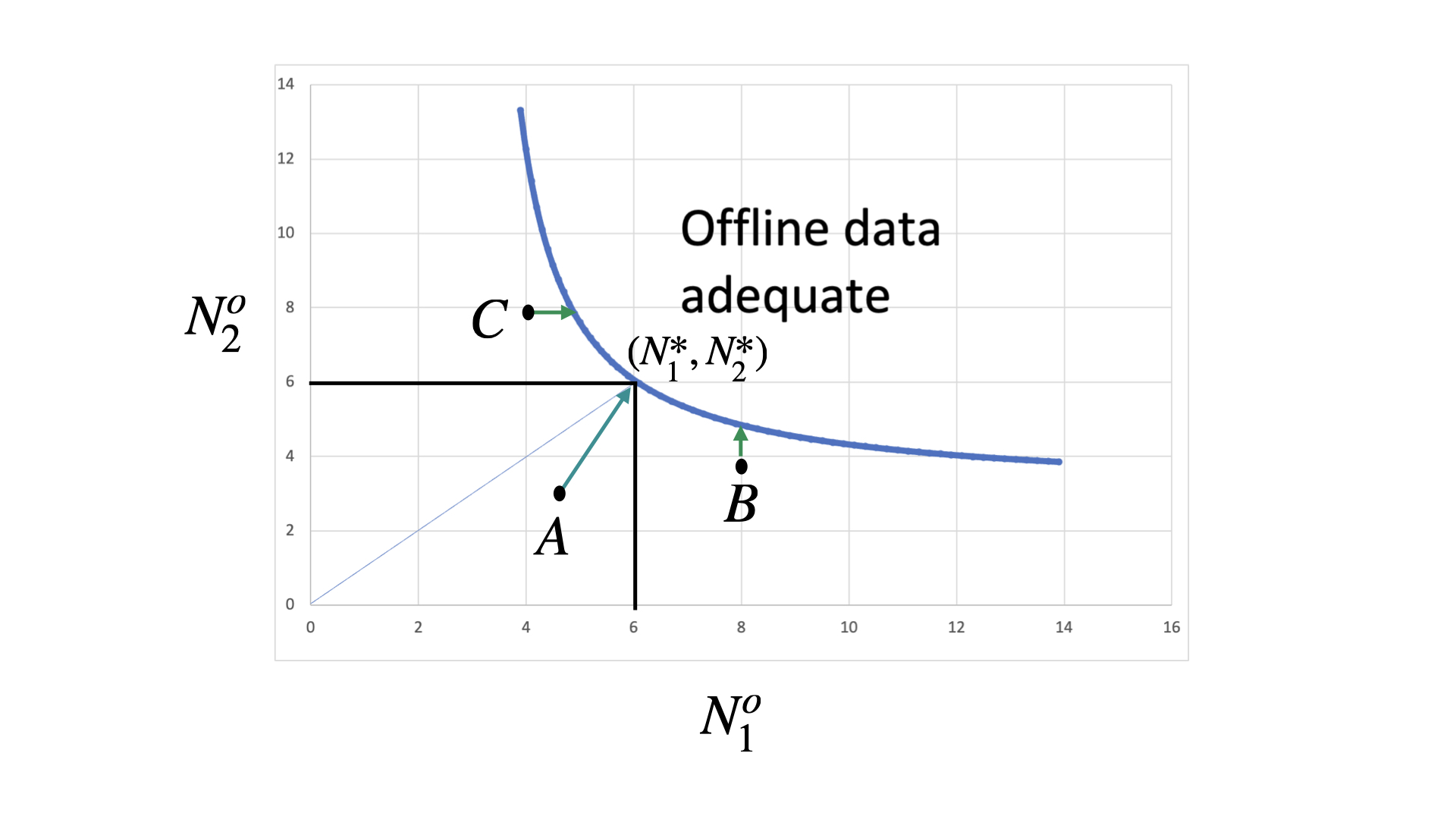}
\caption{Plot depicting the set $A_{\delta}(\mu)$ for a 2-armed MAB with Gaussian rewards, where $\mu_1-\mu_2=2$, and $\delta= 0.001$. The blue curve is the set of points $(N_1,N_2)$ where the constraint is tight. $(N^*_1, N^*_2)$ is the optimal online allocation. Point $A$ corresponds to an offline allocation with $(N_1^o, N_2^o) < (N_1^*, N_2^*)$,
$B$ to $N_1^o> N_1^*$, $N_2^o<N_2^*$, and  
$C$ to $N_1^o < N_1^*$, $N_2^o> N_2^*$. }
\label{fig:online_offline_lb_properties}
\end{wrapfigure}

Let $A_{\delta}(\nu)$ denote the set of 
${\bf N} \in \Re^K_+$ that satisfy
$\inf_{x} \{ N_{1} \KL(\nu_1, x) + N_{a} \KL(\nu_a, x)\} \geq \log\frac{1}{2.4\delta}$, for all $ a\ne 1.$ A key observation we often rely on in this section is that the LHS of the above inequality is an increasing function of $N_1$ and  $N_a$ (see Lemma~\ref{lem:convex.alloc}).

Recall, in the purely online setting, the optimal allocation  ${\bf N^*}$ in the lower bound analysis is obtained by solving the following problem (\cite{kaufmann2016complexity}):
$ \min \sum_{a=1}^K N_a ~\mbox{  such that  } ~ {\bf N} \in A_{\delta}(\mu).
$ Let ${\bf \widetilde{N}}= (\widetilde{N}_1, \ldots, \widetilde{N}_K)$ denote the optimal solution to the corresponding  o-o lower bound problem ($\mathbf{P1}$), which can be re-written as: $\min_{{\bf N} \succeq 0} \sum_{a=1}^K
N_a ~\mbox{  such that  }~ {\bf N^o + N} \in A_{\delta}(\mu).
$ Here, ${\bf N^{o}}$ are the offline samples (ignoring the expectation sign for simplicity) which we assume are known
for this discussion. 

First, note that $\sum_{a=1}^K N^*_a \geq \sum_{a=1}^K \widetilde{N}_a.$ This follows since ${\bf N^*}$ is a feasible point for the constraints in $\mathbf{P1}$. This indicates that offline data - irrespective of the policy used to collect it -  \textbf{reduces} the number of required online samples. We now discuss more nuanced properties of the solution of $\mathbf{P1}$ (illustrated in Figure~\ref{fig:online_offline_lb_properties}). To the extent that algorithms that closely match the lower bound closely mimic these properties, this also sheds light on their performance.

Suppose  ${\bf N}^o \succeq {\bf N^*}$ (i.e., $N_a^o \geq N^*_a$ for all $a \in [K]$), then online samples ${\bf \widetilde{N}} ={\bf 0}$, and we don't need any online samples. This follows from the observation that ${\bf N^*}$, and as a consequence, ${\bf N^o}$ lie in $A_{\delta}(\mu).$  On the other hand, suppose ${\bf N}^o \preceq {\bf N^*}$, then ${\bf \widetilde{N}} = {\bf N^*} -{\bf N^{o}}$, and one simply needs to allocate $N^*_a- N_a^o$ samples to each arm to solve $\bf P1$. Now, suppose ${\bf N^o} \not\preceq {\bf N^*}$, and ${\bf N^o} \not\succeq {\bf N^*}$. Consider three sub-cases suggested by Figure~\ref{fig:online_offline_lb_properties}\vspace{0.05in}:

\textbf{(a)} If $N_1^o >  N^*_1,$ then $\widetilde{N}_1= 0$.

\textbf{(b)} For 2-armed bandit problems, if $N_2^o >  N^*_2,$ then $\widetilde{N}_2= 0$. But this need not be true when $K>2$.

\textbf{(c)} For ${\bf N^o}$ to lie in the constraint set $A_{\delta}(\mu)$, we need \[N_1^o > \max\limits_{a \in [K]\setminus{\{1\}}}~ \frac{\log\frac{1}{2.4 \delta}}{ \KL(\mu_1, \mu_a)}.\] 
Similarly, for each $a\in [K]\setminus{\{1\}}$, we require 
\[N_a^o > \frac{\log\frac{1}{2.4 \delta}}{\KL(\mu_a, \mu_1)}.\] These observations are easy to see from Theorem~\ref{th:optsol.char}. We give supporting arguments towards the end of Appendix~\ref{app:lb.prop}.

\noindent {\bf Optimal solution to $\mathbf{P1}$:}
 Theorem~\ref{th:optsol.char} below characterizes the optimal solution to   the o-o lower bound problem. In particular, it establishes  the uniqueness of the solution. This extends the uniqueness of solution of the equivalent online problem shown in \cite{garivier2016optimal}. As we discuss after the theorem, its proof is much simpler, and brings out the simplicity of the problem. Recall, we assume that arm 1 is the unique best arm in $\mu$.
\begin{theorem}\label{th:optsol.char}
 If ${\bf \tilde{N}}$ is an
optimal solution  to $\mathbf{P1}$, then for arms $b \in A_1:=\{a \ne 1: \tilde{N}_a>0\}$, the index constraints are tight, i.e., $S_{1,b}({\bf \tilde{N}}) = \log\frac{1}{2.4\delta}$. For each arm $a$, let $\tilde{x}_{1,a}$ denote the corresponding infimizer in (\ref{eq:unique_inf}), let $A_2 :=  \lbrace a \ne 1: ~~\tilde{N}_a = 0,
~S_{1,a}({\bf \tilde{N}}) = \log \frac{1}{2.4 \delta} \rbrace$, and $A := A_1 \cup A_2$. Then, 
\begin{equation} \label{eqn:optsol_1}
\sum_{a \in A_1}
\frac{\KL(\mu_{1},\tilde{x}_{1,a})}{\KL(\mu_a,\tilde{x}_{1,a})}
\leq 1.
\end{equation}
If $\tilde{N}_{1}>0$, then 
\begin{equation} \label{eqn:optsol_2}
\sum_{a \in A}
\frac{\KL(\mu_{1},\tilde{x}_{1,a})}{\KL(\mu_a,\tilde{x}_{1,a})}
\geq 1.
\end{equation}
Furthermore, these conditions uniquely identify the optimal solution.  
\end{theorem}

The formal proof of Theorem~\ref{th:optsol.char} relies on Lagrangian duality and is given in Appendix \ref{app:lb.prop}. The key ideas are seen more simply in the  online setting where each $\mathbb {E}[N_a^o]=0$. We outline these. Let us replace RHS  $\log\frac{1}{2.4 \delta}$ in constraints in  ${\bf P1}$ by 1 since the solution to ${\bf P1}$ simply scales with it. We now argue that the online solution is unique, strictly positive ${\bf N^*}$, such that $A_1=[K]\setminus{1}$ and (\ref{eqn:optsol_1}) is tight.
To see the necessity of these conditions, observe that 
we cannot have $N_1^*=0$ or $N_a^*=0$ as that implies 
$S_{1,a}({\bf N^{*}})=0$. Thus, $A_2 = \emptyset$. Further if $S_{1,a}({\bf N^{*}})>1$, the objective improves by reducing $N^*_a$ so $A_1=[K]\setminus{\{1\}}$. 
To see the tightness of (\ref{eqn:optsol_1}) 
first observe through a quick calculation that derivative of 
$S_{1,a}({\bf N})$ with respect to $N_1$ and $N_a$, respectively,  equals 
$\KL(\mu_{1},x_{1,a})$ and $\KL(\mu_{a},x_{1,a})$. 

Now, perturbing $N_1$ 
by a tiny $\epsilon$ and adjusting each $N_a$ by
$\frac{\KL(\mu_{1},x_{1,a})}{\KL(\mu_a,x_{1,a})} \epsilon$
maintains the value of $S_{1,a}({\bf N})$.
Thus, at optimal ${\bf N^*}$, necessity of  tightness of (\ref{eqn:optsol_1}) follows. (This argument also  justifies 
the necessity of (\ref{eqn:optsol_1}) and (\ref{eqn:optsol_2}) in the optimal solution
in the more general o-o setting).

The fact that these three criteria uniquely specify the optimal solution
is seen by observing that for any $N_1>0$ and sufficiently large,
there exists a
unique $N_a(N_1)>0$ (through implicit function theorem) such that  $S_{1,a}({\bf N})=1$. Further, 
$N_a(N_1)$ decreases with $N_1$ and the numerator  in 
$\frac{\KL(\mu_{1},x_{1,a})}{\KL(\mu_a,x_{1,a})}$ continuously decreases with $N_1$
while the denominator continuously increases with it. 
As $N_1 \rightarrow 0$, the numerator converges to
$\KL(\mu_1, \mu_a)$, while the denominator goes to zero. Similarly,
as $N_1 \rightarrow \infty$, the numerator converges to zero
 while the denominator goes to $\KL(\mu_a, \mu_1)$.
Thus, (\ref{eqn:optsol_1}) equals 1 for a unique $N_1$. 

The algorithm that we propose in Section \ref{sec:alg} is a version of the TaS algorithm pioneered by \cite{garivier2016optimal} in that it sequentially solves for the maximizers for the lower-bound optimization problem for the empirical mean vector and uses these to decide the arm to sample at each time. Properties of the solution to the lower bound problems described above are crucially used to solve the optimization problems efficiently through bisection search in the TaS algorithm.

\section{The algorithm}\label{sec:alg}
An algorithm for the fixed-confidence BAI problem has three main components: a) \emph{sampling rule} is a specification of the arm to sample at each step; b) \emph{stopping rule} decides if the algorithm should stop generating more samples; and c) \emph{recommendation rule} outputs an estimate for the arm with the maximum mean at the stopping time. We now introduce some notation that will be useful in specifying these different components of the algorithm. 
    
At each time $t$, let the empirical mean corresponding to $N^o_a+N_a(t)$ samples from arm $a$ be denoted by $\hat{\mu}_a(t)$. It equals $\sum\nolimits_{l=1}^{\tau_1 + t} \frac{X_l {\bf 1}\lrp{A_l = a}}{N^o_a + N_a(t)},$ where $X_l$ denotes the random reward at time $l$. Let $\hat{\mu}(t)= (\hat{\mu}_a(t): a \in [K])$. Let     $i^*(t) \in \argmax_a ~ \hat{\mu}_a(t)$ be the empirical best arm. Here, ties (if any) are broken arbitrarily. Let $\mathbf{U}_K = \left[\frac{1}{K} \ldots \frac{1}{K} \right]$. We first describe various components of Algorithm 1: 
    
\noindent{\bf Stopping rule: } We use the generalized likelihood ratio test (GLRT) to decide when to stop. Recall that for $p\in\Re$ and $q\in\Re$, $\KL(p,q)$ denotes the $\KL$ divergence between unique probability distributions in $\cal S$ with means $p$ and $q$. For $t\in\mathbb{N}$, define threshold $\beta(t,\delta)$ as   
    \begin{align}\label{eq:beta}
        \beta(t,\delta) = \log\frac{K-1}{\delta} &+ 6\log\lrp{\log \frac{t}{2} + 1 }+ 8 \log\lrp{1 + \log\frac{K-1}{\delta}}.
    \end{align}
    
For ${\bf N}(t) = \lrset{N_1(t), \dots, N_K(t)} $ and for $b\ne a$, define 
\[Z_{a, b}({\bf N}(t), \hat{\mu}(t)) := \inf\nolimits_{x}  \{(N^o_{a}+N_{a}(t))\KL(\hat{\mu}_{a}(t), {x})+(N^o_{b}+N_{b}(t))\KL(\hat{\mu}_{b}(t), {x}) \}.\]
Observe that the statistic $Z_{a,b}(\cdot)$ is similar to $S_{a,b}(\cdot)$ defined in (\ref{eq:Sab}), except that in $Z_{a,b}(\cdot)$ we use the observed number of offline samples from each arm, and the empirical estimates of the means of each arm ($\hat{\mu}(t)$), instead of their actual values. To simplify the notation, in the sequel, we often drop the dependence of $Z_{a,b}(\cdot)$ on $\hat{\mu}(t)$. The stopping rule corresponds to checking if $Z_{i^*(t),b}({\bf N}(t))$ is at least $\beta(\tau_1 + t,\delta)$ for each arm $b\ne i^*(t)$. The statistic $Z_{i^*(t),b}(\cdot)$ is related to the generalized log-likelihood ratio (GLR) for testing if $i^*(t)$ is the arm with the maximum mean against all the alternative hypothesis. Let $\tau_\delta$ denote the time at which that algorithm stops and outputs $i^*(t)$. Then, \[ \tau_\delta = \min\{t\in\mathbb N: ~ Z_{i^*(t),b}({\bf N}(t)) \ge \beta(\tau_1 + t,\delta)~\forall b\ne i^*(t) \}. \]
    
\noindent\textbf{Sampling rule: }  Whenever $t \in \{r^2 K , ~ r \in \mathbb{N}\}$, our algorithm solves the following upper bound problem with plug-in empirical mean estimates $\hat{\mu}(t)$ (call it $\mathbf{P2}(\hat{\mu}(t))$):
\begin{align}\label{prob:P2}
    & \min\sum\nolimits_{a} N_a ~~ \text{ s.t. } ~~ Z_{i^{*}(t),b}({\bf N}, \hat{\mu}(t)) \ge \log\frac{1}{\delta} + \log\log\frac{1}{\delta}, ~\forall b\ne i^{*}(t), ~~ N_b  \ge 0, ~ \forall b \in [K].
\end{align}

The main differences between this and the lower bound problem $\mathbf{P1}$ are that the constraint is modified to $\log \frac{1}{\delta} + \log \log (\frac{1}{\delta})$ and usage of actual offline samples $N_a^{o}$ instead of their expectation. In Section~\ref{sec:comp_complexity}, we present an efficient algorithm to solve this optimization problem.

Let ${\bf {N}^*}$ be the optimal solution to $\mathbf{P2}$, and let $\hat{w}_a(t) := {{N}^*_a}/{\sum_{b} {N}^*_b}$. Our algorithm uses running average of $\hat{w}(t)$ to pull arms until the problem $\mathbf{P2}$ is solved again (at which point, the algorithm switches to the new proportions). To be precise, it maintains a running average of $\hat{w}(t)$ (call it $w(t)$), and pulls arms in such a way that the arm sampling proportions match ${w}(t)$. Whenever $t$ is such that $t=r^2K$ for some $r\in \mathbb{N}$, the algorithm goes into an exploration phase for $K$ rounds, where uniform proportions $\mathbf{U}_K$ are added to the running average of $\hat{w}(t)$. At the end of this exploration phase, $\hat{w}(t)$ is re-computed and $w(t)$ is set to $\hat{w}(t)$. See Algorithm~\ref{alg:tas_batch} for details.

 \textbf{Tracking online proportions:} The sampling strategy is based on the following rule: $A_t \leftarrow \argmax_{a} w_a(t)/N_a(t-1)$. This is known to ensure that the tracking error for the running-averages $w(t)$ remains within $\frac{K}{t}$ for any $t$ (see Appendix \ref{sec:tracking}).
 
\textbf{Computational cost: } In Section~\ref{sec:comp_complexity}, we discuss an algorithm to compute an $\epsilon$ approximation of $\hat{w}(t)$ that runs in $O(\log^2 \frac{1}{\epsilon})$ time. Since Algorithm~\ref{alg:tas_batch} only solves problem $\mathbf{P2}$ $\sqrt{T}$ many times till $T$ online trials, the amortized computational cost of our algorithm is $\Tilde{O}(K)$. We note that even very small values of epsilon do not blow up the computation in practice. We use $\epsilon = 10^{-6}$ in our experiments.
    
    \begin{algorithm2e}
   \caption{Batch Track-and-Stop (TaS)}
   \label{alg:tas_batch}
   	\DontPrintSemicolon 
	\KwIn{confidence level $\delta$, historic data $\left(\tau_1, \{N^o_a, \hat{\mu}_a^o\}_{a=1}^{K}\right)$}
    \SetKwProg{Fn}{}{:}{}
    { 
        \textbf{Initialization} Pull each arm once. Set  $t\leftarrow K, N_a(K) \leftarrow 1, \mathsf{count} \leftarrow0$, $w(K) \leftarrow \mathbf{U}_K$. Update $\hat{\mu}(K)$. \\
        \While{$\min\limits_{b\ne i^*(t)} Z_{i^*(t),b}({\bf N}(t)) \leq \beta(\tau_1 + t,\delta)$ }{
        \If{$\sqrt{\lfloor t/K\rfloor} \in \mathbb{N}$}
                {
                $\mathsf{count} \leftarrow \mathsf{count} + 1$ \\
                $w(t+1) \leftarrow \frac{t}{t+1} w\lrp{t}+ \frac{1}{t+1}  \mathbf{U}_K $ \\
                \If{$\mathsf{count} \equiv 0 \pmod K$}{
                 $\hat{w}(t+1) \leftarrow$ the estimated fractions from  solution of Problem $\mathbf{P2}$ in (\ref{prob:P2}).\\
                 Set $\mathsf{count} \leftarrow 0$. } 
                 
                }
                \Else{
                $\hat{w}(t+1) \leftarrow \hat{w}(t)$ \\
                 $w(t+1) \leftarrow \frac{t}{t+1} w\lrp{t}+ \frac{1}{t+1}  \hat{w}(t+1)$. 
                }
        Sample arm $A_{t+1} \leftarrow \argmax_{a\in[K]}  \frac{w_a(t+1)}{N_a(t)} $  \\
        Update $N_{A_t}(t+1)$ and $\hat{\mu}(t+1)$.  \\
        Set $t \leftarrow t + 1$.
        }
        \KwOut{$\argmax_{a\in[K]} \hat{\mu}_a(t)$}
    }
\end{algorithm2e}

\subsection{ Theoretical guarantees } 
The following theorem shows that the proposed algorithm is $\delta$-correct. We refer the reader to Appendix \ref{sec:delta_correctness} for its proof.
\begin{theorem}\label{thm:delta-correct}
 Over the randomness in offline samples, (unknown) offline policy, online policy, and samples, the algorithm proposed is $\delta$-correct.
\end{theorem}

\noindent{\bf Stopping time analysis: } We now present the stopping time of the proposed algorithm. We are interested in solution to problem $\mathbf{P2}(\hat{\mu}(t))$ where $i^{*}(t)=1$ and when empirical estimates $\hat{\mu}(t)$ are used versus an identical version of the problem when true means $\mu$ are used, i.e. $\mathbf{P2}(\mu)$.
We would like to show that the solutions are close when $\hat{\mu}(t)$ gets close to $\mu$. For this, one needs solution space of the problem to be compact. The solution $N_a$ is potentially unbounded. So we consider the following normalized version of the problem which gives rise to a new max-min formulation compared to  the purely-online case.

Let us introduce some notation before stating the equivalent normalized version. Let $\hat{p} \in \Sigma_K$ denote the observed fraction of samples from each arm in the offline phase, i.e., $\hat{p}_a= {N_a^o}/{\sum_a N_a^o}$. For $z\in [0,1]$, $w\in\Sigma_K$, the probability simplex in $\Re^K$, $x \in \Re$, $j\in [K]$ and $j\ne 1$, let 
\begin{align*}
    g_j(w,z,x,\mu, \hat{p}) &:= (z \hat{p}_1 + (1-z)w_1)\KL(\mu_1, x) + (z \hat{p}_j + (1-z)w_j) \KL(\mu_j, x).
\end{align*}    
Define $ V(\mu,z,p):= \max\limits_{w\in\Sigma_K}\min\limits_{j\ne 1} \inf\limits_{x} g_j(w,z,x,\mu,p).$ Consider the following optimization problem: 
    \begin{align}\label{equiv3}
        \max\limits_{z\in [0,1]} ~~ z \quad \text{ s.t. }\quad V(\mu,z,p) \ge \frac{z}{\tau_1} \left( \log \frac{1}{\delta} + \log \log \frac{1}{\delta}\right).
    \end{align}
Call it $\mathbf{P3}$. Lemma \ref{lem:monotonicVmuz} in appendix shows that the l.h.s. in the constraint above is a non-increasing function of $z$. Since r.h.s. is monotonically increasing, there is a unique point at which the constraint holds as an equality. $z^*$ is this unique point of intersection if it lies in $[0,1]$, else $z^* = 1$.  

Since we will only be working with the empirically-observed offline fractions $\hat{p}$ in the algorithm, we drop the dependence of various functions on it in the sequel. In the following lemma, we show that for the empirical $\hat{p}\in\Sigma_K$, the problem $\mathbf{P2}(\mu)$ (that uses observed offline data) is equivalent to $\mathbf{P3}(\mu)$. This equivalence also holds when empirical means $\hat{\mu}(t)$ are used instead. 

\begin{lemma}\label{lem:equivalence}The optimal proportions obtained from the solution set of $\mathbf{P2}(\mu)$ are the same as those for $\mathbf{P3}(\mu)$. Moreover, from optimizers of $\bf P2(\mu)$, one recovers the optimizers of $\bf P3(\mu)$.
\end{lemma}
Suppose $(z^*, w^*)=(z^*(\mu), w^*(\mu,z^*(\mu)))$ solves ${\bf P3(\mu)}$, and let ${\bf N^*}$ be the optimal solution for $\bf P2(\mu)$. We show that $w^*_a$ corresponds to the optimal proportion of pulls from arm $a$ in the online samples, i.e., $w^*_a = N^*_a/\sum_b N^*_b$. Moreover, $z^* = \tau_1/(\tau_1 + \sum_b N^*_b)$, i.e., it corresponds to the optimal fraction of offline data. We refer the reader to Appendix \ref{sec:EquivForm} for a proof of the lemma. 

It also follows that the lower bound on $\mathbb{E}(\tau_\delta) + \tau_1$  (Theorem~\ref{thm:Lower_bound}) can equivalently be shown to equal $\frac{\tau_1}{z^*}$. Recall that $\tau_1$ is allowed to be a function of $\delta$. Using the normalized formulation in $\bf P3$, we prove the following non-asymptotic bound on the expected sample complexity of the algorithm.

\begin{theorem}\label{thm:stop_bound1}
    For $\delta > 0$, suppose $\mu\in\mathcal S^K$ is such that $ 1 \ge z^*(\mu) > \eta$, for some $\eta > 0$. Let $\epsilon'>0, \tilde{\epsilon}>0, \epsilon_1>0 $ be constants. If the given problem instance $(\mu,\hat{p},\tau_1)$ is such that:
    \begin{equation}\label{eq:cond_tau_1}\frac{\tau_1}{\log \frac{1}{\delta} + \log \log \frac{1}{\delta}} \notin (C_1^a(c_{\epsilon_1}),C_2^a(c_{\epsilon',\tilde{\epsilon}})) \quad \forall a\ne 1,\end{equation}
    then the algorithm satisfies:
    \[ \mathbb{E}[\tau_{\delta}] + \tau_1 \le   \frac{\tau_1}{z^*}\lrp{1 + \frac{2\alpha(\epsilon') + \tilde{\epsilon}}{\epsilon_1 C(\mu,\hat{p},\eta)} } + T(\epsilon') + T(\tilde{\epsilon}) + 1 + o\lrp{\log\frac{1}{\delta}}.\]
    Here, for an instance-dependent constant $L_\mu$, $c_{\epsilon_1} := \epsilon_1 L_{\mu} ({1}/{\eta}-1) $ and $c_{\epsilon',\tilde{\epsilon}}:= \lrp{2 \alpha(\epsilon') + \tilde{\epsilon}}/{\eta}$, where $\alpha(\epsilon')$ is a continuous function of $\epsilon'$ such that $\alpha(\epsilon'){\rightarrow} 0 $ as $\epsilon' \rightarrow 0$. Moreover, $C_1^a(\cdot)$ and $C_2^a(\cdot)$ are functions such that $C^a_1(c) \rightarrow C^a_1(0)$ and $C^a_2(c) \rightarrow C^a_2(0)$ as $c\rightarrow 0$, and $C^a_1(0) = C^a_2(0)$. $C(\mu,\hat{p},\eta)$ is a non-negative function of $\mu, \hat{p}$ and $\eta$ that is strictly positive for $\eta > 0$.
\end{theorem}

\noindent{\bf Discussion. } Observe that the first term in the bound on expected sample complexity is the dominant term for small $\delta$ and the lower bound is $\frac{\tau_1}{z^*}$. We show in Appendix~\ref{app:samplecomplexity} that $C^a_1(0)=C^a_2(0)$ is precisely the scaling of $\tau_1$ where the offline data from arm $a$ is \textit{just} sufficient, \emph{i.e.,} the inequality in~\eqref{equiv3} is \textit{just} satisfied for arm $a$, when $\mu$ is exactly known. Let us now make the the statement of Theorem~\ref{thm:stop_bound1} more intelligible by considering different cases. 

Consider a sequence of problems $(\mu,\hat{p}, \tau^\delta_1)$ indexed by $\delta$ such that the liminf of the ratio on l.h.s. in~\eqref{eq:cond_tau_1} is smaller than $C^a_1(0)$ for all arms $a$. This is the setting when in the limit as $\delta\rightarrow 0$, the offline samples for each arm are not sufficient to identify the best arm, and the algorithm needs to generate some online samples from each arm. Then, for small enough $\delta$, we can find $\epsilon_1$ small enough, so that the condition in~\eqref{eq:cond_tau_1} is satisfied, and the bound in Theorem~\eqref{thm:stop_bound1} then gives: \begin{equation}\label{eq:bound_special1} \limsup\limits_{\delta \rightarrow 0} \frac{\mathbb{E}(\tau_\delta) + \tau^\delta_1}{ \log\frac{1}{\delta}} \le \limsup\limits_{\delta\rightarrow 0}\frac{\tau^\delta_1}{z^* \log\frac{1}{\delta}}\lrp{1+ \frac{2\alpha(\epsilon') + \tilde{\epsilon}}{\epsilon_1 C(\mu,\hat{p}, \eta)}}. \end{equation}
Since $\epsilon'$ and $\tilde{\epsilon}$ are arbitrary, we take them to $0$ to get a matching upper bound.

Next, consider the other case when the limsup of the ratio on l.h.s. in~\eqref{eq:cond_tau_1} is greater than $C^a_2(0)$. Here again for small enough $\delta$, we can find $\epsilon'$ and $\tilde{\epsilon}$ such that $\tau^\delta_1$ satisfies~\eqref{eq:cond_tau_1}, and Theorem~\ref{thm:stop_bound1} again gives a bound similar to~\eqref{eq:bound_special1}. Again taking Since $\epsilon'$ and $\tilde{\epsilon}$ to $0$, we get a matching upper bound. 

It is interesting to note that unlike in the purely-online setting, the amount of offline data available for each arm, influences the proportion of online samples that need to be generated from them.  While in the purely-online setting, these are guaranteed to be strictly positive, it is no longer the case in our setting. We carefully handle this nuance in our analysis.

\noindent{\bf Gain achieved through offline data } The benefit of availability of the offline data is best seen through a reduction in the lower bound on online samples generated for o-o problem compared to the pure online problem. Recall from~\cite{garivier2016optimal} that in the online setting, it equals $\frac{1}{V(\mu,0)} \log\frac{1}{2.4 \delta}$. In the o-o setting this equals $\frac{1-z^*}{V(\mu,z^*)}\log\frac{1}{2.4 \delta}$ for $z^*>0$. Further, it is easy to see from definition of $V$ that $\frac{1-z}{V(\mu,z)}$ decreases with $z$. Thus offline samples always reduce the lower bound, and the benefit is essentially $\frac{(1-z^*)V(\mu,0)}{V(\mu,z^*)}$.

\subsection{Computational complexity}\label{sec:comp_complexity}
In this section, we present a $O(\log^2\frac{1}{\epsilon})$ algorithm for computing an $\epsilon$-approximate solution for $\mathbf{P2}$ that uses a nested bisection search. To simplify the presentation, we assume the empirical-best arm $i^*(t) = 1$. Let $\lrset{N^*_a: ~ a\in [K]}$ denote the optimal solution for $\mathbf{P2}(\hat{\mu}(t))$. Recall that these represent the number of samples required to be generated from arm $a$ in addition to  $N^o_a$ offline samples. Theorem \ref{th:optsol.char} characterizes the unique solution to problem $\mathbf{P2}(\hat{\mu}(t))$. Since the solution exists, each of the optimal allocations $N^*_a$ are bounded. Denote the maximum of these bounds by $B$. 

Next, for any given $N_{1}$, let ${N}_a(N_{1})$ denote the solution to the following equation: $Z_{1,a}(N_1, \{{N}_b(N_1)\}_{b\ne 1}) = \log\frac{1}{\delta} + \log\log\frac{1}{\delta}, $ with  $Z_{1,a}(N_1, \{{N}_b(N_1)\}_{b\ne 1})$ defined below Equation~\eqref{eq:beta}.  Notice that ${N}_a(N_1)$ can possibly be negative. If the above equation does not admit a solution in $\Re$, then $N_1$ is infeasible for $\mathbf{P2}(\hat{\mu}(t))$ and we need to increase it (detailed later). 

With these notation, $\mathbf{P2}$ is equivalent to solving the following for $n_1$. Call it $ \mathcal  O_2$: $\min ~ n_1 + \sum\nolimits_{b\ne 1} \max\lrset{0,{N}_b(n_1)} \text{ such that } {n_1 \in [0,B]}.$
Let $n^*_1$ be its optimal solution. Lemma~\ref{lem:convex.alloc} shows that the objective in $\mathcal O_2$ is convex in $n_1$, which we compute $n_1^*$ using a bisection search (Algorithm~\ref{alg:bisection} in the appendix). Then, $N^*_1 = n^*_1$ and $N^*_a = \max\{0, N_a(n^*_1)\}$ (Lemma \ref{lem:optall}).

Bisection search for $n_1$ converges to an $\epsilon$-approximate $n^*_1$ in $O(\log\frac{1}{\epsilon})$ iterations (see, \cite[Section 2.2]{bubeck2015convex}). Each iteration of this bisection search involves computing $N_a(n_1)$ for $a\ne 1$. Since for fixed $N_1=n_1$, $Z_{1,a}({\bf N})$ is a monotonic function of $N_a$ (Lemma~\ref{lem:convex.alloc}), we again rely on bisection search to compute $N_a(n_1)$. However, for the bisection search to succeed, we first require existence of $N_a(n_1)$. This can be checked by computing the maximum value of the index ($Z_{1,a}$) for arm $a$ for a fixed $n_1$. If this maximum value is smaller than  $\log\frac{1}{\delta} + \log\log\frac{1}{\delta}$, we increase $n_1$. We perform bisection search to compute $N_a(n_1)$ only if the maximum value of the index for each arm is larger than $\log\frac{1}{\delta} + \log\log\frac{1}{\delta}$. We discuss existence of $N_a(n_1)$ and its computation in Appendix \ref{sec:Nan}.

\section{Experiments and discussions}\label{sec:exps}
We first present empirical evidence showing the importance of online-offline paradigm over purely online, purely offline learning paradigms in Figure \ref{fig:online_offline} (and in Appendix~\ref{sec:additional_exps}). We generated offline data from two policies: (a) a policy which uniformly samples all the arms, and (b) a policy that uniformly samples all the arms except the best arm. The latter policy doesn't pull the best arm at all. Figure~\ref{fig:online_offline} presents the expected stopping time of our algorithm, as we vary the amount of offline data. Here are two important takeaways from these results: (a) the number of online samples decreases as the amount of offline data increases, (b) even if the offline data is of poor quality (\emph{e.g.,} data generated by the second policy above), it  helps reduce the number of online rounds. Note that learning algorithms that solely rely on offline data wouldn't have worked in such cases. 

We considered the important problem of
merging offline and online data to improve learning outcomes. Specifically, we focused on the best arm identification problem in the fixed confidence setting, where the learner has access to offline samples 
from the same bandit instance as online samples. We developed a lower bound
on the sample complexity, 
and developed a TaS algorithm that matched the lower bounds.  A direction for future work would be to consider algorithms outside TaS akin to those in the purely-online setting like $\beta$-top two. Extending these techniques to o-o setting is non-trivial. This is because the optimality conditions for our problem (Theorem~\ref{th:optsol.char}) differ from the usual online problem.

\begin{figure}[tb]
\begin{center}
\includegraphics[scale=0.25,  trim={1cm 0 0 0cm},clip]{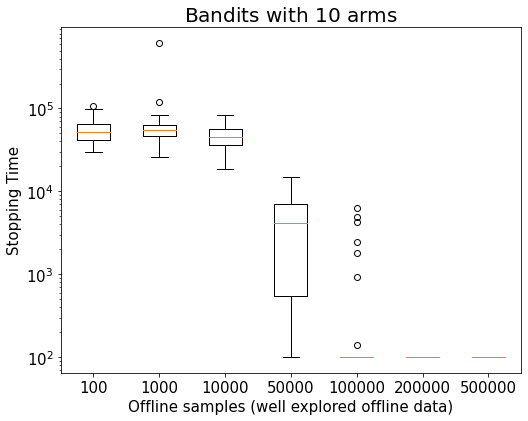} 
    \includegraphics[scale=0.25,  trim={1cm 0 0 0cm},clip]{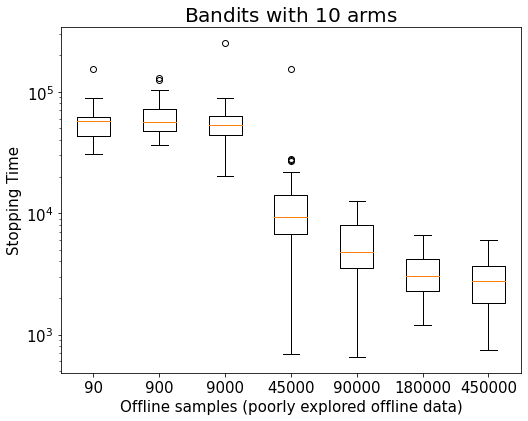}
      \caption{Stopping time of Algorithm~\ref{alg:tas_batch} with varying number of offline samples. The offline samples are collected using two policies that are described in Section~\ref{sec:exps}. The rewards of the arms follows Gaussian distribution with variance $1$. We chose $\delta = 10^{-3}$ for these experiments. Results are averaged over $50$ trials.}
    \label{fig:online_offline}
\end{center}
\end{figure}

\bibliography{ref}

\newpage
\appendix
\section{Properties of SPEF}\label{app:SPEF}

Consider an SPEF distribution family ${\cal I}$ with the following property: $\eta_\theta \in \cal I$ has density with respect to $\rho$ (taken to be counting measure or Lebesgue Measure) given by 
\[ \frac{d \eta_\theta}{d\rho} = e^{\theta x - b(\theta)}, \quad \forall x\in \Re, \]
where $b: \Theta \rightarrow \Re$ is a normalizing factor that is twice differentiable and strictly convex. Let $\cal S \subset \Re$ denote the set of means of all distributions in $\cal I$. It can be shown that mean of distribution $\eta_\theta$, denoted by $m(\eta_\theta)$, equals $\frac{d b(\theta)}{d\theta}$ (denoted by $\dot{b}(\theta)$). Thus, there is a one-to-one correspondence between distributions in $\cal I$ and their means in $\cal S$.

Next, for $m_1$ and $m_2$ in $\cal S$ let $\KL(m_1, m_2)$ denote the $\KL$ divergence between the unique distributions in $\cal I$ with means $m_1$ and $m_2$. Then, 
$$\KL(m_1, m_2)=(\dot{b}^{-1}(m_1) - \dot{b}^{-1}(m_2))m_1 - (b(\dot{b}^{-1}(m_1)) - b(\dot{b}^{-1}(m_2))). $$
It follows from the above expression and from properties of $b$ that $\KL(\cdot, \cdot)$ is a jointly continuous function. We henceforth denote every  distribution in our SPEF by its mean,
and we let $\cal S$ denote this collection of distributions.
The following result is well known for SPEF.

\begin{lemma}[Lemma $3$ in \cite{garivier2016optimal}]\label{lem:optima_SPEF}
  Consider two distributions from the SPEF family $ \mu_1, \mu_2 \in \cal S$ with  $\mu_1 > \mu_2$. For fixed $\lambda_1,~\lambda_2 \in \Re^{+}$.
  \begin{align*}
      \inf \limits_{\substack{x \leq y,\\ x,y \in \cal I}}   \lrset{\lambda_1 \KL(\mu_1,x) + \lambda_2 \KL(\mu_2,y)} = \inf \limits_{ x \in \cal I} &  \lrset{\lambda_1 \KL(\mu_1,x) + \lambda_2 \KL(\mu_2,x) }.
  \end{align*}
Furthermore, the infimum is attained at 
\[x^{*} = \frac{ \lambda_1 \mu_1 + \lambda_2 \mu_2}{\lambda_1 + \lambda_2}.\] 
\end{lemma}

\section{Supporting results and proofs for Section \ref{sec:setup_lb}}\label{app:LB}
\subsection{Proofs for Section \ref{sec:lb}}
\begin{proof}[Proof of Theorem \ref{thm:Lower_bound}]
    Let the $N^o_a$ offline samples from arm $a$ be denoted by $\lrset{Y_{a,i}}$, for $i\in \lrset{1, \dots, N^o_a}$. Looking at this offline data for each arm, suppose the $\delta$-correct algorithm collects an additional $\tau_\delta$ samples of which $N_a(\tau_\delta)$ are from arm $a$, for each arm $a\in\lrset{1, \dots, K}$. For simplicity of notation, we let $\lrset{Y_{a,i}}$, for $i\in \lrset{ N^o_a+1, \dots, N^o_a + N_a(\tau_\delta) }$ denote the online samples from arm $a$ till time $\tau_\delta$. 
    
    With this notation, let $L_\mu(\lrset{Y_{a,i}}_{a,i})$ denote the likelihood of the $\tau_1 + \tau_\delta$ samples under the bandit instance $\nu$ and $L_{\tilde{\nu}}(\lrset{Y_{a,i}}_{a,i})$ be that under any alternative bandit model $\tilde{\nu}$. Then,
    \[L_\mu(\{Y_{a,i}\}_{a,i}) = \prod\limits_{a=1}^K\prod_{i=1}^{N^o_a + N_a(\tau_{\delta})} \mu_a(Y_{a,i}),\quad \text{ and }\quad L_{\tilde{\nu}}(\{Y_{a,i}\}_{a,i}) = \prod\limits_{a=1}^K\prod_{i=1}^{N^o_a + N_a(\tau_{\delta})} \tilde{\nu}_a(Y_{a,i}). \]
    Taking expectation of the log-likelihood ratio with respect to all the randomness in the system, 
    \begin{equation} \label{eq:averagell}
    \E\lrs{ \log\frac{L_\mu(\lrset{Y_{a,i}}_{a,i})}{L_{\tilde{\nu}}(\lrset{Y_{a,i}}_{a,i})} } = \sum\limits_{a\in [K]} \lrp{ {\mathbb{E}[N^o_a]}+  \E\lrs{N_a(\tau_\delta)}} \KL(\mu_a,\tilde{\nu}_a)  . 
    \end{equation}

    For $q$ and $r$ in $[0,1]$, let $d(q,r)$ denote the $\KL$ divergence between Bernoulli distributions with mean $q$ and $r$. Clearly, the L.H.S. above is $\KL$ divergence between the joint distribution of $\lrset{Y_{a,i}}_{a,i}$ when the samples are generated from bandit instance $\mu$ and that when they are generated from the bandit instance $\tilde{\nu}$. Data processing inequality (\cite{garivier2019explore}, Section  2.8 in \cite{cover2006elements}) then guarantees that the above is at least 
    \[ d( \mathbb{P}_{\mu}(\mathcal E), \mathbb{P}_{\tilde{\nu}}(\mathcal E) ), \quad \forall \mathcal E \in \mathcal F_{\tau_1 + \tau_\delta}, \]
    where $\mathbb{P}_\mu$ and $\mathbb{P}_{\tilde{\nu}}$ denote the probabilities when the interactions are with bandit instances $\mu$ and $\tilde{\nu}$, respectively. This gives 
    \[ \sum\limits_{a\in [K]} \lrp{ \mathbb{E}[ N^o_a] + \E\lrs{N_a(\tau_\delta)} }\KL(\mu_a, \tilde{\nu}_a) \ge \sup\limits_{\mathcal E\in \mathcal F_{\tau_1 + \tau_\delta}} d(\mathbb{P}_\mu(\mathcal E), \mathbb{P}_{\tilde{\nu}}(\mathcal E)).  \]

Now, we minimize the left hand side over all alternate instance $\mathrm{Alt}(\mu)= \{\tilde{\nu}: \max_a \tilde{\nu}_a \geq \mu_1 ,~ a \neq 1\}$ over the SPEF family. We consider the event ${\cal E} = 1_{k_{\tau_{\delta}} = 1 }$. Recall that $\mu_1$ is the larges mean in $\mu$. By the definition of the alternate instance, $a=1$ is not the best arm in $\tilde{\nu}$. Therefore, RHS becomes $d(\mathbb{P}_\mu(\mathcal E), \mathbb{P}_{\tilde{\nu}}(\mathcal E)) = d(\delta, 1-\delta)$ for a $\delta$ correct algorithm that works for all instances in the SPEF. This shows that:
 \begin{align}\label{eqn:constraint}
       \inf \limits_{\tilde{\nu} \in \mathrm{Alt}(\mu)} \sum\limits_{a\in [K]} \lrp{ \mathbb{E}[ N^o_a] + \E\lrs{N_a(\tau_\delta)} }\KL(\mu_a, \tilde{\nu}_a) \ge d(\delta,1-\delta).
 \end{align}

It is known (see, \cite{garivier2016optimal} and \cite[Chapter 33]{lattimore2020bandit}) that the alternate instance $\nu$ of the form $\tilde{\nu}_1,\tilde{\nu}_a, \{\mu_b\}_{b \notin \{a,1\}}$ such that $\tilde{\nu}_a \ge \tilde{\nu}_1$ are the infimizers in the l.h.s of (\ref{eqn:constraint}). Thus, the inequality in (\ref{eqn:constraint}) reduces to the following.

\begin{align}
       \min\limits_{a\ne 1}\inf \limits_{\tilde{\nu}_1\le \tilde{\nu}_a} \lrp{ \mathbb{E}[ N^o_1] + \E\lrs{N_1(\tau_{\delta})} }\KL(\mu_1, \tilde{\nu}_1) + \lrp{ \mathbb{E}[ N^o_a] + \E\lrs{N_a(\tau_\delta)} }\KL(\mu_1, \tilde{\nu}_a) \nonumber \\
       \geq d(\delta,1-\delta) \geq \log \frac{1}{2.4 \delta}\label{eqn:constraint2}.
 \end{align}

 From Lemma \ref{lem:optima_SPEF}, for fixed $a\ne 1$,  the infimum in the l.h.s. above equals
 \[ \inf \limits_{x} \lrp{ \mathbb{E}[ N^o_1] + \E\lrs{N_1(\tau_{\delta})} }\KL(\mu_1, x) + \lrp{ \mathbb{E}[ N^o_a] + \E\lrs{N_a(\tau_\delta)} }\KL(\mu_1, x), \]
and the optimal $x$, denoted by  $x_{1,a}$ is given by Lemma \ref{lem:optima_SPEF}. Thus, the  constraint in (\ref{eqn:constraint2}) re-writes as 
\begin{align}\label{eq:finalCons} 
\inf \limits_{x} \lrp{ \mathbb{E}[ N^o_1] + \E\lrs{N_1(\tau_{\delta})} }\KL(\mu_1, x) + \lrp{ \mathbb{E}[ N^o_a] + \E\lrs{N_a(\tau_\delta)} }\KL(\mu_1, x) \nonumber \\
\ge \log\frac{1}{2.4\delta}, \quad \forall a\ne 1. 
\end{align}
Clearly, stopping time minimizes for $\sum \mathbb{E}[N_a]$ subject to (\ref{eq:finalCons}) establishing the theorem.
\end{proof}

\subsection{Conditional $\delta$-correctness} \label{sec:cond_correct}
In this section we argue that an exploration algorithm that picks the best arm conditioned on the past realization of offline samples with probability $1-\delta$ uniformly for all realizations (\emph{i.e., } a conditionally $\delta$-correct algorithm defined below) would have to discard offline samples. Observe that this is a very strong notion of $\delta$-correctness. We show that under such a strong requirement, it is not possible to do any better than a naive algorithm that simply discards the offline data. This negative result (Theorem~\ref{them:cond} below) motivates the use of the notion of $\delta$ correctness in the main text.

\begin{definition}[Conditionally $\delta$-correct policy/algorithm]\label{def:cond_delta}{\em
 An online policy that given any history $H$ (of length $\tau_1$, say) which is an event in the sigma algebra ${\cal H}_0$ generated by the offline policy, samples arms adaptively till a stopping time $\tau_{\delta}$, and outputs an estimate for the best arm $k_{\tau_1 + \tau_\delta}$, while guaranteeing 
 \[\mathbb{P} \left(k_{\tau_1 + \tau_\delta} \neq 1 | { H} \right) < \delta, ~ \forall  {H} \mathrm{~a.s} \] 
 is said to be conditionally $\delta$ correct. }
\end{definition}

To keep the discussion simple, for Theorem~\ref{them:cond}, we assume that the distributions in ${\cal S}$ are all Bernoulli with parameter within $(0,1)$. Suppose  the offline data is denoted by ${\bf N^o}$ where $\tau_1= \sum_a N^o_a$.  Let $H$ denote the positive probability event  of seeing the reward history $\{x_{a,\ell}\}_{\ell=1}^{N^o_a},~\forall a \in [K]$. 

\begin{theorem} \label{them:cond}
 A conditionally $\delta$-correct policy has to satisfy the following inequality for all $H$ a.s.
 \begin{equation} \label{eqn:cond_Ineq}
      \sum \nolimits_{a} \mathbb{E}[N_a|H] \mathrm{KL}(\mu_a,\nu_a) \geq d(\delta,1-\delta),~ \forall \nu: \argmax\limits_{a} ~ \nu_a \ne \argmax\limits{b} ~\mu_b.
\end{equation}
\end{theorem}

Observe that the lower bound on expected number of samples required by a purely-online problem is given by the optimal value of the following problem:
\[ \min ~\sum\limits_{a=1}^K N_a ~\text{ s.t. } \sum_a N_a \KL(\mu_a, \nu_a) \ge d(\delta, 1-\delta) \ge \log\frac{1}{\delta}, \quad \forall \nu: \argmax_{a}~\nu_a \ne \argmax_b~\mu_b. \]
Theorem~\ref{them:cond} then implies that a conditionally $\delta$-correct algorithm would require at least as many samples as a purely-online algorithm would need. We now prove Theorem \ref{them:cond}. Intuitively, since the measures conditioned on offline data differ only on the online data, the  LHS in (\ref{eqn:cond_Ineq}) follows for any offline data $H$. The RHS follows because of the stringent demands we put on the 
 conditional $\delta$-correct policy.

\begin{proof}
Consider the following filtered conditional probability space 
\[(\{H\} \times \Omega, {\cal F}, \left({\cal F}_{\tau_1+t}\right)_{t \in \mathbb{N}}, \mathbb{P}_{\mu}(\cdot |  H)).\]
Here, $\Omega$ corresponds to the space of all possible online outcomes, ${\cal F}$ denotes the sigma algebra corresponding to offline and online outcomes. $\{{\cal F}_t\}$ denotes the filtration, with ${\cal F}_t$ capturing the information from first $t$ samples from the sequence of offline plus online outcomes. See \cite{lattimore2020bandit} for technical details for the probabilistic structure of bandit models.

It is clear, even conditioned on the event $H$ from the past, $\{\tau_{\delta} \leq t\} \in {\cal F}_{\tau_1+t},~\forall t \in \mathbb{N}$, where $\tau_\delta$ is an online sampling stopping time. 
We recall that since Bernoulli variables have bias bounded away from $0$ and $1$, $\mathbb{P}_{\mu}(H) >0$ and our event has nonzero measure under any bandit instance being considered. Let 
\[C_{\delta}= (a_{\tau_1+1},x_{\tau_1+1}, \ldots a_{\tau_1+\tau_{\delta}},x_{\tau_1+\tau_{\delta}}) \] 
be the set of online samples till the online stopping time. Consider an alternate measure $\mathbb{P}_\nu$ corresponding to the bandit instance $\nu$. Then, since $\tau_\delta$ is a valid stopping time with respect to the conditional filtered space, we have the following data-processing inequality with measure change between two conditional measures $\mathbb{P}_{\mu}(|H)$ and $\mathbb{P}_{\nu}( | H)$ (see \cite[Lemma 1]{kaufmann2016complexity}).
 \begin{align}
     \KL \left( \mathbb{P}_{\mu}(H,{C}_{\delta} | H), \mathbb{P}_{\nu}(H,{ C}_{\delta} | H) \right) \geq \KL ( \mathbb{P}_{\mu}({\cal E} | H), \mathbb{P}_{\nu}({\cal E} | H) ),
 \end{align}
 where ${\cal E}$ is any event in ${\cal F}_{\tau_1+\tau_{\delta}}$. In particular, choose 
 \[\mathcal E = \{k_{\tau_1 + \tau_\delta} \neq 1\}.\] 
 Further let $\nu$ an alternate Bernoulli MAB instance such that $\argmax_a \nu_a \neq 1$ (recall that $\argmax_a \mu_a = 1$).  Because our policy is conditionally-$\delta$ correct (Definition~\ref{def:cond_delta}), we have:
 \[\mathbb{P}_{\mu}(k_{\tau_1 + \tau_\delta} \neq 1 | H) < \delta \quad \text{ and } \quad  \mathbb{P}_{\nu}(k_{\tau_1 + \tau_\delta} \neq 1 | H) >1- \delta.\] 
 Therefore, we have
  \begin{align}\label{eq:cond_KL}
     \KL \left( \mathbb{P}_{\mu}(H,{C}_{\delta} | H), \mathbb{P}_{\nu}(H,{ C}_{\delta} | H) \right) \geq d(\delta,1-\delta),
 \end{align}
 where $d(p,q)$ denotes the $\KL$ divergence between Bernoulli distributions with means $p$ and $q$, respectively.
 
 Since, the realizations are identical under both the measures and the policy (offline and online) are identical, (\ref{eq:cond_KL}) becomes:
   \begin{align}
     & \KL \left( \mathbb{P}_{\mu}(H | H), \mathbb{P}_{\nu}(H| H) \right) + \KL \left( \mathbb{P}_{\mu}({C}_{\delta} | H), \mathbb{P}_{\nu}({C}_{\delta}| H) \right)
      \geq d(\delta,1-\delta) \nonumber \\
    \overset{a}{\Rightarrow} & ~~ \KL \left( \mathbb{P}_{\mu}({C}_{\delta} | H), \mathbb{P}_{\nu}({C}_{\delta}| H) \right)
      \geq d(\delta,1-\delta).
 \end{align}
 (a) follows because $\mathbb{P}(H | H) =1$ under both Bernoulli instances $\mu$ and $\nu$ since it is an event of positive probability under both measures. 
 
 Since the policy is identical and given the past ${\cal H}_0$ future rewards from any arm $a$ are sampled independently and identically under both measures, we have: 
   \begin{align}
     \KL \left(  \mathbb{P}_{\mu}(C_\delta | H), \mathbb{P}_{\nu}(C_\delta| H) \right) = \sum_a \mathbb{E}[N_a|H] \KL(\mu_a,\nu_a) ,
   \end{align}
   where $N_a$ denotes the number of observations from arm $a$ in $C_\delta$. Substituting  this in (\ref{eq:cond_KL}) we have:
  \begin{align*}
      \sum_a \mathbb{E}[N_a|H] \KL(\mu_a,\nu_a) \geq d(\delta,1-\delta) ,
  \end{align*}
  proving the result.
\end{proof}

\subsection{Proofs from Section \ref{Sec:lb.prop}}\label{app:lb.prop}

\begin{proof}[Proof of Theorem \ref{th:optsol.char}]
The Lagrangian for the convex programming problem ({\bf P1}) is 
    \begin{align*}
     L(\lambda, {\bf N})
     &= \sum_a N_a -\sum_{a}\gamma_a N_a\\
    & -
      \sum_{a \ne 1} \lambda_a \left((\mathbb{E}[N_1^o]+ N_1) \KL(\mu_1, x_{1,a}) + (\mathbb{E}[N_a^o]+ N_a) \KL(\mu_a, x_{1,a}) - 1\right ) ,
     \end{align*}
    where recall that for  fixed ${\bf N}$,
    \[ x_{1,a} = \frac{ (\mathbb{E}[N^o_1]+N_1)\mu_1 + (\mathbb{E}[N^o_a] + N_a ) \mu_a }{N^o_1 + N_1 + N^o_a + N_a}. \]

    Then, ${\bf \tilde{N}}$ satisfies the first order conditions.
    These are  (differentiating w.r.t. $N_1$)
     \begin{equation} \label{eqn:KKT1}
     1= \sum_{a \ne 1} \lambda_a \KL(\mu_1, \tilde{x}_{1,a})
     +\gamma_1,
     \end{equation}
     (differentiating w.r.t. $N_a$, for each a)
     \begin{equation} \label{eqn:KKT2}
    1- \lambda_a  \KL(\mu_a, \tilde{x}_{1,a}) - \gamma_a=0, 
     \end{equation}
     where  for each $a \ne 1$
     \[
     \tilde{x}_{1,a} = \frac{(\mathbb{E}[N^o_1]+\tilde{N}_1) \mu_1+ (\mathbb{E}[N^o_a]+ \tilde{N}_a) \mu_a}
     {\mathbb{E}[N^o_1]+\tilde{N}_1 +\mathbb{E}[N^o_a] + \tilde{N}_a}.
     \]

    ${\bf \tilde{N}}$ also satisfies complimentary slackness. That is, 
    \[
    \lambda_a \left( 
    (\mathbb{E}[N_1^o]+ \tilde{N}_1) \KL(\mu_1, \tilde{x}_{1,a}) + (\mathbb{E}[N_a^o]+ \tilde{N}_a) \KL(\mu_a, \tilde{x}_{1,a}) - 1\right )
    =0
     \]
     for all $a \ne 1$, $\gamma_a \tilde{N}_a=0$ for all $a$. Further, $\gamma_a, \lambda_a \geq 0$. 
    
    For $a \in A_1$, $\gamma_a=0$,
    \[
    \lambda_a =\frac{1}{KL(\mu_a, \tilde{x}_{1,a})}.
    \]
    Along $A_2$ we have
    \[
    \lambda_a =\frac{1-\gamma_a}{\KL(\mu_a, \tilde{x}_{1,a})}.
    \]
    Therefore,  $0 \leq \gamma_a \leq 1$.
    
    Clearly, $\lambda_a=0$ for $a \notin   A_1\cup A_2$.
    Then, (\ref{eqn:KKT1}) equals
    \begin{equation}\label{eq:dualcond}
            \sum_{a \in A_1} \frac{\KL(\mu_1, \tilde{x}_{1,a})}{\KL(\mu_a, \tilde{x}_{1,a})}
    + \sum_{a \in A_2} \frac{\KL(\mu_1, \tilde{x}_{1,a})}{\KL(\mu_a, \tilde{x}_{1,a})}(1-\gamma_a)
     +\gamma_1=1.
    \end{equation}

    Thus, (\ref{eqn:optsol_1}) follows. 
    Further, if $\tilde{N}_1>0$, $\gamma_1=0$ and (\ref{eqn:optsol_2}) follows. 

To see the uniqueness
of ${\bf \tilde{N}}$, suppose that there 
are two optimal solutions ${\bf \tilde{N}}$ and
${\bf \hat{N}}$. Recall the definitions of 
$A_1, A_2$ and  $A=A_1\cup A_2$.
First suppose that 
$\tilde{N}_1= \hat{N}_1$. Then, 
for all $a \in A$, $\hat{N}_a \geq 
\tilde{N}_a$ and hence $\hat{N}_a = 
\tilde{N}_a$ since ${\bf \hat{N}}$ is optimal.
Further, outside of $A$ we
have $\hat{N}_a = 
\tilde{N}_a=0$.
 
Now suppose that $\tilde{N}_1 > \hat{N}_1 > 0$.
Clearly, for each $a \in A, \hat{N}_a > \tilde{N}_a \geq 0$. By the optimality condition for ${ \bf \tilde{N}}$, we have
    \[
    \sum_{a \in A_1}
    \frac{\KL(\mu_1,\tilde{x}_{1,a})}{\KL(\mu_a,\tilde{x}_{1,a})}
    \leq 1
    \]
    and 
    \[
    \sum_{a \in A}
    \frac{\KL(\mu_1,\tilde{x}_{1,a})}{\KL(\mu_a,\tilde{x}_{1,a})}
    \geq 1.
    \]
    Again, since $ \hat{N}_1
    < \tilde{N}_1$, for $a\in A$ we 
    have $\hat{N}_a > \tilde{N}_a \ge 0$. Let 
    \[
    B_1= \{a \notin A, \ne 1: \hat{N}_a > 0\} 
    \]
    and 
    \[
    B_2= \{a \notin A, \ne 1 : \hat{N}_a = 0
    \mbox{  and the corresponding index constraint is tight.} \} 
    \]
    Let $B= B_1 \cup B_2$.
    Then, due to optimality of ${\bf \hat{N}}$, 
    \[
    \sum_{a \in A \cup B_1}
    \frac{\KL(\mu_1,\hat{x}_{1,a})}{\KL(\mu_a,\hat{x}_{1,a})}
    \leq 1,
    \]
    and 
    \[
    \sum_{a \in A \cup B}
    \frac{\KL(\mu_1,\hat{x}_{1,a})}{\KL(\mu_a,\hat{x}_{1,a})}
    \geq 1,
    \]
    where
    \[
     \hat{x}_{1,a} = \frac{(\mathbb{E}[N^o_1]+\tilde{N}_1) \mu_1+ 
     (\mathbb{E}[N^o_a]+ \hat{N}_a) \mu_a}
     {\mathbb{E}[N^o_1]+\hat{N}_1  +\mathbb{E}[N^o_a]+ \hat{N}_a}.
     \]
    First suppose that $B=\emptyset$.
    Then from above,
    \begin{equation} \label{eqn:uhi}
    \sum_{a \in A}
    \frac{\KL(\mu_1,\hat{x}_{1,a})}{\KL(\mu_a,\hat{x}_{1,a})}
    = 1.
    \end{equation}
This leads to 
    a contradiction as 
    $ \mathbb{E}[N^o_1] +\hat{N}_1 < \mathbb{E}[N^o_1]+ \tilde{N}_1$
    and
    $\mathbb{E}[N^o_a]+ \tilde{N}_a < \mathbb{E}[N^o_a] +\hat{N}_a$
     so that
    $\tilde{x}_{1,a} > \hat{x}_{1,a}$
    for each $ a \in A$. This implies that 
    \[
    \frac{\KL(\mu_1,\tilde{x}_{1,a})}{\KL(\mu_a,\tilde{x}_{1,a})}
    < \frac{\KL(\mu_1,\hat{x}_{1,a})}{\KL(\mu_a,\hat{x}_{1,a})}.
    \]
    Therefore
    the LHS in (\ref{eqn:uhi})
    strictly dominates 
    \[
    \sum_{a \in A}
    \frac{\KL(\mu_1,\tilde{x}_{1,a})}{\KL(\mu_a,\tilde{x}_{1,a})}.
    \]
    But the latter is $ \geq 1$
    providing the desired contradiction. 
  
  Now suppose that   $B$ is not empty.
   Again, the contradiction follows similarly
as
\[
\sum_{a \in A \cup B_1}
\frac{\KL(\mu_1,\hat{x}_{1,a})}{\KL(\mu_a,\hat{x}_{1,a})}
\leq 1
\]
and
\[
\sum_{a \in A}
\frac{\KL(\mu_1,\tilde{x}_{1,a})}{\KL(\mu_a,\tilde{x}_{1,a})}
\geq 1.
\]

Next, suppose that $\tilde{N}_1 > \hat{N}_1 = 0$. Define the sets
\[ \tilde{A}_1 := \lrset{a\ne 1: \tilde{N}_a > 0}, \quad\text{and}\quad \hat{A}_1 := \lrset{a\ne 1: \hat{N}_a > 0}.\]
Similarly, 
\[ \tilde{A}_2 := \lrset{a\ne 1: \tilde{N}_a = 0, ~ S_{1,a}(\tilde{\bf N}) = \log\frac{1}{2.4\delta}},\] and
\[ \hat{A}_2 := \lrset{a\ne 1: \hat{N}_a = 0, ~ S_{1,a}(\hat{\bf N}) = \log\frac{1}{2.4\delta}}.\]

Define $\tilde{A} = \tilde{A}_1 \cup \tilde{A}_2$ and $\hat{A} = \hat{A}_1 \cup \hat{A}_2$. Clearly, $ \tilde{A}_1 \subset  \hat{A}_1$. Moreover, $\tilde{A}_1 \cup \tilde{A}_2 \subset \hat{A}_1$. From the optimality conditions, we have the following: 

\[ \sum\limits_{a\in \tilde{A}_1} \frac{\KL(\mu_1, \tilde{x}_{1,a})}{\KL(\mu_a, \tilde{x}_{1,a})} \le 1, \quad \sum\limits_{a\in\tilde{A}} \frac{\KL(\mu_1, \tilde{x}_{1,a})}{\KL(\mu_a, \tilde{x}_{1,a})} \ge 1, \quad \text{and}\quad \sum\limits_{a\in \hat{A}_1} \frac{\KL(\mu_1, \hat{x}_{1,a})}{\KL(\mu_a, \hat{x}_{1,a})} \le 1.\]

Now, for $a\in \hat{A}_1$, $ 0\le \tilde{N}_a < \hat{N}_a$. This implies that $\tilde{x}_{1,a} > \hat{x}_{1,a}$, which further implies that 
\[ \frac{\KL(\mu_1, \tilde{x}_{1,a})}{\KL(\mu_a, \tilde{x}_{1,a})} < \frac{\KL(\mu_1, \hat{x}_{1,a})}{\KL(\mu_a, \hat{x}_{1,a})}, \quad \forall a\in \hat{A}_1, \] 
giving 
\[ \sum\limits_{a\in\hat{A}_1} \frac{\KL(\mu_1, \tilde{x}_{1,a})}{\KL(\mu_a, \tilde{x}_{1,a})} < \sum\limits_{a\in\hat{A}_1} \frac{\KL(\mu_1, \hat{x}_{1,a})}{\KL(\mu_a, \hat{x}_{1,a})}.  \]
However, from the optimality conditions for $\tilde{N}$ and $\hat{N}$, the  l.h.s. above is at least $1$, while the r.h.s. is at most $1$, contradicting the strict inequality above.

This completes the proof for the necessity of the conditions in the theorem for optimality. To see that these are also sufficient, one can argue that if $\bf N$  satisfies the conditions given by the theorem, then there are feasible dual variables $\lambda_a$ and $\gamma_a$ such that the KKT conditions in~\eqref{eqn:KKT1}-~\eqref{eq:dualcond} hold, proving the optimality of $\bf N$.
\end{proof}

\subsubsection*{Properties of lower bound}
Recall that 
\begin{equation}
A_{\delta}(\mu) = \lrset{{\bf N} \in \Re^K_+ : \inf_{x} \{ N_{1} \KL(\mu_1, x) + N_{a} \KL(\mu_a, x)\} \geq \log\frac{1}{2.4\delta}, \quad \forall  a\ne 1}.  \label{eq:Amudelta}\end{equation}
The infimum in the index constraint for arm $a\ne 1$ is attained at $x_{1,a}$ given by 
\[ x_{1,a} = \frac{N_1 \mu_1 + N_1 \mu_a}{N_1 + N_a}. \]
Moreover, $x_{1,a} \rightarrow \mu_1$ as $N_1 \rightarrow \infty$ and $x_{1,a} \rightarrow \mu_a$ as $N_a \rightarrow \infty$. Now, consider again  the observations made in Section~\ref{Sec:lb.prop}:
\\
\textbf{(a)} If $N_1^o >  N^*_1,$ then $\widetilde{N}_1= 0$.
\\
\textbf{(b)} For 2-armed bandit problems, if $N_2^o >  N^*_2,$ then $\widetilde{N}_2= 0$. But this need not be true when $K>2$.
\\
\textbf{(c)} For ${\bf N^o}$ to lie in the constraint set $A_{\delta}(\mu)$, we need \[N_1^o > \max_{a \in [K]\setminus{\{1\}}} \frac{\log\frac{1}{2.4 \delta}}{\KL(\mu_1, \mu_a)}.\] 
Similarly, for each $a\in [K]\setminus{\{1\}}$, we require 
\[N_a^o > \frac{\log\frac{1}{2.4 \delta}}{\KL(\mu_a, \mu_1)}.\]

\textit{Arguments supporting the observations. } To see the first observation above, first recall that for the optimal solution of the purely-online problem $\bf N^*$, we have from the optimality conditions that 
\[ \sum\limits_{a\ne 1} \frac{\KL(\mu_1,{x}^*_{1,a})}{\KL(\mu_a,{x}^*_{1,a})}
= 1, \quad \text{ where }\quad x^*_{1,a} = \frac{N^*_1\mu_1 + N^*_a\mu_a}{N^*_1 + N^*_a}.\]
Now, consider 
${\bf N} = (N_1^o,N_2^*, \ldots, N_K^*)$ with $N^o_1 > N^*_1$. 
For this, we have
\[
\sum_{a \in [K]\setminus{\{1\}}}
\frac{\KL(\mu_1,{x}_{1,a})}{\KL(\mu_a,{x}_{1,a})}
< 1, \quad \text{ where }\quad x_{1,a} = \frac{N^o_1\mu_1 + N^*_a\mu_a}{N^o_1 + N^*_a}.
\]
This follows since the LHS of the {\em sum-ratio} above is decreasing in $N_1$. 

Now, recall the optimality conditions for our o-o framework (Theorem~\ref{th:optsol.char}). For arms $a \in A_1$, since the corresponding index
constraints are tight, $N_a(N_1^o)$ (total offline + online samples to arm $a$) that solves the index equality are less than $N_a^*$ (since $N^o_1 > N^*_1$ and index is non-decreasing in $N_1$). Hence, at $N_a(N^o_1)$, the above {\em sum-ratio} inequality continues to hold since it is non-decreasing in $N_a$, and thus the optimality conditions are satisfied with $\tilde{N}_1 = 0$.

The second observation follows immediately for $K=2$. To see that it does not hold more generally, consider the case where $N_a^o=\infty$ for $a \geq 3$, $N_1^o=0$. Let $N_2(N_1)$ solve the index constraint corresponding to arm 2 for a given $N_1$. Suppose $ N_2^o \in (N_2^*, \hat{N_2}(N_1))$ where $\hat{N_2}(N_1)$ solves \[
\frac{\KL(\mu_1,{x})}{\KL(\mu_2,{x})}=1, \quad \text{ where }\quad  x = \frac{N_1 \mu_1+ \hat{N}_2(N_1) \mu_2}{N_1  + \hat{N}_2(N_1)},\]
and is non-decreasing in $N_1$.

Observe that for this setup of offline data, the optimality conditions for the online sampling in Theorem~\ref{th:optsol.char} reduce to $A_1 = \lrset{2}$, $A_2 = \emptyset$, and
\[ \frac{\KL(\mu_1, x_{1,2})}{\KL(\mu_2, x_{1,2})} = 1, \quad \text{ where } \quad x_{1,2} = \frac{ N_1 \mu_1 + N_2 \mu_2  }{N_1 + N_2}, \]
where $N_1 = \tilde{N}_1$ (since $N^o_1 = 0$) and $N_2 = N^o_2 + \tilde{N}_2$. This implies that $\hat{N}_2(N^*_1)$ defined earlier, is at least  $N^*_2$. It now follows that for $\tilde{N}_1 < N^*_1$, $\hat{N_2}(\tilde{N}_1) > N_2^*$, and $\tilde{N}_2= \hat{N_2}(\tilde{N}_1) - N_2^o > 0$ satisfy the optimality conditions. 

The third observation follows from  the
index constraint in~\eqref{eq:Amudelta} and the observation that for $N_a^o \rightarrow \infty$, $x_{1,a} \rightarrow \mu_a$ and similarly for $N^o_1$.

\section{$\delta$-correctness of the algorithm}\label{sec:delta_correctness}
In this section, we briefly present a proof for $\delta$-correctness of the proposed algorithm. 
\begin{proof}[Proof of Theorem \ref{thm:delta-correct}] Recall that we assume that arm $1$ is the unique arm with the maximum mean in $\mu$. An error occurs if the algorithm's estimate for the best-arm at the stopping time is not arm $1$. It is well known that a bandit algorithm using GLRT (or an upper bound on it) with an appropriate choice of the stopping threshold as a stopping rule, is $\delta$-correct (see, \cite{garivier2016optimal, Kauffman_21, thesis} for GLRT-based stopping statistics in different settings). For completeness, we briefly outline a proof relating the error event to the deviation of empirical $\KL$ divergences. \cite{Kauffman_21} show that the obtained deviation is a rare event, establishing the  $\delta$-correctness of the algorithm.

Recall that the stopping rule corresponds to checking 
\[ \min\limits_{b\ne i^*(t)} Z_{i^*(t),b}({\bf N}(t))  \ge \beta(t+\tau_1,\delta),\]
where
\[ Z_{a,b}({\bf N}(t)) = \inf\limits_{x} \lrset{(N_a^o+ N_a(t))\KL(\hat{\mu}_a(t), x) + (N_b^o+N_b(t))\KL(\hat{\mu}_b(t), x)}. \]

From Lemma \ref{lem:optima_SPEF}, $Z_{a,b}({\bf N}(t))$ also equals
\[ \inf\limits_{x \le y} \lrset{(N_a^o+ N_a(t))\KL(\hat{\mu}_a(t), x) + (N_b^o+N_b(t))\KL(\hat{\mu}_b(t), y)}. \]

Consider the error event given as below:
    \[ \mathcal E = \lrset{\tau_\delta < \infty, i^*(\tau_\delta) \ne 1 }.  \]
    
    The above event is contained in 
    \[\lrset{ \exists t \in \mathbb{N}, \exists a \ne 1, ~ \min\limits_{b\ne a} Z_{a,b}({\bf N}(t)) \ge \beta(t+\tau_1,\delta)},\]
    which is further contained in 
    \[\lrset{ \exists t \in \mathbb{N}, \exists a \ne 1, ~ Z_{a,1}({\bf N}(t)) \ge \beta(t+\tau_1,\delta)}.\]
    Thus, to bound the probability of error, it suffices to bound the probability of the above event. Clearly, $x=\mu_a$ and $y=
    \mu_1$ are feasible choices for the variables being optimized in $Z_{a,1}({\bf N}(t))$, giving an upper bound on $Z_{a,1}({\bf N}(t))$ in terms of scaled sums of the two $\KL$-divergence terms.
    \cite[Equation (25), Section 5.1]{Kauffman_21} then bounds the probability of this deviation by $\delta$, showing that the proposed algorithm using the GLRT stopping rule described above with the threshold $\beta(t+\tau_1,\delta)$ specified in \eqref{eq:beta} is $\delta$-correct. 
\end{proof}

\section{Max-min normalized formulation and properties of the optimizers}

Recall that the given bandit instance $\mu \in \mathcal S^K$ is such that arm $1$ is the unique optimal arm,  $\tau_\delta$ denotes the total number of online samples generated, which is allowed to be a function of all the observations, $N^o_a$ denotes the total number of samples from arm $a$ in the offline data, and $N_a(\tau_\delta)$ denotes the total number of samples from arm $a$ from the online sampling till time $\tau_\delta$. Moreover, recall that  $p\in\Sigma_K$ denotes the fraction of samples from each arm in the offline data. For $q_1\in\Re$ and $q_2\in\Re$, $\KL(q_1,q_2)$ denotes the $\KL$ divergence between the unique probability distributions in $\cal S$ with means $q_1$ and $q_2$. With this notation, for $z\in [0,1]$, $w\in\Sigma_K$, $x \in \Re$, $j\in [K]$ and $j\ne 1$, 
    $$g_j(w,z,x,p) = (z p_1 + (1-z)w_1)\KL(\mu_1, x) + (z p_j + (1-z)w_j) \KL(\mu_j, x).$$ 

In the sequel, for the simplicity of notation, we remove the dependence on $p$ of the various functions. Define
    \begin{align}\label{eq:V_def}
     V(\mu,z)= \max\limits_{w\in\Sigma_K}~\min\limits_{j\ne 1}~ \inf\limits_{x} ~ g_j(w,z,x). 
     \end{align} 
    Moreover, recall that $z^*$ is the optimal value of the following optimization problem ($\mathbf{P3}$) from the main paper. 
    \begin{align*}
        \max ~~~ &z \\
        \text{s.t.}~~~ &V(\mu,z) \ge \frac{z}{\tau_1} \left(\log\frac{1}{\delta} + \log \log \frac{1}{\delta} \right).
    \end{align*}

\subsection{Alternative formulation}\label{sec:EquivForm}
\begin{proof}[Proof of Lemma \ref{lem:equivalence}]
    Let $p \in \Sigma_K $ denote the fraction of samples of each arm in the offline data, $z \in [0,1]$ denote the fraction of offline samples, and let $w \in \Sigma_K$ denote the fraction of online samples from each arm, i.e., 
    $$p_a = \frac{N^o_a}{\tau_1},\quad z = \frac{\tau_1}{\tau_1 + \sum\limits_{a=1}^K N_a},\quad  w_a = \frac{N_a}{\sum\limits_{b=1}^K N_b}.$$
    Let $N = \sum_b N_b.$ Then, all the left hand side constraints in (\ref{prob:P2}) in $\bf P2(\mu)$ re-write as: 
    
    \begin{align*} 
    (\tau_1 + N) &\lrp{ \inf \limits_{a \neq 1} \inf \limits_{x_{1,a}} (z p_1 + (1-z) w_1) \KL(\mu_1, x_{1,a}) +  (zp_a + (1-z) w_a) \KL(\mu_a, x_{1,a})}\\
    &\geq \log \frac{1}{\delta} + \log \log \frac{1}{\delta}. 
    \end{align*}
 
 Now, we can rewrite problem $\mathbf{P2}(\mu)$ equivalently as (since $\tau_1$ is a constant and using the definition of $z$):
 \begin{align}
     \max & ~~~~z \nonumber\\
     \mathrm{s.t~} &  \lrp{ \inf \limits_{a \neq 1} \inf \limits_{x_{1,a}} (z p_1 + (1-z) w_1) \KL(\mu_1, x_{1,a}) +  (zp_a + (1-z) w_a) \KL(\mu_a, x_{1,a})}\nonumber\\
     &\qquad\qquad \qquad\qquad \geq \frac{z}{\tau_1}\left(\log \frac{1}{\delta} + \log \log \frac{1}{\delta}\right)  \nonumber \\
    & ~~z \in [0,1], ~ w \in \Sigma_{K}.
 \end{align}
 This is equivalent to (just by notational substitution of $g_a(\cdot)$):
 \begin{align}\label{equiv1}
 \max & ~~z \nonumber\\
\mathrm{s.t~} & \min \limits_{a \neq 1} \inf \limits_x g_a(w,z,x) \geq \frac{z}{\tau_1}\left(\log \frac{1}{\delta} + \log \log \frac{1}{\delta}\right)  \nonumber \\
& ~z \in [0,1], ~ w \in \Sigma_{K}.
\end{align}

Suppose there is a feasible point $z, w$ to problem (\ref{equiv1}). Then, $z$ satisfies the following constrained problem as well:

 \begin{align}\label{equiv2}
   \max & ~~z \nonumber\\
\mathrm{s.t~} & \sup \limits_{w \in \Sigma_K} \min \limits_{a \neq 1} \inf \limits_x g_a(w,z,x) \geq \frac{z}{\tau_1}\left(\log \frac{1}{\delta} + \log \log \frac{1}{\delta}\right)  \nonumber \\
& ~z \in [0,1].
 \end{align}
 
 This implies that optimal $z^{*}$ of (\ref{equiv2}) at least optimal $z^{*}$ of (\ref{equiv1}) since the feasible set is bigger in (\ref{equiv2}). In the reverse direction, if $z^{*}$ is an optimal solution to (\ref{equiv2}), consider the corresponding maximizer $w^{*}(\mu,z^{*})$ in the constraint. Then, $z^{*},w^{*}(\mu,z^{*})$  is a feasible solution to (\ref{equiv1}). Therefore, problem in (\ref{equiv1}) is same as problem in (\ref{equiv2}). (\ref{equiv2}) is precisely problem $\mathbf{P3}(\mu)$. 
 
 Thus, suppose we have optimal solutions $(z^*, w^*)$ for ${\bf P3(\mu)}$, then  $$N^*_a = w^*_a \tau_1 \lrp{\frac{1}{z^*}-1}$$ gives optimal solution for $\bf P2(\mu)$. Similarly, let  ${\bf N^*} = \lrset{N^*_1, \dots, N^*_K}$ be optimizers for ${\bf P2}(\mu)$. Then, $$w^*_a = \frac{N^*_a}{\sum_b N^*_b}  \quad \text{and }\quad z^* = \frac{\tau_1}{\tau_1 + \sum_b N^*b}$$
 gives the optimizers for ${\bf P3}(\mu)$.
 This completes the proof.
\end{proof}

\subsection{Properties of the max-min problem and its optimizers: monotonicity}\label{app:mon_propmaxmin}

As in the previous sections, in this section we fix $p$ to the observed fractions of each arm in the offline data. Since $p$ is fixed, we supress the dependence of the various functions in the sequel on $p$.
    
\begin{lemma}\label{lem:monotonicVmuz}
 For $\mu \in \mathcal S^K$ with arm $1$ being the unique arm with the maximum mean, $V(\mu,z)$ defined in \eqref{eq:V_def} is non-increasing in the second argument.
 \end{lemma}
 \begin{proof}
    Recall that 
    $$
        V(\mu,z) = \max\limits_{w\in\Sigma_K}~ \min\limits_{j\ne 1} ~ \inf\limits_{x} ~\lrset{ (zp_1 + (1-z)w_1) \KL(\mu_1, x) + (zp_j + (1-z)w_j)\KL(\mu_j, x) },
    $$
    where $p \in \Sigma_K$. For $z\in [0,1]$, define $\Sigma^z_K = \lrset{z p + (1-z) w :~ w\in\Sigma_K }$. With this notation, 
    \[ V(\mu,z) = \max\limits_{\tilde{w}\in \Sigma^z_K }~ \min\limits_{j\ne 1} ~ \inf \limits_x ~ \lrset{ \tilde{w}_1 \KL(\mu_1, x) + \tilde{w}_j \KL(\mu_j, x) }. \]
    Moreover, for $z_1 < z_2$, $\Sigma^{z_2}_K \subset \Sigma^{z_1}_K $, implying that $V(\mu, z_1) \ge V(\mu, z_2)$.
    \end{proof}

\subsection{Properties of the max-min problem and its optimizers: continuity results for fixed $\delta$ }\label{app:propmaxmin}

For $z\in [0,1]$ and $\nu\in \mathcal I^K$, let $w^*(\nu, z)$ denote the set of optimizers for $V(\nu, z)$. Clearly, $w^*: \mathcal I^K \times [0,1] \rightarrow \Sigma_K$. The following lemma shows that these optimizers satisfy some continuity properties. As we will see later, this will be crucial for proving the convergence of the  proposed plug-and-play strategy.

\begin{lemma}\label{lem:cont.uniq}
For $\mu \in  \mathcal S^K$, the set of optimizers $z^*(\mu)$ and $w^*(\mu, z^*)$ are respectively continuous and upper-hemicontinuous in their respective arguments. Moreover, $V(\cdot, \cdot)$ is a jointly continuous function. In addition, if $\mu$ has a unique optimal arm, then $w^*(\mu,z^*)$ is also jointly continuous. 
\end{lemma}

\begin{proof} 
The proof of the above lemma proceeds by applying the classical Berge's Theorem (see, \cite{berge1997topological}) at various steps. Towards this, we first prove the continuity of functions being optimized and establish the properties of the feasible regions as a function of the the bandit instance. An application of the Berge's Theorem then gives the desired result. 

The upper-hemicontinuity of $w^*$ follows from Lemma \ref{lem:contVmu}. Moreover, if $\mu$ has a unique optimal arm, the set of maximizers for $V(\mu,z)$ is unique (see, Theorem \ref{th:optsol.char} and Lemma \ref{lem:equivalence}). This then gives the joint-continuity of $w^*$ in its arguments.

\noindent{\bf Continuity of $z^*$: }Consider the set 
    $\mathcal Z(\mu) := \lrset{z\in [0,1] : ~ V(\mu, z) \ge \frac{z}{\tau_1} (\log\frac{1}{\delta} + \log \log \frac{1}{\delta})}.$
    
Recall that $z^*(\mu) = \max\lrset{ z : ~ z \in \mathcal{Z}(\mu) }$. To prove continuity of the optimal value of this optimization problem, first observe that the objective function is independent of $\mu$, hence continuous in $\mu$. It now suffices to show that $\mathcal Z(\cdot)$ is both a lower and upper hemicontinous correspondence, hence continuous correspondence. Berge's Theorem (see, \cite{berge1997topological}) then gives the continuity of $z^*$ in $\mu$. Note that lower- and upper-hemicontinuity of $\mathcal Z(\cdot)$ follows from continuity of $V(\mu,z)$ in $\mu$ (Lemma \ref{lem:contVmu}). This follows from the sequential characterization of upper and lower hemicontinuity (see, Section 9.1.3 in \cite{sundaram1996first} ). 
\end{proof}

Let us now prove the results that we used in the proof of the above continuity-lemma.

For $j\ne 1$, recall the definitions of $g_j(\cdot, \cdot, \cdot, \cdot)$ and $V(\cdot, \cdot)$. Define 
\begin{equation}\label{eq:Gj}
    G_j(\mu,w,z) = \inf\limits_x ~g_j(\mu,w,z,x).
\end{equation}
\begin{lemma}\label{lem:contGj}
    $G_j: \mathcal S^K \times \Sigma_K\times [0,1] \rightarrow \Re^+$ defined in \eqref{eq:Gj} is a jointly continuous function of its arguments.
\end{lemma}

\begin{proof}
    Let $\tilde{w}(\mu,z) \in \Sigma_K$ be given by $zp + (1-z)w$. It is not hard to see that the $x$ achieving the infimum in $G_j$ (denoted by $x_{1,j}$) belongs to the set $M_{1,j}:= [\mu_j, \mu_1]$. This follows from the monotonicity of the two $\KL$ divergence terms in the expression for $g_j$. Hence,  
    \[ G_j(\mu,w,z) = \inf\limits_{x \in M_{1,j}}~ g_j(\mu,w,z,x). \]
    Now, observe that $M_{1,j}$, viewed as a set-valued map from $\mathcal S^K\times \Sigma_K  \times [0,1]$, is jointly continuous and  compact-valued. This follows from the sequential characterization of lower and upper hemicontinuity (see, \cite[Section 9.1.3]{sundaram1996first}). Moreover, $g_j$ is a jointly continuous function of its arguments. Then, Berge's Theorem (see, \cite{berge1997topological}, \cite{sundaram1996first}) gives that $G_j$ is a jointly continuous function. 
\end{proof}

\begin{lemma}\label{lem:contVmu}
    $V: \mathcal S^K \times [0,1] \rightarrow \Re^+$ is a jointly continuous function. Moreover, the set of maximizers in  $V(\cdot,\cdot)$, i.e., $w^*(\mu,z)$ is a jointly upper-hemicontinuous correspondence.
\end{lemma}
\begin{proof}
    Recall that 
    \[ V(\mu,z)  = \max\limits_{w\in\Sigma_K}~ \min\limits_{j\ne 1} ~ G_j(\mu,w,z), \]
    and $w^*$ is the set of maximizers in the $V(\mu,z)$ optimization problem.    From Lemma \ref{lem:contGj}, $G_j$ is a jointly-continuous function. Hence, $\min_j G_j$ is also jointly-continuous in $(\mu,w,z)$. Since $\Sigma_K$ when viewed as a correspondence from $\mu\times [0,1]$ is a constant and compact-valued correspondence, it is  jointly-continuous. Berge's Theorem (see, \cite{berge1997topological}) now implies that $V(\mu,z)$ is a jointly continuous function. Moreover, the set of maximizers, $w^*(\mu, z)$, is a jointly upper-hemicontinuous correspondence.
\end{proof}

\section{Expected sample complexity of the algorithm}\label{app:samplecomplexity}

    \begin{theorem}[Restatement of Theorem \ref{thm:stop_bound1}. Non asymptotic Expected Sample Complexity]\label{thm:mainthm}
        For $\delta > 0$, suppose $\mu\in\mathcal S^K$ is such that $ 1 \ge z^*(\mu) > \eta$, for some $\eta > 0$. Let $\epsilon'>0, \tilde{\epsilon}>0, \epsilon_1>0 $ be constants. If the given  problem instance $(\mu,\hat{p},\tau_1)$ is such that:
        \begin{equation}\label{eq:cond_tau_1}\frac{\tau_1}{\log \frac{1}{\delta} + \log \log \frac{1}{\delta}} \notin (C_1^a(c_{\epsilon_1}),C_2^a(c_{\epsilon',\tilde{\epsilon}}))     \quad \forall a\ne 1,\end{equation}
        then the algorithm satisfies:
        \[ \mathbb{E}[\tau_{\delta}] + \tau_1 \le  \mathbb{E} \left[  \frac{\tau_1}{z^*}\lrp{1 + \frac{2\alpha(\epsilon') + \tilde{\epsilon}}{\epsilon_1 C(\mu,\hat{p},\eta)} } \right] + T(\epsilon') + T(\tilde{\epsilon}) + 1 + o\lrp{\log\frac{1}{\delta}}.\]
        Here, for an instance-dependent constant $L_\mu$, $c_{\epsilon_1} := \epsilon_1 L_{\mu} ({1}/{\eta}-1) $ and $c_{\epsilon',\tilde{\epsilon}}:= \lrp{2 \alpha(\epsilon') + \tilde{\epsilon}}/{\eta}$, where $\alpha(\epsilon')$ is a continuous function of $\epsilon'$ such that $\alpha(\epsilon'){\rightarrow} 0 $ as $\epsilon' \rightarrow 0$. Moreover, $C_1^a(\cdot)$ and $C_2^a(\cdot)$ are functions such that $C^a_1(c) \rightarrow C^a_1(0)$ and $C^a_2(c) \rightarrow C^a_2(0)$ as $c\rightarrow 0$, and $C^a_1(0) = C^a_2(0)$. $C(\mu,\hat{p},\eta)$ is a non-negative function of $\mu, \hat{p}$ and $\eta$ that is strictly positive for $\eta > 0$.
    \end{theorem}

    \textbf{Interpreting Theorem \ref{thm:mainthm}} We restate some discussion points from the main paper as to how to interpret Theorem \ref{thm:mainthm}. We first point out that the RHS is a function of empirical proportions $\hat{p}$ observed. Our upper bound is a function of empirical proportion observed in the offline samples. The outer expectation averages over all offline realizations. So we now show that this achieves the optimal asymptotic rates in the simple case when $\hat{p}=p$ is fixed for all offline realizations by the offline policy.
    
    Consider a sequence of bandit instances $(\mu,p,\tau_1(\delta))$ with $\delta \rightarrow 0$ such that the offline proportions is fixed at $p$. Further, let $\liminf \limits_{\delta \rightarrow 0} z^{*}(\mu,p) > \eta$. 
    
     1) Suppose (Condition $1$) $\tau_1(\delta): \liminf \limits_{\delta \rightarrow 0} \frac{\tau_1(\delta)}{\log \frac{1}{\delta} + \log \log \frac{1}{\delta}} > C_2^a(0) $.  In other words, offline samples available is \textit{just sufficient} for any arm $a$ to be not sampled in the online phase according to the lower bound problem. Then there exists a $\delta_0: \forall \delta < \delta_0(\epsilon',\tilde{\epsilon})$ and small enough $\epsilon', \tilde{\epsilon}$: 
     $ \frac{\tau_1(\delta)}{\log \frac{1}{\delta} + \log \log \frac{1}{\delta}} >  C_2^a(c_{\epsilon',\tilde{\epsilon}})$. Then, the theorem's conditions hold.

    2) Suppose (Condition $2$) $\tau_1(\delta): \limsup \limits_{\delta \rightarrow 0} \frac{\tau_1(\delta)}{\log \frac{1}{\delta} + \log \log \frac{1}{\delta}} < C_1^a(0) $.  In other words, offline samples available is \textit{just not sufficient} for some arm $a$ to be not sampled in the online phase according to the lower bound problem. Then there exists a $\delta_0: \forall \delta < \delta_0(\epsilon_1)$ and small enough $\epsilon_1 >0$: 
     $ \frac{\tau_1(\delta)}{\log \frac{1}{\delta} + \log \log \frac{1}{\delta}} <  C_1^a(c_{\epsilon_1})$. Again, the conditions in theorem holds. 
    
If for a subset of sub optimal arms Condition $1$ occurs, and for the complement amongst the suboptimal arms Condition $2$ occurs.
One takes the minimum of $\epsilon_1, \epsilon', \tilde{\epsilon} $ that are needed for the respective conditions. Then, we first take $\delta \rightarrow 0$ we have:
 \begin{align}
           \limsup \limits_{\delta \rightarrow 0} \frac{\mathbb{E} \left[\tau_{\delta} + \tau_1(\delta) \right]}{\log (1/\delta)}  \leq  \left(1 + \frac{2\alpha(\epsilon') + \tilde{\epsilon}}{\epsilon_1 C(\mu,p,\eta)} \right) \limsup \limits_{\delta \rightarrow 0} \frac{ \frac{\tau_1}{z^*(\mu,p)} }{\log (1/\delta)}.
      \end{align}

Now, we take $\epsilon', \tilde{\epsilon} $ to $0$. Therefore, if for every sub-optimal arm if either Condition $1$ or $2$ is true, we have: 
   \begin{align}
           \limsup \limits_{\delta \rightarrow 0} \frac{\mathbb{E} \left[\tau_{\delta} + \tau_1(\delta) \right]}{\log (1/\delta)}  \leq  \limsup \limits_{\delta \rightarrow 0} \frac{ \frac{\tau_1}{z^*(\mu,p)} }{\log (1/\delta)}.
      \end{align}
      
   \textbf{Varying $p$:}      
   Suppose offline proportions are random. Consider the case when $(\mu,p,\tau_1(\delta)): \lim_{\delta \rightarrow 0} p = p^{*}~ \mathrm{a.s.}, \lim_{\delta \rightarrow 0} z^{*}(\mu,p) = \tilde{z}^{*} > \eta \mathrm{a.s.}$. Then, under same conditions on all the sub-optimal arms in the previous remark, we have: 
     \begin{align}
       \limsup \limits_{\delta \rightarrow 0} \frac{\mathbb{E} \left[\tau_{\delta} + \tau_1(\delta) \right]}{\log (1/\delta)} & \leq \limsup \limits_{\delta \rightarrow 0} \frac{\mathbb{E} \left[ \frac{\tau_1}{z^*(\mu,p)}\lrp{1 + \frac{2\alpha(\epsilon') + \tilde{\epsilon}}{\epsilon_1 C(\mu,p,\eta)} } \right]}{\log (1/\delta)} \nonumber \\
       \hfill & \leq \lrp{1 + \frac{2\alpha(\epsilon') + \tilde{\epsilon}}{\epsilon_1 C(\mu,p^{*},\eta)} } \limsup \limits_{\delta \rightarrow 0} \frac{ \left[ \frac{\tau_1}{\tilde{z}^*} \right]}{\log (1/\delta)}.
     \end{align}

    \begin{proof}[Proof of Theorem \ref{thm:mainthm}]
        In our proof we take $p$ to be the empirical proprotions $\hat{p}$. Our proof extends the analysis of track-and-stop to the current setup with access to offline data as well as tracking the proportions computed in batches, tackling additional complications due to $w^*_a$ being close to $0$ which does not happen in online problems. 
        
        Recall that $\mu$ is a SPEF bandit with arm $1$ being optimal. Fix an $\epsilon' > 0$. By upper-hemicontinuity of the set of maximizers $\tilde{w}$  in the bandit instance and uniqueness of the maximizer (uniqueness follows from Theorem \ref{th:optsol.char}  and Lemma $\ref{lem:equivalence}$.), there exists $\zeta(\epsilon') \le \min\limits_{a \ge 2}\Delta_a/4$ such that the set of means,
        \[ \mathcal I_{\epsilon'} := [\mu_1 - \zeta(\epsilon'), \mu_1 + \zeta(\epsilon')] \times \dots \times [\mu_K - \zeta(\epsilon'), \mu_K + \zeta(\epsilon')] ,  \]
        is such that for all bandit instances, $\mu' \in \mathcal I_{\epsilon'}$
\[ \lVert w^{*}(\mu', z^{*}(\mu')) - w^{*}(\mu, z^{*}(\mu))   \rVert_{\infty} \leq \epsilon' \]
        Recall that for $a\ne 1$, 
        \[ G_a(\mu,w,z) := \inf\limits_x \lrset{ (zp_1 + (1-z)w_1) \KL(\mu_1, x) + (z p_a + (1-z)w_a) \KL(\mu_a, x) }. \]
        From joint continuity of $G_a$ (Lemma \ref{lem:contGj}), it follows that for $\mu' \in \mathcal I_{\epsilon'}$, $\forall z\in [0,1]$ and for all $w' \in\Sigma_K$ such that  $$ \| w' - w^{*}(\mu, z^{*}(\mu)) \| \le 4\epsilon',$$

         \begin{equation}\label{eq:Gabound}
            \max\limits_{a\ne 1}| G_a\lrp{\mu', z, w'} - G_a\lrp{\mu, z, w^{*}(\mu, z^{*}(\mu))} |  \le \alpha(\epsilon'),
        \end{equation}
        
        for some $\alpha(\cdot)$ such that $\alpha(\epsilon')\rightarrow 0$ as $\epsilon' \rightarrow 0$.

        Furthermore, recall that $\hat{\mu}(t)$ is the empirical distribution for the total  $\tau_1+t$ offline and online samples available with the algorithm after $t$ online trials. Furthermore, for each arm $a$, recall that $N_a(t)$ denotes the total  online samples allocated to arm $a$ till time  $t$. Observe that whenever $\hat{\mu}(n) \in \mathcal I_{\epsilon'}$, the empirically-best arm is arm $1$. 
         
        Next, let $T\in \mathbb{N}$, $h(T) = T^{\frac{1}{4}}$, and $h_1(T) = T^{\frac{1}{2}}$. Define the `good set' as the event  
        \[ \mathcal E_T(\epsilon'):= \bigcap\limits_{t= h(T)}^T \lrp{\hat{\mu}(t) \in \mathcal I_{\epsilon'}}. \]
        
        Let $\hat{z}_t$ denote the fraction of offline samples at time $t$, i.e.,  
        \[ \hat{z}_t := \frac{\tau_1}{\tau_1 + t}. \]
        
        On $\mathcal E_T$, for $t \ge h(T)$, the empirically-best arm is arm $1$ and the stopping statistic is 
        \begin{align*} 
            \min\limits_{b\ne 1} Z_{1,b}({\bf N}(t))
            &= \min\limits_{b\ne 1} ~ \inf\limits_{x} ~ \lrset{(N^o_1+N_1(t)) \KL(\hat{\mu}_1(t), x) + (N^o_a + N_a(t))\KL(\hat{\mu}_a(t), x) }\\
            &= ( \tau_1 + t) ~ \min\limits_{b\ne 1} ~ G_b\lrp{\hat{\mu}(t), \hat{z}_t, \lrset{\frac{N_a(t)}{t}}_a},
        \end{align*} 
        where, recall that for $\nu\in\mathcal S^K$ with arm $1$ being the unique best arm,
        $$ G_b(\nu,z,w) = \inf\limits_{x} \lrset{(zp_1 + (1-z)w_1) \KL(\nu_1,x) + (zp_b + (1-z)w_b) \KL(\nu_b,x) } .$$
        
        From Lemma \ref{lem:emp_prop_close}, there exists $T_{\epsilon'}$ such that for all $T\ge T_{\epsilon'}$ and $t\ge h(T)$, on $\mathcal E_T$ the empirical fractions of online sample is close to  $w^{*}(\mu, z^*(\mu))$, i.e.,

         \[  \left\| \lrset{\frac{N_a(t)}{t}}_a - w^{*}(\mu, z^{*}(\mu)) \right\| \le 4\epsilon'.  \]

        Using this with \eqref{eq:Gabound}, we get that for $T\ge T_{\epsilon'}$ and $t\ge h(T)$, on $\mathcal E_T$,

       \[  \max\limits_{a\ne 1} ~ \left| G_a\lrp{\hat{\mu}(t), z, \lrset{\frac{N_b(t)}{t}}_b} - G_a\lrp{\mu,z, w^{*}(\mu, z^{*}(\mu)} \right| \le \alpha(\epsilon'), \quad \forall z\in [0,1]. \]
       
        Recall that $\alpha(\epsilon') \rightarrow 0$ as $\epsilon' \rightarrow 0$.
Then, for $T \ge T_{\epsilon'}$, on the set $\mathcal E_T$, \begin{align*}
        \min ({\tau_\delta},T) &\le h(T) + \sum\limits_{t= h\lrp{T}+1}^{ T} {\bf 1}\lrp{ \tau_\delta \ge t}\\
        &\le  h(T) + \sum\limits_{t=h(T)+1}^{T} {\bf 1}\lrp{  \min \limits_{b \ne i^*(t)} Z_{i^*(t),b}({\bf N}(t)) \le \beta(t+\tau_1,\delta) }\\
        & = h(T) + \sum\limits_{t=h(T)+1}^{T} {\bf 1}\lrp{  (\tau_1 + t) \min\limits_{b\ne 1} G_b\lrp{ \hat{\mu}(t),\hat{z}_t, \lrset{\frac{N_a(t)}{t}}_a } \le \beta(t+\tau_1,\delta) }\\
        &\le h(T) + \sum\limits_{t=h(T)+1}^T {\bf 1}\lrp{  \min\limits_{a\ne 1} ~G_a\lrp{\mu, \hat{z}_t, w^*(\mu,z^*(\mu))} - \alpha(\epsilon') \le \frac{\hat{z}_t}{\tau_1} \beta(t+\tau_1,\delta)}\\
        & \le h(T)  + T_0(\delta) - h(T) + 1,
        \end{align*}
        where $T_0(\delta)$ is defined as 
        \begin{equation} \label{eq:T0delta}
            T_0(\delta) := \inf\lrset{t\in \mathbb{N}: ~  \min\limits_{a\ne 1} ~G_a\lrp{\mu, \hat{z}_t, w^{*}(\mu, z^{*}(\mu))} - \alpha(\epsilon') \ge \frac{\hat{z}_t}{\tau_1} \beta(t+\tau_1,\delta)}. 
        \end{equation}

        Thus for $T > T_{\epsilon'} + T_0(\delta) + 1$, on the good set $\mathcal{E}_T$ we have that $\tau_\delta \le T_0(\delta) + 1$. Hence, $\tau_\delta > T_{\epsilon'} + T_0(\delta) + 1 $ implies  the complement of good set. Using this,
        \begin{align*} 
        \mathbb{E}_{\mu}\lrs{\tau_\delta} &= \sum\limits_{T=0}^\infty \mathbb{P}_\mu\lrp{ \tau_\delta \ge T } \\
        &\le  T_{\epsilon'} + T_0(\delta) + 1 + \sum\limits_{T =  T_{\epsilon'}+ T_0(\delta) + 2 }^{\infty} \mathbb{P}\lrp{ \tau_\delta \ge T }\\
        &\le  T_{\epsilon'} + T_0(\delta) + 1 + \sum\limits_{T = 1 }^{\infty} \mathbb{P}\lrp{ \mathcal E^c_T }.
        \end{align*}
        Adding $\tau_\delta$ on both sides to get the total number of samples (including offline) and dividing by $\log\frac{1}{\delta}$, we get 
        \[
        \tau_1 +  \mathbb{E}_\mu\lrs{\tau_\delta} \le T_{\epsilon'} + 1 + \mathbb{E} [T_0(\delta) + \tau_1] + \sum \limits_{T=1}^\infty \mathbb{P}(\mathcal E^c_T) .    
        \]
        
        Under the conditions on $\tau_1$, substituting from Lemma \ref{lem:T0delta.exp.complexity} and using Lemma \ref{lem:compGoodSet} to bound the error probabilities ${\cal E}_T^c$, we have:
        
        \[
        \tau_1 +  \mathbb{E}_\mu\lrs{\tau_\delta} \le T_{\epsilon'} + T(\tilde{\epsilon}) + 1 + \mathbb{E} \left[ \frac{\tau_1}{z^*}\lrp{1 + \frac{2\alpha(\epsilon') + \tilde{\epsilon}}{\epsilon_1 C(\mu,p,\eta)} } \right] + o\lrp{\log \frac{1}{\delta}}.    
        \]
        This proves the result.
        \end{proof}

    \begin{lemma}\label{lem:T0delta.exp.complexity}
    Suppose $\mu$ is such that $ 1 \ge z^*(\mu) > \eta$, for some $\eta > 0$. Let $\epsilon'>0$ be a constant used in (\ref{eq:T0delta}). Let $\tilde{\epsilon}>0, \epsilon_1>0 $ be constants. Let $\alpha(\epsilon')$ be as in (\ref{eq:T0delta}). Let 
    \[c_{\epsilon_1} = \epsilon_1 L_{\mu} \lrp{\frac{1}{\eta}-1},\quad c_{\epsilon_2}= \frac{2 \alpha(\epsilon') + \tilde{\epsilon}}{\eta},\]
    where $L_{\mu} = \max (\KL(\mu_1,\mu_a), \KL (\mu_a,\mu_1))$ and 
    $\epsilon_2=(\epsilon',\tilde{\epsilon})$. Then, there exists functions $C_1^a(\cdot)$ and $C_2^a(\cdot)$ such that $C^a_2(c) \rightarrow C^a_2(0)$ and $C^a_1(c) \rightarrow C^a_1(0)$ as $c\rightarrow 0$. Moreover,  $C^a_1(0) = C^a_2(0)$. Furthermore, suppose 
    \[\frac{\tau_1}{\log \frac{1}{\delta} + \log \log \frac{1}{\delta}} \notin (C_1^a(c_{\epsilon_1}),C_2^a(c_{\epsilon_2})), \]  then for bandit instance $\mu$, $T \ge T_{\epsilon'} + T_0(\delta) + 1$, on the set  $\mathcal E_T$, there exists $T(\tilde{\epsilon}) > 0$ such that 
$$T_0(\delta) + \tau_1 \le T(\tilde{\epsilon}) + \frac{\tau_1}{z^*}\lrp{1 + \frac{2\alpha(\epsilon') + \tilde{\epsilon}}{\epsilon_1 C(\mu,p,\eta)} } , \quad a.s.$$
    \end{lemma}

    \begin{proof}
    Consider $\delta > 0$. Recall that $\hat{z}_t = \frac{\tau_1}{t+\tau_1}$ denotes the observed fraction of the offline data at time $t$. For non-negative constants $c$ and $d$ recall that  
    $$\beta(T+\tau_1,\delta) = \log\frac{1}{\delta} + \log\log\frac{1}{\delta} + c\log\log(T+\tau_1)+Kd,$$ and $T_0(\delta)$ (defined in (\ref{eq:T0delta})) is the smallest $t\in\mathbb N$ such that 
        \begin{equation*} 
             - \alpha(\epsilon') \ge \frac{\hat{z}_t}{\tau_1}\lrp{\log\frac{1}{\delta} + \log\log\frac{1}{\delta} + c\log\log(T+\tau_1)+Kd} ~ - ~ \min\limits_{a\ne 1} ~G_a\lrp{\mu, \hat{z}_t, {w}^*({z}^*(\mu), \mu)}.
        \end{equation*}
    Let $T(\tilde{\epsilon})$ be the smallest time $t$ such that $c\log\log(t+\tau_1) + K d \le \tilde{\epsilon} (t+\tau_1) $. Then, $T_0(\delta)$ is at most $T(\tilde{\epsilon})$ plus the smallest time $t$ such that 
        \begin{equation}\label{eq:cond1T0delta}
            - \alpha(\epsilon') - \tilde{\epsilon} \ge \frac{\hat{z}_t}{\tau_1} \lrp{\log\frac{1}{\delta} + \log\log\frac{1}{\delta}}  ~ - ~ 
\min\limits_{a\ne 1} ~G_a\lrp{\mu, \hat{z}_t, {w}^*({z}^*(\mu), \mu)}.
        \end{equation}
    Call this time $T_1(\delta)$. Moreover, recall that $z^{*}$ is the maximum $z$ satisfying
    \begin{equation}\label{eq:defz}
    V(\mu,z) \ge \frac{z}{\tau_1}\lrp{\log \frac{1}{\delta} + \log\log\frac{1}{\delta}}.  \end{equation}
    In \eqref{eq:cond1T0delta}, r.h.s. is a continuous function of $\hat{z}_t$. This follows from the joint-continuity of $G_a$ in its arguments (Lemma \ref{lem:contGj}). 
    
    Furthermore, r.h.s. stays negative for all $z \in [0,z^{*}]$. This is because $G_a(\mu, z, w)$ is a concave function of $z$ which is also non-negative. The first term on the r.h.s. in (\ref{eq:cond1T0delta}) is $0$ at $z=0$ and does not intersect the second term till $z^*$. Thus, the second term is greater than the first term for  $z \in [0,z^*]$.
    
    Now, to argue that $T_1(\delta)$ is not too large compared to $t^*$, we will show that at time $T_1(\delta)$, the fraction of offline data is not too small compared to the optimal fraction $z^*$. To this end, finding the smallest $t$ such that (\ref{eq:cond1T0delta}) holds is same as finding the largest $z$ such that the following holds:
    
    \begin{equation}\label{eq:cond1zdelta}
        - \alpha(\epsilon') - \tilde{\epsilon} \ge \frac{z}{\tau_1} \lrp{\log\frac{1}{\delta} + \log\log\frac{1}{\delta}}  ~ - ~
\min\limits_{a\ne 1} ~G_a\lrp{\mu, z, {w}^*({z}^*(\mu), \mu)}.
    \end{equation}
    
    Since $\mu$ has a unique optimal arm, the set of optimizers $w^*$ is unique. This follows from Theorem \ref{th:optsol.char}  and Lemma $\ref{lem:equivalence}$.

    With this and a few rearrangements, the above constraint re-writes as 
    \begin{equation*}
        -\frac{ \alpha(\epsilon') + \tilde{\epsilon}}{z} \ge \frac{1}{\tau_1} \lrp{\log\frac{1}{\delta} + \log\log\frac{1}{\delta}}  ~ - ~  \min\limits_{a\ne 1} ~\frac{G_a\lrp{\mu, z, w^*}}{z}.
    \end{equation*}
    
    {We show that largest  $z$ satisfying the above is at least $ z^*(1-k(\epsilon', \tilde{\epsilon}))$, for some function $k(\epsilon',\tilde{\epsilon})$ such that $k(\epsilon', \tilde{\epsilon}) \rightarrow 0$ as $(\epsilon', \tilde{\epsilon}) \rightarrow (0,0)$. } Let us restrict to $z\ge \frac{z^*}{2}$. L.h.s. above then is at least 
    \[ -\frac{2(\alpha(\epsilon') + \tilde{\epsilon})}{z^*}. \]
    This tightens the constraint giving a lower bound on the required $z$. Constraint now becomes 
    \[-\frac{2(\alpha(\epsilon') + \tilde{\epsilon})}{z^*} \ge \frac{1}{\tau_1} \lrp{\log\frac{1}{\delta} + \log\log\frac{1}{\delta}}  ~ - ~  \min\limits_{a\ne 1} ~\frac{G_a\lrp{\mu, z, w^*}}{z}.\]
    The above re-writes as 
    \begin{align}\label{eq:overshoot}
    \frac{2(\alpha(\epsilon') + \tilde{\epsilon})}{z^*} \le -\frac{1}{\tau_1} \lrp{\log\frac{1}{\delta} + \log\log\frac{1}{\delta}}  ~ + ~  \min\limits_{a\ne 1} ~\frac{G_a\lrp{\mu, z, w^*}}{z}.
    \end{align}
    Adding and subtracting the second term in the r.h.s., evaluated at $z^*$, tightening the constraint by observing that 
    \[  \min\limits_{a\ne 1} ~\frac{G_a\lrp{\mu, z^*, w^*}}{z^*} -\frac{1}{\tau_1} \lrp{\log\frac{1}{\delta} + \log\log\frac{1}{\delta}} \ge 0, \]
    the constraint becomes 
    \[\frac{2(\alpha(\epsilon') + \tilde{\epsilon})}{z^*} \le  \min\limits_{a\ne 1} \lrp{\frac{G_a\lrp{\mu, z, w^*}}{z} -  ~\frac{G_a\lrp{\mu, z^*, w^*}}{z^*}}.\]    

    Now, recall that 
    \[ G_a(\mu,z,w^*) = \inf\limits_{x}  \lrset{ (zp_1 + (1-z)w^*_1)\KL(\mu_1,x) + (zp_a + (1-z)w^*_a)\KL(\mu_a, x) }. \]
    Let us denote the infimizer above by $x_{1,a}$. By computation, it is not hard to see that 
     \begin{align}
    \frac{\partial {G_a(\mu,z, w^*)}/{z}}{\partial {1}/{z}} & = {w^*_1}\KL(\mu_1, x_{1,a}) + {w_a^*}\KL(\mu_a, x_{1,a})\label{eq:derivinvz}.
    \end{align}
    Suppose for some $\epsilon_1 > 0$, we get a lower bound on the derivative above, say  $\epsilon_1{C(\mu,p,\eta)} > 0$ (to be proven later), that is independent of $\delta$, then
\[ \frac{G_a\lrp{\mu, z, w^*}}{z} -  ~\frac{G_a\lrp{\mu, z^*, w^*}}{z^*} \ge \lrp{\frac{1}{z}-\frac{1}{z^*}}\epsilon_1 C(\mu,p,\eta).\]
    Thus, the required $z$ is at least the largest $z$ satisfying 
    \[ \frac{2(\alpha(\epsilon') + \tilde{\epsilon})}{z^*} \le \lrp{\frac{1}{z}-\frac{1}{z^*}}\epsilon_1 C(\mu,p,\eta), \]
    which equals
    \[ \frac{z^*}{1+ \frac{2(\alpha(\epsilon')+ \tilde{\epsilon
    })}{\epsilon_1 C(\mu,p,\eta)}} =: z_\delta. \]
    This is the required lower bound on largest $\hat{z}_t$ satisfying (\ref{eq:cond1zdelta}), also giving an upper bound on $T_1(\delta)$ as below:
    \[ T_1(\delta) \le \tau_1\lrp{\frac{1}{z_\delta} - 1}. \]
    
    Thus, $T_0(\delta) \le T(\tilde{\epsilon}) + T_1(\delta) \leq T(\tilde{\epsilon}) + t_1,$ where
    \[ t_1 = \frac{\tau_1}{z^*}\lrp{1 + \frac{2\alpha(\epsilon') + \tilde{\epsilon}}{\epsilon_1 C(\mu,p,\eta)} }  - \tau_1. \]
    giving the desired bound on $T_0(\delta)$.
    
    We now prove the existence of the bound $C(\mu,p,\eta)$. Recall that $w^*$ in the above discussion refers to the optimal online proportions for $z^*$. 
    
    \textbf{Case 1 - $\max(w^*_1,w^*_a) > \epsilon_1 $}
    
    To this end, consider an arm $a\ne 1$, such that $\max\lrset{w^*_1, w^*_a} \ge \epsilon_1$. Lemma \ref{lem:largeweights} then gives the required bound.
    
    \textbf{Case 2 - $\max(w^*_1,w^*_a) < \epsilon_1 $:}
    Suppose $w^*$ is such that $w^*_1 < \epsilon_1$ and there exists an arm $a\ne 1$ such that $w^*_a < \epsilon_1$. Then, along the sequence $\delta_n$, from Lemma \ref{lem:smallweights}, we have that arm $a$ already satisfies the overshoot condition in (\ref{eq:overshoot}) before $z^*$ is reached. Since the derivative in~\eqref{eq:derivinvz} is positive, arm $a$ continues to satisfy the overshoot condition in~\eqref{eq:overshoot}. Thus, arm $a$ will not be a candidate for minimum of $G_a$ in the constraint (\ref{eq:overshoot}). It thus suffices to get a bound on the derivative of $G_a(z)/z$ for which $w^*_1 + w^*_a \ge \epsilon_1$.
    \end{proof}           

 \subsection{Analyzing the case when $\max(w^{*}_1, w^{*}_a) > \epsilon_1$} \label{sec:largeweights}

    \begin{lemma}\label{lem:largeweights}
    Suppose $\mu$ is such that $ 1 \ge z^*(\mu) > \eta$, for some $\eta > 0$. If for $\epsilon > 0 $,  $w^*:=w^*(\mu,z^*(\mu))$ is such that for an arm $a\ne 1$, $\max\lrset{w^*_1, w^*_a} \ge \epsilon$, then there exists $C(\mu,p,\eta)$ such that  
    \[ \frac{\partial {G_a(\mu, z, w^*)}/{z}}{\partial {1}/{z}} \ge \epsilon C(\mu,p,\eta) > 0, \quad \text{for all } z\ge \frac{z^*(\mu)}{2}. \]
    \end{lemma}
    \begin{proof}
        Recall that 
        \[ G_a(\mu,z,w^*) = \inf\limits_{x}  \lrset{ (zp_1 + (1-z)w^*_1)\KL(\mu_1,x) + (zp_a + (1-z)w^*_a)\KL(\mu_a, x) }.\]
        Let the infimum $x$ above be denoted by $x_{1,a}$ and recall that
        \begin{align*} x_{1,a} &= \frac{(zp_1 + (1-z)w^*_1)\mu_1 + (zp_a + (1-z)w^*_a)\mu_a}{z(p_1 + p_a) + (1-z)(w^*_1+w^*_a)}\\
        &= \frac{ (p_1 - w^*_1 + \frac{w^*_1}{z})\mu_1 + (p_a - w^*_a + \frac{w^*_a}{z})\mu_a }{p_1 + p_a - w^*_1 - w^*_a + \frac{w^*_1 + w^*_a}{z}}, \end{align*}
        and
        \[\frac{\partial {G_a}/{z}}{\partial {1}/{z}} = {w^*_1}\KL(\mu_1, x_{1,a}) + {w^*_a}\KL(\mu_a, x_{1,a}).\]
        Now, for an arbitrary $\epsilon > 0$, if $w^*_1 \ge \epsilon$, the infimizer above increases if we replace $w^*_1$ by $1$ and $w^*_a$ by $0$ and we get 
        \[x_{1,a} \le \frac{(p_1 - 1)\mu_1 + p_a\mu_a + \frac{\mu_1}{z} }{p_1 + p_a -1  + \frac{1}{z}}. \]
        R.h.s. in increasing in $1/z$. Since we restrict to $z$ such that $\frac{1}{z} \le \frac{2}{z^*} \le \frac{2}{\eta}$, we have
        \[ x_{1,a} \le \frac{(p_1 - 1)\mu_1 + p_a\mu_a + \frac{2\mu_1}{\eta} }{p_1 + p_a -1  + \frac{2}{\eta}} =: x_{a}, \]
        which is a constant between $\mu_a$ and $\mu_1$. Using this,
        \[ \frac{\partial {G_a}/{z}}{\partial {1}/{z}} \ge w^*_1 \KL(\mu_1, x_{1,a}) \ge w^*_1 \KL(\mu_1, x_a) \ge \epsilon \KL(\mu_1, x_{a}). \]
        
        In the other case, i.e., when $w^*_1 \le \epsilon$, we get a similar bound by choosing a lower bound on $x_{1,a}$ that is closer to $\mu_a$ instead. We obtain this by replacing $w^*_1$ by  $0$ and $w^*_a$ by  $1$ as below: 
        \[ x_{1,a} \ge \frac{p_1\mu_1 + (p_a-1)\mu_a + \frac{\mu_a}{z}}{p_1 + p_a - 1 + \frac{1}{z}}. \]
        Further lower bounding the above by setting $\frac{1}{z} = \frac{2}{\eta}$ (since the bound above is decreasing in $1/z$), we get 
        \[ x_{1,a} \ge \frac{p_1\mu_1 + (p_a-1)\mu_a + \frac{2\mu_a}{\eta}}{p_1 + p_a - 1 + \frac{2}{\eta}} =: \tilde{x}_a. \]
        Thus in this setting, 
        \[ \frac{\partial \frac{G_a}{z}}{\partial \frac{1}{z}} \ge \epsilon \KL(\mu_a, \tilde{x}_a). \]
        With these,
        \[ C(\mu,p,\eta):= \max\limits_{a\ne 1} \max\lrset{\KL(\mu_1, x_a), \KL(\mu_a, \tilde{x}_a)}. \]
    \end{proof}

    \subsection{Analyzing the case when $w^{*}_1 + w^{*}_a < \epsilon_1$} \label{sec:smallweights}
    Recall that
    \[G_a(\mu,1,w^{*})  = p_1 d(\mu_1, \mu_p) + p_a d(\mu_a,\mu_p),\]
    where 
    \[\mu_p  =  \frac{p_1 \mu_1 + p_a \mu_a}{p_1 + p_a}.\]
    Let the  threshold 
    \[P := \frac{\log \frac{1}{\delta} + \log \log \frac{1}{\delta}}{\tau_1}.\]
    Further, let $C_1^a(c)$ be a function of a scalar $c$ (and also of problem parameters $p,\mu$) such that 
    \[\tau_1 < C_1^a(c) \left(\log \frac{1}{\delta} + \log \log \frac{1}{\delta}\right) \quad \implies  \quad P > G_a(\mu,1,w^{*})+ c, \]
    where $c>0$ will be set later. Similarly, let $C_2^a(c)$ be a function of scalar $c>0$ (and also only of problem parameters $p,\mu$) 
    such that 
    \[\tau_1 > C_2^a(c) \left( \log \frac{1}{\delta} + \log \log \frac{1}{\delta} \right) \implies P < G_a(\mu,1,w^{*}) - c.\]
    We will call $G_a(\mu,z,w^{*})$ by $G_a(z)$ for simplifying the notation in the following lemma.
 
    \textbf{Observation:}  Observe that $C_1^a(c) \rightarrow C_1^a(0)$ and $C_2^a(c) \rightarrow C_2^a(0)$ as $c \rightarrow 0$ and in fact, 
    \[C_1^a(0)=C_2^a(0)= \frac{1}{G_a(\mu,z=1,w^{*})}.\]

     \begin{lemma}\label{lem:smallweights}
       Suppose $z^{*} > \eta > 0$ (for a fixed $\eta$). If $w^{*}_1 + w^{*}_a < \epsilon_1$ and $\frac{G_a(z)}{z}$ does not satisfy the overshoot condition in (\ref{eq:overshoot}) for $z=1$, we have the following implications: 
       \begin{itemize}
       \item $\tau_1 \geq C_1^a(c_{\epsilon_1}) \left(\log \frac{1}{\delta} + \log \log (\frac{1}{\delta})\right)$,
       \item $ \tau_1 \leq C_2^a(c_{\epsilon_2}) \left(\log \frac{1}{\delta} + \log \log (\frac{1}{\delta})\right)$, 
       \end{itemize}
       where 
       \[c_{\epsilon_1} = \epsilon_1 L_{\mu} \lrp{\frac{1}{\eta}-1}, \quad  c_{\epsilon_2}= \frac{2 \alpha(\epsilon') + \tilde{\epsilon}}{\eta}, \quad \text{ and } \epsilon_2=(\epsilon',\tilde{\epsilon}).\]
       
     \end{lemma}   
     \begin{proof}
    
    Let $L_{\mu} = \max (\KL(\mu_1,\mu_a), \KL (\mu_a,\mu_1))$. If $w_1 + w_a < \epsilon$, from (\ref{eq:derivinvz}), we have the following:
    \begin{align}\label{eq:derivinvzub}
        \frac{\partial {G_a}/{z}}{\partial {1}/{z}} \overset{(a)}{\leq} \epsilon_1 L_{\mu}
    \end{align}
    The inequality (a) above follows since $x_{1,a} \in [\mu_a,\mu_1]$, hence $\KL(\mu_1,x_{1,a}) \leq \KL(\mu_1,\mu_a)$. Similarly, it holds for the other term $\KL(\mu_a,x_{1,a})$. Next, we know that $z^{*}$ has the following property: \[\frac{G_a(\mu,z^{*},w^{*})}{z^{*}} \geq P.\] 
    We will call $G_a(\mu,z,w^{*})$ by $G_a(z)$ for simplifying the notation. 
    
    \textbf{Case 1:} $\tau_1 < C_1^a(c_{\epsilon_1}) \left(\log \frac{1}{\delta} + \log \log (\frac{1}{\delta})\right)$, where $c_{\epsilon_1}= \epsilon_1 L_{\mu} \lrp{\frac{1}{\eta}-1} $.
    
    From (\ref{eq:derivinvz}), $\frac{G_a(z)}{z}$ increases as $z$ goes from $1$ to  $z^{*}$ since gradient with respect to $1/z$ is positive. The gradient is also upper bounded by $\epsilon_1 L_{\mu}$. Therefore, if $\frac{G_a(z^{*})}{z^{*}} > P$ then, integrating the upper bound on the derivative from $1$ to $\frac{1}{z^{*}}$, we have
    \begin{align} \label{eq:Pub}
    P &\leq G_a(1) + \epsilon_1 L_{\mu} \lrp{\frac{1}{z^{*}}-1 } \nonumber \\
     &\leq  G_a(1) + \epsilon_1 L_{\mu} \lrp{\frac{1}{\eta}-1 }
    \end{align}
    If $c_{\epsilon_1}=\epsilon_1 L_{\mu}\lrp{\frac{1}{\eta}-1}$, then the definition of $C_1^a(c_{\epsilon})$ yields a contradiction that $P > G_a(1) + \epsilon_1 L_{\mu}\lrp{\frac{1}{\eta}-1}$.     This implies, that $\tau_1 \geq C_1^a(c_{\epsilon_1}) \left(\log \frac{1}{\delta} + \log \log (\frac{1}{\delta})\right)$. 

    \textbf{Case 2:}  $\tau_1 > C_2^a(c_{\epsilon_2}) \left(\log \frac{1}{\delta} + \log \log (\frac{1}{\delta})\right)$, where $c_{\epsilon_2}= \frac{2 \alpha(\epsilon') + \tilde{\epsilon}}{\eta}$.   
  
    The above assumption on $\tau_1$ means that 
    \[P< G_a(1) - \frac{2 \alpha(\epsilon') + \tilde{\epsilon} }{\eta} .\] 
    This means that the earliest $z$ at which $\frac{G_a(z)}{z} -P > \frac{2 \alpha(\epsilon') +  \tilde{\epsilon} }{z^{*}}$ is $z=1$ (recall that $\frac{1}{z^{*}} < \frac{1}{\eta}$). Therefore for $G_a$ the smallest time that satisfies the overshoot condition for $G_a$ is $z=1 > z^{*}$. Therefore, if $\frac{G_a(z)}{z}$ does not satisfy the overshoot condition (\ref{eq:overshoot}) in the beginning ($z=1$), then $\tau_1 \leq C_2^a(c_{\epsilon_2}) \left(\log \frac{1}{\delta} + \log \log (\frac{1}{\delta})\right)$.
\end{proof}

    \begin{lemma}\label{lem:compGoodSet}
        \[\limsup\limits_{\delta\rightarrow 0} \frac{\sum\limits_{T=1}^\infty \mathbb{P}\lrp{\mathcal E^c_T} }{\log\frac{1}{\delta} } = 0.\]
    \end{lemma}
    \begin{proof}
        Recall that $\epsilon' > 0$, for $T\ge 0$, $h(T) = \sqrt{ T}$ and
        \[ \mathcal E_T(\epsilon'):= \bigcap\limits_{t= h(T)}^T \lrp{\hat{\mu}(t) \in \mathcal I_{\epsilon'}}, \]
        where for $\zeta(\epsilon') > 0$,
        \[ \mathcal I_{\epsilon'} := [\mu_1 - \zeta(\epsilon'), \mu_1 + \zeta(\epsilon')] \times \dots \times [\mu_K - \zeta(\epsilon'), \mu_K + \zeta(\epsilon')] .  \]
        Clearly, 
        \[ \mathbb{P}\lrp{\mathcal E^c_T} \le \sum\limits_{t=h(T)}^T \sum\limits_{a=1}^K \mathbb{P}\lrp{ \hat{\mu}_a(t) \not\in [\mu_a - \zeta(\epsilon'), \mu_a + \zeta(\epsilon')] }. \]Since each arm is pulled at least $\sqrt{t/K}-1$ times till time $t$, using Chernoff bound
        
        \begin{align*} 
        &\mathbb{P}\lrp{\hat{\mu}_a(t)\not\in[\mu_a - \zeta(\epsilon'), \mu_a + \zeta(\epsilon')]}\\ 
        &\le \sum\limits_{s=\sqrt{\frac{t}{K}}-1}^T e^{-s d(\mu_a - \zeta(\epsilon'), \mu_a) } + e^{-s d(\mu_a + \zeta(\epsilon'), \mu_a)}\\
        &\le \frac{ e^{- \lrp{\sqrt{\frac{t}{K}} -1 } d(\mu_a - \zeta(\epsilon')) } }{1-e^{- d(\mu_a - \zeta(\epsilon')) }} + \frac{e^{- \lrp{\sqrt{\frac{t}{K}} -1 } d(\mu_a + \zeta(\epsilon')) }}{1-e^{-  d(\mu_a + \zeta(\epsilon')) }}.
        \end{align*}
        
        Let $D := \min_a\lrset{ d(\mu_a-\zeta(\epsilon'), \mu_a), d(\mu_a + \zeta(\epsilon'),\mu_a) }$ and define  
        \[  B:= \sum\limits_{a=1}^K \lrp{ \frac{ e^{d(\mu_a - \zeta(\epsilon')) } }{1-e^{- d(\mu_a - \zeta(\epsilon')) }} + \frac{e^{d(\mu_a + \zeta(\epsilon')) }}{1-e^{-  d(\mu_a + \zeta(\epsilon')) }} }.\]
        Then, 
        \[ \mathbb{P}(\mathcal E^c_T) \le \sum\limits_{t=h(T)}^T Be^{- D \sqrt{\frac{t}{K}} } \le BTe^{-D \sqrt{h(T)}}. \]
        On dividing by $\log(1/\delta)$ and taking limits, we get the desired result.
    \end{proof}

\section{Properties of tracking rule} \label{sec:tracking}
In the Algorithm \ref{alg:tas_batch}, for the first $K$ time slots in the online phase, each arm is pulled once and then the arm with the maximum value of  $\frac{w(t)}{N_a(t-1)}$ is chosen, where $w(t)$ is the sequence of weights that the algorithm tracks (see Algorithm~\ref{alg:tas_batch} for the exact expression). 
For this tracking rule, \cite{wang2021fast} show that:
    \begin{align} \label{eq:proportion_closeness}
     N_a(t) \in \left[ t w_a(t) - K-1, t w_a(t) + 1 \right].
    \end{align}

\textbf{Recalling continuity properties:}
\begin{enumerate}
\item  Let $w^{*}(\mu, z^*(\mu))$ denote the  optimal allocations in $\Sigma_K$ for the optimal $z^{*}(\mu)$ for $\mu$. Recall that the map $\mu \rightarrow {w}^{*}(\mu, z^*(\mu))$ is continuous (Lemma \ref{lem:cont.uniq}). We suppress the dependence of $w^*$ on $z^*$, whenever it is clear from the context.

\item For all $a$, $G_a(\mu,z,w)$ is jointly continuous in all the arguments (Lemma~\ref{lem:contGj}).  

\end{enumerate} 

Let $\hat{\mu}(t) = (\hat{\mu}_1(t), \dots, \hat{\mu}_K(t))$ denote the vector of empirical means of the arms at the online time $t$.

\textbf{Good Event - ${\cal E}_{\epsilon}(t)$:} Define the good event ${\cal E}_{\epsilon}(t)$  for time $t$ as: $ \max_k  | \hat{\mu}_k(t) - \mu_k | < \psi(\epsilon)$, where $\psi(\epsilon)>0$ is chosen such that:
   \begin{align}\label{eq:good_bounds}
       \lVert w^{*}(\hat{\mu}(t)) - w^*(\mu) \rVert_2 & < \epsilon, \nonumber \\
       | G_a(\hat{\mu}(t), z, w) - G_a(\mu, z, w) | & < \epsilon, ~ \forall w \in \Sigma_K,~a \in [K].
   \end{align}
Equation (\ref{eq:good_bounds}) is ensured by some $\psi(\epsilon)$ due to continuity properties recalled above.

Next, let $h_1(T) = T^{\gamma_1},~h_2(T) = T^{\gamma_2}$ where $0<\gamma_1 < \gamma_2 < 1$ and let 
\[{\cal E}_{\epsilon}(T) = \bigcap \limits_{t=h_1(T)}^{T} {\cal E}_{\epsilon}(t).\]
Let 
\[T > T_\epsilon:= \inf\lrset{ t: \epsilon > \max\lrset{4 t^{-(\gamma_2-\gamma_1)}, 4 K t^{-\gamma_2},  2 \sqrt{K}T^{-\gamma_2/2 }  } }.\]

\begin{lemma} \label{lem:emp_prop_close}
Let $h_2(T) < t < T$, $T > T_{\epsilon}$. Let the event ${\cal E}_{\epsilon}(T)$ hold. Then, the empirical proportions $\frac{N(t)}{t}$ of Algorithm \ref{alg:tas_batch} satisfy
   \begin{align*}
       \left\| \frac{N(t)}{t} - w^*(\mu) \right\|_{\infty} < 4 \epsilon.
   \end{align*}
\end{lemma}
\begin{proof}
Note that Algorithm \ref{alg:tas_batch} updates only for $t= q^2 K,~ q \in \mathbb{N}$. Let $q_1$ be the smallest natural number such that $q_1^2K > h_1(T)$. Fix a $h_2(T) < t < T$. Let $q_2$ be the largest natural number such that $q_2^2K < t < T $.

In Algorithm \ref{alg:tas_batch}, let $w(q_1^2 K)=\mathbf{v}$. Let $\hat{w}(q^2 K) = \mathbf{v}_q,~ q \in \left[q_1,q_2\right]$. Let $\mathbf{r}= \left[\frac{1}{K} \frac{1}{K} \ldots \frac{1}{K}\right]$. By the updates in Algorithm \ref{alg:tas_batch}, we have:
\begin{align} \label{eq:w_tilde}
w(t) = \frac{1}{t}\lrp{ q_1^2K \mathbf{v} + (t- q_1^2 K  - (q_2-q_1)K) \tilde{\mathbf{v}} + (q_2-q_1)K \mathbf{r}}.
\end{align}
Here, $\tilde{\mathbf{v}} \in \mathrm{conv}(\mathbf{v}_{q_1} , \mathbf{v}_{q_1+1} \ldots \mathbf{v}_{q_2})$. Let 
\[\eta_{q} \in \mathbb{R}: \sum_{q=q_1}^{q_2} \eta_{q}=1,~ \sum \limits_{q=q_1}^{q_2} \eta_q \mathbf{v}_q = \tilde{\mathbf{v}} ,~\eta_q \geq 0.\]

Let $w^* := w^*(\mu)$. Then, 
\begin{align}\label{eq:iteratew}
  &\lVert w(t) - {w}^* \rVert_{\infty} \nonumber\\
  &\leq \frac{q_1^2 K}{ t} \lVert \mathbf{v} - {w}^* \rVert_{\infty} + \frac{1}{t} (t - q_1^2 K - (q_2-q_1) K)  \lVert \mathbf{\tilde{v}} - {w}^* \rVert_{\infty} + (q_2-q_1)\frac{ K}{t} \lVert \mathbf{r} - {w}^* \rVert_{\infty}     \nonumber \\
  & \leq \frac{K}{t} \left( 1+ \sqrt{\frac{h_1(T)}{K}} \right)^2  + \lVert \mathbf{\tilde{v}} - {w}^* \rVert_{\infty} + \frac{\sqrt{K}}{t} (\sqrt{t}- \sqrt{h_1(T)}) \nonumber \\
  & \leq  \frac{4}{h_2(T)} \max(h_1(T),K) + \frac{2\sqrt{K} }{\sqrt{h_2(T)}} + \lVert \mathbf{\tilde{v}} - {w}^* \rVert_{\infty} \nonumber \\
  & \leq \max ( 4 T^{-(\gamma_2 - \gamma_1)}, 4K T^{-\gamma_2} ) + 2 \sqrt{K}T^{-\gamma_2/2} + \lVert \mathbf{\tilde{v}} - {w}^* \rVert_{\infty} \nonumber \\
  & \overset{(a)}{\leq} 2 \epsilon + \lVert \mathbf{\tilde{v}} - {w}^* \rVert_{\infty}.
\end{align}

From continuity of $w^*(\cdot)$, it follows that on the good set, for $\forall q\ge q_1 $, $\lVert\mathbf{v}_q - w^*(\mu)\rVert_\infty \le \epsilon.$ Thus, 
\begin{align} \label{eq:closeness}
 \lVert \mathbf{\tilde{v}} - {{w}^*} \rVert \leq \sum \limits_{q=q_1}^{q_2} \eta_q  \lVert \mathbf{v}_q - {w}^* \rVert_{\infty}  \leq \epsilon.   
\end{align}

Combining (\ref{eq:iteratew}), (\ref{eq:closeness}), we have for $h_2(T) < t < T$, $T > T_{\epsilon}$ under event ${\cal E}_{\epsilon}(T)$,
\begin{align}\label{eq:close_convex_hull}
  \lVert w(t) - {w}^* \rVert_{\infty} < 3 \epsilon.
 \end{align} 

Next, due to the property of the tracking rule (Equation~\eqref{eq:proportion_closeness}), for $t > h_2(T)$ we have: 

 \begin{align}\label{eq:tracking_closeness}
     \left| \frac{N_a(t)}{t} - w_a(t) \right| < \frac{K}{t} < K T^{-\gamma_2} < \epsilon.
 \end{align}

For $h_2(T) < t < T$, $T > T_{\epsilon}$, under the event ${\cal E}_{\epsilon}(T)$, combining (\ref{eq:close_convex_hull}) and (\ref{eq:tracking_closeness}), we have :
\begin{align*}
      \left\| \frac{N(t)}{t} - w^*(\mu) \right\|_{\infty} < 4 \epsilon,
\end{align*}
proving the result.
\end{proof}

\section{Supporting results and proofs from Section~\ref{sec:comp_complexity}}

\begin{lemma}\label{lem:convex.alloc} For fixed $N_1$, $Z_{1,a}({\bf N})$ is a monotonic function of $N_a$. Moreover, $N_a(N_1)$ is convex.
\end{lemma}

\begin{proof}
For a fixed $N_1$, ${\bf N}= \lrset{N_1, \dots, N_K}$, and $a\ne 1$
\[ Z_{1,a}({\bf N}) = (N^o_1 + N_1)\KL(\hat{\mu}_1(t), x_{1,a}) + (N^o_a + N_a)\KL(\hat{\mu}_a(t), x_{1,a}).\]
Here, $x_{1,a}$ is the infimizer in the definition of  $Z_{1,a}({\bf N})$, and depends on $\bf N$. 

Differentiating and using optimality of $x_{1,a}$, we get that 
\[ \frac{\partial Z_{1,a}({\bf N})}{\partial N_a} = \KL(\hat{\mu}_a(t), x_{1,a}) \ge 0,  \]
proving the monotonicity of $Z_{1,a}({\bf N})$ in $N_a$ for fixed $N_1$. 

Recall that for a fixed $N_1$, $N_a(N_1)$ is the $N_a$ that solves the following equation:
\begin{equation}\label{eq:NaN1} (N^o_1 + N_1)\KL(\hat{\mu}_1(t), x_{1,a}) + (N^o_a + N_a)\KL(\hat{\mu}_a(t), x_{1,a}) = \log\frac{1}{\delta} + \log\log\frac{1}{\delta},\end{equation}
where $x_{1,a}$ is the infimizer in the definition of $Z_{1,a}({\bf N})$ and is given by
\[ x_{1,a} = \frac{ (N^o_1 + N_1)\hat{\mu}_1(t) + (N^o_a + N_a)\hat{\mu}_a(t) }{N^o_1 + N_1 + N^o_a + N_a}.\]
Differentiating (\ref{eq:NaN1}) with respect to $N_1$ and using optimality of $x_{1,a}$ gives
\[ \KL(\hat{\mu}_1(t), x_{1,a}) + \KL(\hat{\mu}_a(t), x_{1,a})\frac{\partial N_a}{\partial N_1} = 0. \]
This gives that 
\[\frac{\partial N_a}{\partial N_1} = - \frac{\KL(\hat{\mu}_1(t), x_{1,a})}{\KL(\hat{\mu}_a(t), x_{1,a})}.\]
From the above expression, observe that on increasing $N_1$, $N_a(N_1)$ decreases and $x_{1,a}$ increases and gets closer to $\hat{\mu}_1(t)$. This observation with the monotonicity of $\KL$ in the second argument with a fixed first argument implies that $\KL(\hat{\mu}_1(t),x_{1,a})$ reduces and $\KL(\hat{\mu}_a(t), x_{1,a})$ increases. This argues  that the derivative of $N_a$ with respect to $N_1$ increases on increasing $N_1$, implying convexity of $N_a(N_1)$. 
\end{proof}

\begin{lemma}\label{lem:optall}
$N^*_1=n^*_1$ and $N^*_a = \max\lrset{0, N_a(n^*_1)}$ are the optimal allocations for $\mathbf{P2}$. 
\end{lemma}
\begin{proof}
The proof of this Lemma follows from the equivalence between problems $\mathcal{O}_2$ and ${\bf P2}$.
\end{proof}

\begin{lemma}\label{lem:convx}
 For a convex function $f:\Re \rightarrow \Re$, the set of points with $0$ in their set of sub-gradients is an interval. At any point to the left of this interval the sub-gradients are all strictly negative. Similarly at a point to the right of this interval, the sub-gradients are strictly positive. 
\end{lemma}
\begin{proof}
The proof of the above lemma follows from the observation that the set of minimizers of a convex function is a convex set. Moreover, the set of points with $0$ in the sub-gradient is precisely the set of minimizers. Hence, this set is convex (interval in the current setup). Sign of the sub-gradients at points to the left or right of this interval follows from convexity of the function and definition of sub-gradients.
\end{proof}
\subsection{Bisection search for solving ${\bf P2}$}
In this section, we present our algorithm for solving optimization problem ${\bf P2}$ described in Equation~\eqref{prob:P2}. Algorithm~\ref{alg:bisection} describes this procedure.z
\begin{algorithm2e}
   \caption{\texttt{BisectionOracle}}
   \label{alg:bisection}
   	\DontPrintSemicolon 
   	\tcc{For simplicity of presentation, we assume arm $1$ is the empirical best arm}
	\KwIn{${\bf N^o}$, $\hat{\mu}(t), \delta$}
    \SetKwProg{Fn}{}{:}{}
    { 
        \textbf{Initialization}{ $N^l_1 = 0$, $N^u_1 = 1, \text{tolerance} = 10^{-3}$}
        
        \tcc{identify appropriate range of $N_1$ for bisection search}
        \While{ $\min\limits_{b\ne 1} (N^o_1+N^u_1) \KL(\hat{\mu}_1(t), \hat{\mu}_b(t)) < \log\frac{1}{\delta}+\log\log\frac{1}{\delta} $}{   
             $N^l_1 \leftarrow N^u_1$,
            
             $N^u_1 \leftarrow 2*N^u_1 $.
        }
        \tcc{perform bisection search}
        \While{ $N^u_1 - N^l_1 > \text{tolerance}$ }{

            $N_1^{\text{next}} \leftarrow \frac{N^l_1 + N^u_1}{2}$
            
            $g \leftarrow \text{Grad}\left(n_1 + \sum\limits_{b\ne 1} \max\lrset{0,{N}_b(n_1)}, N_1^{\text{next}}\right)$ 
            \tcp*{call Algorithm~\ref{alg:bisection_grad}}
        
        \If{
                $g \le 0 $
                }
                {
                
                    $N^l_1 \leftarrow N_1^{\text{next}}$
                }
                \Else{
                    $N^u_1 \leftarrow N_1^{\text{next}}$.
                }
        }
        \KwOut{$N_1$}
    }
\end{algorithm2e}

\begin{algorithm2e}
   \caption{\texttt{GradientComputationOracle}}
   \label{alg:bisection_grad}
   	\DontPrintSemicolon 
   	\tcc{Computes gradient of objective $\mathcal{O}_2$}
	\KwIn{${\bf N^o}$, $\hat{\mu}(t), \delta, N_1$}
    \SetKwProg{Fn}{}{:}{}
    { 
        \textbf{Initialization}{ $ \text{tolerance} = 10^{-3}$}
        
        \tcc{check if $N_1$ is a feasible point}
        \If{ $\min\limits_{b\ne 1} (N_1^o+N_1) \KL(\hat{\mu}_1(t), \hat{\mu}_b(t)) \leq \log\frac{1}{\delta}+\log\log\frac{1}{\delta} $}{   
             \KwOut{-1}
        }
        \tcc{perform bisection search}
        
        Compute $N_a(N_1)$ for all $a>1$ using bisection search

        Let $\mathcal{A} = \{a: N_a(N_1) > 0\}$ be the set of active arms

            For each $a\in \mathcal{A}$, define $x_{1,a}$ as 
            $$ x_{1,a} = \frac{(N^o_1+N_1) \hat{\mu}_1(t) + (N^o_a + N_a(N_1)) \hat{\mu}_a(t)}{N^o_1 + N^o_a + N_1 + N_a(N_1) }$$
            
        \KwOut{$1 - \sum\limits_{a\in \mathcal A}\frac{d(\hat{\mu}_1(t), x_{1,a})}{d(\hat{\mu}_a(t), x_{1,a})}$}
    }
\end{algorithm2e}

\subsection{Existence of $N_a(n_1)$ and computing them}\label{sec:Nan}
\paragraph{Existence of $N_a(n_1)$. } Since for a fixed $n_1$, the index is a monotonic function of $N_a$ (Lemma \ref{lem:convex.alloc}), its value is bounded by its limit when $N_a \rightarrow \infty$. At this point, $x_{1,a} \rightarrow \hat{\mu}_a(t)$, and the value of the index $Z_{1,a}$ is  $(N_1^o + n_1)\KL(\hat{\mu}_1(t), \hat{\mu}_a(t))$. If this quantity is less than $\log\frac{1}{\delta} + \log\log\frac{1}{\delta}$, then $N_a(n_1)$ doesn't exist. We next show how to compute $N_a(n_1)$ efficiently in  $O(\log\frac{1}{\epsilon})$ time. Thus, computing the unique optimizers of $\mathcal O_2$ can be done in  $O(K\log^2\frac{1}{\epsilon})$ computations.

\paragraph{Computing $N_a(n_1)$. } For a fixed $n_1$ for which the maximum value of the index exceeds $\log\frac{1}{\delta} + \log\log\frac{1}{\delta}$,  computing $N_a(n_1)$ can be done in $O(\log(1/\epsilon))$ using another bisection search. This follows from the fact that for a fixed $n_1 \ge 0$, $Z_{1,a}(\cdot)$ is monotonic in $N_a$ (see, Lemma \ref{lem:convex.alloc}). Thus, computing the unique optimizers of $\mathcal O_2$ can be done in  $O(K\log^2(1/\epsilon))$ computations. 

\section{A discussion on UCB-style algorithms for BAI}\label{app:KL-LUCB}
The BAI problems in the purely-online setting were first studied by \cite{even2006action,gabillon2012best, jamieson14, kalyanakrishnan2012pac, kaufmann2013information, karnin13}. While most of these works provide guarantees that hold for finite $\delta$, their bounds are sub-optimal for small values of $\delta$. In this section, we adapt the LUCB-style algorithms of \cite{kalyanakrishnan2012pac, kaufmann2013information} to our o-o framework and show numerically that even for practical values of $\delta$ (set to $0.05$ in the experiments), they under-perform compared to the batched o-o version of the track-and-stop algorithm analysed in the main text. In particular, we observe that the LUCB style algorithms require at least $10$ times as many samples as Algorithm~\ref{alg:tas_batch} before they stop. 

\subsection{Algorithm}
As in the tracking-based algorithms, LUCB-algorithm is a specification of a \emph{sampling} rule, a \emph{stopping} rule, and a \emph{recommendation} rule. We now describe these. For arm $a$, let $N^o_a$ denote the number of offline samples from arm $a$. Till time $t$ of online sampling, let $N_a(t)$ denote the number of online samples generated from that arm. Let $\hat{\mu}_a(t)$ denote the empirical mean of arm $a$ constructed using its $N^o_a + N_a(t)$ samples. 

At each time $t$ and for each arm $a$, the algorithm uses a UCB index to construct upper and lower confidence bounds for the means of each arms, call these $U_a(t)$ and $L_a(t)$. Examples of these used in \cite{kalyanakrishnan2012pac,kaufmann2013information} are given below. The following can be derived using Hoeffding's inequality:
\begin{align}\label{eq:lucb}
& U_a(t) := \hat{\mu}_a(t) + \sqrt{\frac{C(\tau_1 + t, \delta)}{2(N^o_a + N_a(t))}},\\
&L_a(t) := \hat{\mu}_a(t) - \sqrt{\frac{C(\tau_1 + t, \delta)}{2(N^o_a + N_a(t))}}.\nonumber
\end{align}
One can also arrive at the  following upper and lower bounds using concentration for $\KL$-divergences (as in \cite{cappe2013kullback} for SPEF and \cite{agrawal2021regret} for more general distributions).
\begin{align}\label{eq:kl_lucb}
&U_a(t) := \max\lrset{ x\in\Re :~ (N^o_a+N_a(t))d(\hat{\mu}_a(t), x) \le C\lrp{\tau_1 + t,\delta}  }, \\
&L_a(t) := \min\lrset{ x\in\Re :~ (N^o_a+N_a(t))d(\hat{\mu}_a(t), x) \le C\lrp{\tau_1 + t,\delta}  }, \nonumber
\end{align}
where $d(p, q)$ denotes the $\KL$ divergence between unique arms in $\cal S$ with means $p$ and $q$, respectively. The threshold function $C(n,\delta)$ is chosen so that for each arm $a$ for for all $t$, $\mu_a\in [L_a(t),U_a(t)]$ with probability at least $1-\frac{\delta}{K}$ (\cite[Lemma 4]{kaufmann2013information}).  As in the same paper, we use 
\[ C(n,\delta) =  \log\frac{K(\tau_1+t)^2}{\delta}+ \log\lrp{1 + \log\frac{K(\tau_1 + t)^2}{\delta}} \]
in our experiments to follow, when the arms have Bernoulli distributions. As a remark, we point out that one may use $C(\cdot, \cdot)$ from \cite[Proposition 15]{Kauffman_21} to get a lower dependence on the term  $\log(\tau_1 + t)$, which may translate to an improvement in the lower order terms in sample complexity of the LUCB algorithms.

Let $a^*(t)$ denote the empirically-best arm at time $t$. Then, a leader $l_t$, a challenger $c_t$, and the stopping statistic $B(t)$ are defined as below:
\begin{equation}\label{eq:lucb_stats} a^*(t):= \argmax\limits_{a} \hat{\mu}_a(t), ~~ l_t = a^*(t), ~~ c_t = \argmax\limits_{a\ne l_t} U_a(t), ~ \text{ and } ~ B(t) := U_{c_t}(t) - L_{l_t}(t) . \end{equation}

The algorithm proceeds by generating samples from two well-chosen arms (\emph{leader} $l_t$, and \emph{challenger} $c_t$) at each step. It stops as soon as the lower-confidence index for the empirically-best arm (or the leader) is greater that the maximum upper-confidence index for all the sub-optimal arms, \emph{i.e.,} when the optimal arm is well separated from the sub-optimal arms. On stopping, the algorithm outputs the empirically-best arm. The Algorithm \ref{alg:kl-lucb} formally describes the steps.

\begin{algorithm2e}
   \caption{\texttt{LUCB for o-o}}
   \label{alg:kl-lucb}
   	\DontPrintSemicolon 
   	\KwIn{Confidence level $\delta$, historic data $\lrset{N^o_a,\hat{\mu}^o_a}_{a=1}^K$}
    \SetKwProg{Fn}{}{:}{}
    { 
        \textbf{Initialization} Pull each arm once. \\
        Set  $t\leftarrow K, N_a(K) \leftarrow 1$.\\ 
        Update $\hat{\mu}(K)$\\
        Compute $U_a(K)$ and $L_a(K)$ for each arm $a$ using~\eqref{eq:lucb} or ~\eqref{eq:kl_lucb}\\
        Compute $l_K$, $c_K$, $B(K)$ using~\eqref{eq:lucb_stats} \\
        \While{ $B(t) \ge 0$ }{
        Sample arms $l_t$ and $c_t$ and set $t \leftarrow t + 2$\\
        Update $\hat{\mu}(t)$,  $N_{l_t}(t)$, $N_{u_t}(t)$.\\
        Compute $U_a(t), L_a(t)$ using~\eqref{eq:lucb} and~\eqref{eq:kl_lucb}.  \\
        Compute $l_t, c_t, B(t)$ using~\eqref{eq:lucb_stats}
        }
        \KwOut{$\argmax_{a\in[K]} \hat{\mu}_a(t)$}
    }
\end{algorithm2e}

\subsection{Numerical results}
We now present some numerical results for testing the performance of Algorithm~\ref{alg:kl-lucb}, and compare it to the Batched-Track and Stop (Algorithm \ref{alg:tas_batch}) from the main text.

We do two sets of experiments. In both the experiments, to keep the discussion simple, we consider Bernoulli arms. We now describe the setup of each.

\begin{figure}[tb]
\begin{center}
    \includegraphics[scale=0.35]{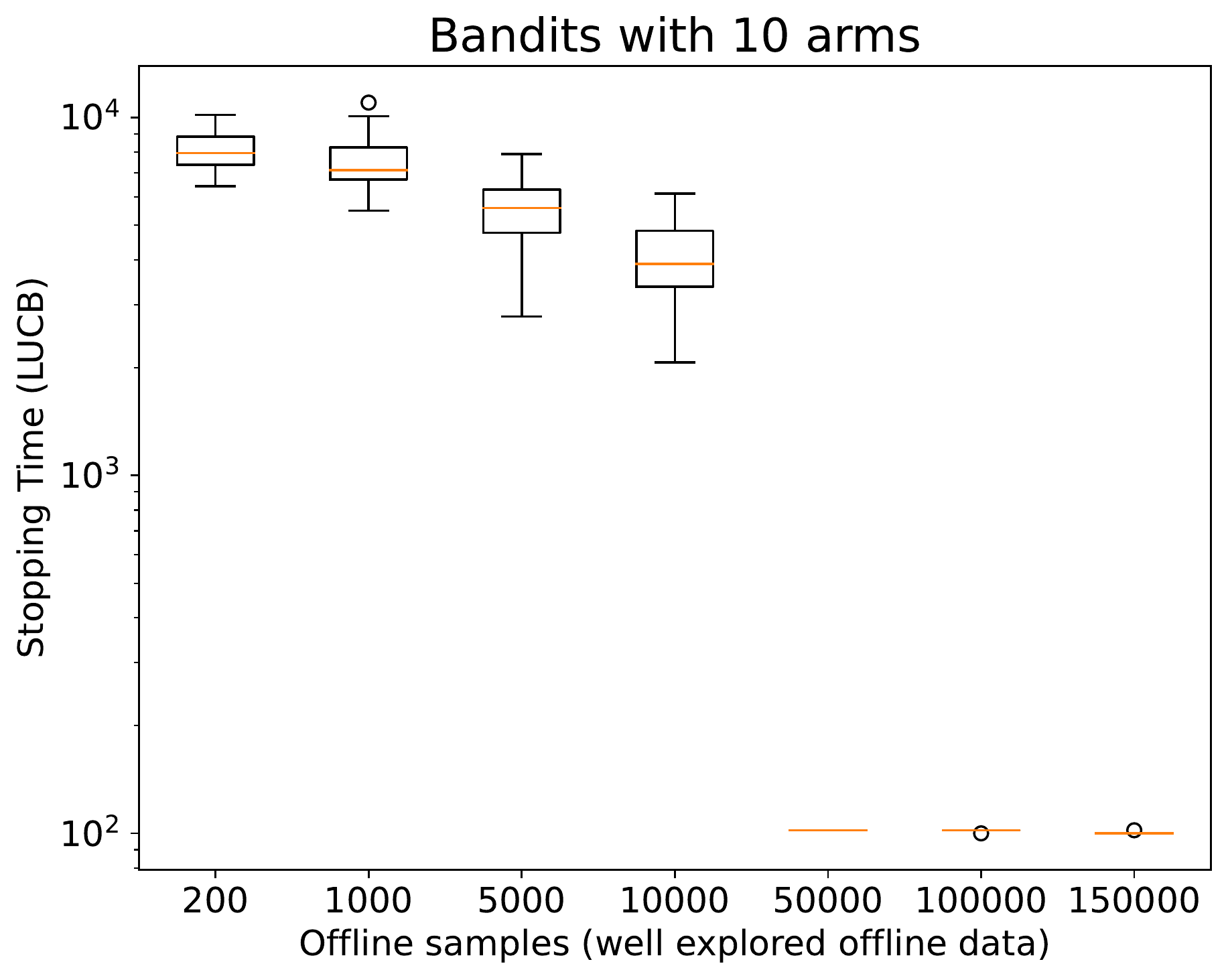}\hfill
    \includegraphics[scale=0.35]{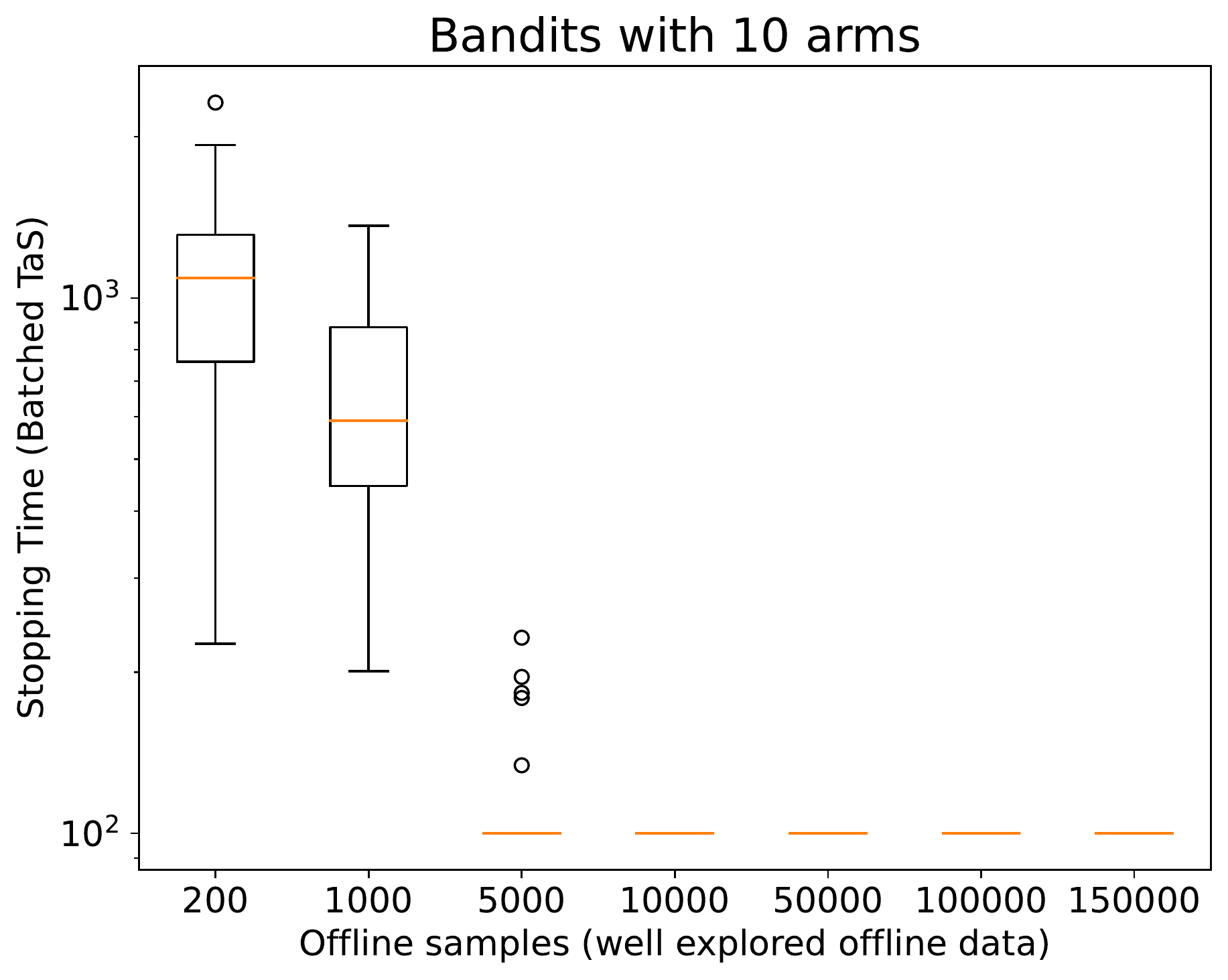}       \caption{Stopping time of LUCB (Algorithm~\ref{alg:kl-lucb}) and Batched-TaS  (Algorithm~\ref{alg:tas_batch}) with varying number of offline samples, when the each arm is uniformly sampled in the offline data. The rewards of the arms follow Bernoulli distributions. $\delta = 0.05$ for these experiments. Results are for $50$ independent trials. LUCB requires at least $10$ times more samples.}
    \label{fig:lucb_tas1}
\end{center}
\end{figure}

In the first experiment, we consider $10$-armed Bernoulli MAB with means \[\mu = (0.298, 0.437,  0.376, 0.651, 0.376,  0.322, 0.600, 0.643, 0.381, 0.8).\] 
All the arms are equally explored in the offline data, \emph{i.e.}, each arm is given equal number of samples in the offline sampling. 
$\delta$ is set to $0.05$ to see the performance of both the algorithms for practical ranges of $\delta$. We record the number of samples required by LUCB-algorithm with indexes constructed using~\eqref{eq:lucb}, when it has access to different amounts of offline data. We repeat this for the Batched-Track and Stop algorithm from the main text. The observations from $50$ independent runs of the experiments for both the algorithms are plotted in Figure~\ref{fig:lucb_tas1}. 

In the second experiment, for the same setup, we record the observations when the offline data is skewed. In particular, we generate offline data with no samples to the best arm, while all other arms are equally sampled. The observations are plotted in Figure~\ref{fig:lucb_tas2}.

\begin{figure}[tb]
\begin{center}
    \includegraphics[scale=0.35]{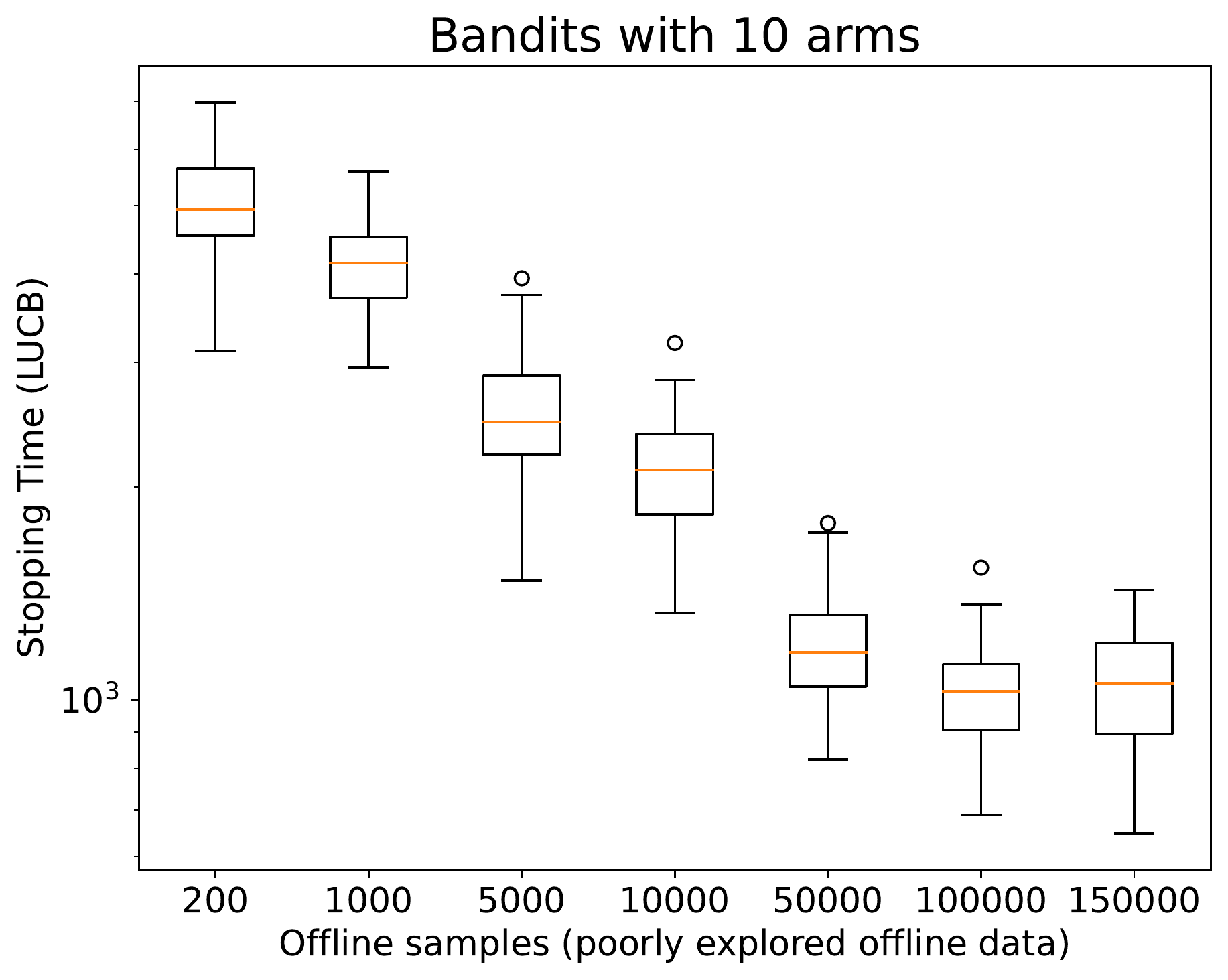}\hfill
    \includegraphics[scale=0.35]{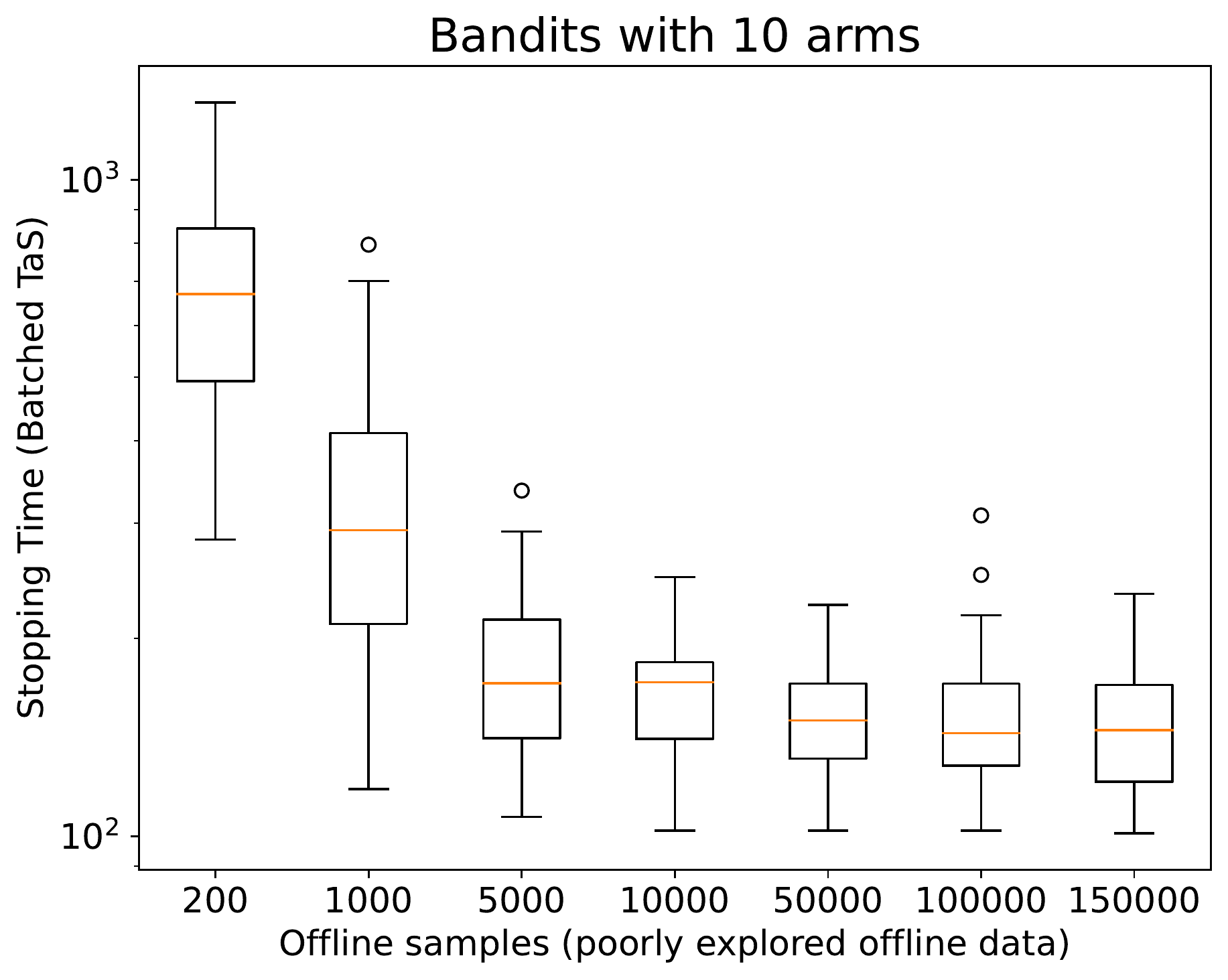}\caption{Stopping time of LUCB (Algorithm~\ref{alg:kl-lucb}) and  Batched-TaS (Algorithm~\ref{alg:tas_batch}) with varying number of offline samples, when the each arm except the best arm is uniformly sampled in the offline data, while the best arm doesn't have any samples. The rewards of the arms follow Bernoulli distributions. $\delta = 0.05$ for these experiments. Results are for $50$ independent trials. LUCB requires almost $10$ times more samples.}
    \label{fig:lucb_tas2}
\end{center}
\end{figure}

Two key takeaways from the experiments are the following:

a) Access to the offline data reduces the number of online samples generated by both the algorithms. 

b) When the offline data is not sufficient, the number of online samples required by the LUCB algorithm is significantly higher (at least $10$ times more) than that generated by the Batched-Track and Stop.

\section{Other Methods for BAI}\label{sec:artificial_replay}
\subsection{Artificial replay}

\cite{banerjee2022artificial} proposed a meta-algorithm called \emph{Artificial Replay} for learning in the offline-online (o-o) setting. Here, one chooses any online learning algorithm as a base algorithm. Suppose at time step $t$, the base algorithm recommends pulling an arm $a_t$. If a sample corresponding to $a_t$ is in the historic data, we remove it from the history and present it to the base algorithm. If $a_t$ is absent in the historic data, we pull that arm and get the reward. 

We now present a formal proof showing that such a strategy is sub-optimal, even if the base algorithm is such that it is an optimal algorithm for the online setting. Consider a 2-armed bandit instance, where the arm rewards follow a Gaussian distribution. In particular, rewards of arm $1$  are sampled from $\mathcal{N}(\mu_1, 1)$ and rewards of arm $2$ are sampled from $\mathcal{N}(\mu_2, 1)$. Without loss of generality, assume $\mu_1 = \mu_2 + \Delta$ for some $\Delta > 0$. Suppose the offline data is such that we have infinite samples from arm $1$, and $0$ samples from arm $2$ (\emph{i.e.,} $N^o_1 = \infty, N^o_2 = 0$). As our choice of base algorithm, we use any optimal BAI algorithm proposed for the purely online setting (see \cite{garivier2016optimal} for one such algorithm). Finally, we consider the setting where $\delta \to 0.$ 

Since our base algorithm is optimal for the online setting, it pulls arms in the  proportions suggested by the following lower bound optimization problem~\citep{garivier2016optimal}
 \[
 \min \sum_{a \in [K]}
N_a ~\mbox{  such that  } ~ {\bf N} \in A_{\delta}(\mu),
\]
where  $A_{\delta}(\nu)$ is the non-negative set of 
${\bf N} \in \Re^K_+$ that satisfy
\begin{equation*} 
\inf_{x} N_{1} \KL(\nu_1, x) +
N_{a} \KL(\nu_a, x) \geq \log  \frac{1}{2.4\delta},~ \forall a\ne 1.
\end{equation*}
The solution to this problem has an analytical expression: $N_1^* = N_2^* = \frac{2}{\Delta^2}\log\left({\frac{1}{2.4\delta}}\right).$
This shows that the sample complexity of artificial replay algorithm is $\frac{2}{\Delta^2}\log\left({\frac{1}{2.4\delta}}\right)$.

Now consider our algorithm for o-o setting. It pulls arms in the proportions suggested by the following lower bound optimization problem
\[
 \min_{{\bf N} \succeq 0} \sum_{a \in [K]}
N_a ~\mbox{  such that  }~ {\bf N^o + N} \in A_{\delta}(\mu).
\]
This solution to this again has a closed form expression: $N_1^* = 0, N_2^* = \frac{1}{\Delta^2}\log\left({\frac{1}{2.4\delta}}\right)$.  This shows there is a factor of $2$ difference in sample complexity between our algorithm and artificial-replay with the best base algorithm.

The above argument shows that even if the base algorithm being used is optimal for purely online setting, artificial replay has sub-optimal sample complexity in the o-o setting. The question that one could potentially ask is : ``are there any other base algorithms, which when combined with artificial-replay, lead to optimal sample complexity in the o-o setting?'' We conjecture that for artificial replay to achieve optimal sample complexity, the base algorithm needs to be chosen based on the available offline data. There is no single base algorithm that will work for all problem instances and all offline policies. 

\subsection{Thompson Sampling (TS)} 
One could consider designing TS style algorithms for BAI in the o-o setting. Here, offline data could be used to set the priors appropriately. While this is an interesting direction, we note that designing optimal TS style algorithms for BAI is an open problem even in the purely online setting~\cite{russo2016simple, jourdan2022top}.

\section{Offline-online regret-minimization problem}\label{app:regret}
In this section, we discuss the regret-minimization MAB problem in the o-o framework. Here, as in the BAI framework, the algorithm is presented with $K$ unknown probability distributions or {\em $K$ arms}. Each time the online algorithm pulls an arm, an independent sample from the corresponding distribution is generated. The samples generated by the algorithm are viewed as rewards. In addition, it also has access to samples generated from the same set of $K$ distributions using some adaptive policy that is independent of the algorithm. We refer to these additional samples as `offline data' as they are not generated by the algorithm, and are available with it as side information before it starts the online sampling. The goal of the algorithm is to maximize the total expected reward in the online trials, while making the best use of the offline data. 

To formally describe the setup, let us introduce some notation. For $t\in\mathbb N$ and $a\in [K] := \lrset{1, \dots, K}$, let $N^o_a(t)$ denote the total offline samples from arm $a$ available with the algorithm till time $t$, and let $N_a(t)$ denote the total online samples that the algorithm has generated from arm $a$ till time $t$. The goal of maximizing average cumulative regret can be formulated in terms of minimizing the average regret (to be defined in the following paragraphs). For simplicity of notation, we assume that the optimal arm in $\mu$ is arm $a$. Let $\Delta_a := \mu_1 - \mu_a$ denote the sub-optimality gap of arm $a$. It is $0$ for the optimal arm. Then for $T\in\mathbb N$,
\begin{equation}
    \E\lrs{R(T)} := \sum\limits_{a} \E\lrs{N_a(T)} \Delta_a.
\end{equation}

Observe that the algorithm's regret is defined only with respect to the samples generated in the online sampling. However, the number of online samples generated may depend on the offline data available. Hence, the expectation above is with respect to all the randomness, including that of the offline sampling algorithm, offline data, online algorithm, and online data. In the sequel, to get a handle on the expected regret, we bound $\E\lrs{N_a(T)}$ for the sub-optimal arms $a$.

\subsection{Lower bound}
We now derive a lower bound on the expected regret suffered by a reasonably-good (consistent) algorithm that in addition to online sampling, has access to offline data as described in the previous section.

For any bandit instance $\kappa\in \mathcal S^K$, let $a^*(\kappa)$ denote the optimal arm in $\kappa$. Let $\mu = \lrp{\mu_1, \dots, \mu_K}$ be the given bandit instance with $a^*( \mu) = 1$. 

\noindent\textbf{Consistent algorithms. }We are interested in algorithms with access to side information in terms of the offline data, that suffer small regret in the online trials, i.e., 
\[ \forall \kappa\in \mathcal S^K, ~ \forall a\ne a^*(\kappa), \quad \E_\kappa\lrs{N_a(T)} = o(T^\alpha), \quad \forall \alpha > 0,\]
where the expectation is with respect to all the randomness in the system including offline policy, offline data, online algorithm, and online data.

\begin{proposition}
For $\mu\in\mathcal S^K$ with $a^*(\mu) = 1$, any consistent algorithm satisfies 
\[ \liminf\limits_{T\rightarrow \infty} \frac{\E\lrs{N^o_a(T)} + \E\lrs{N_a(T)}}{\log T} \ge \frac{1}{\KL(\mu_a, \mu_1)}. \]
\end{proposition}

\begin{proof}
Consider an alternative instance $\nu = \lrp{\mu_1, \nu_2, \mu_3, \dots, \mu_K}$. In $\nu$, only arm $2$ has a different distribution from $\mu$. Without loss of generality, we assume that $\nu$ has arm $2$ as the optimal arm. Suppose the algorithm generated $N_a(T)$ online samples in $T$ online trials, and also observed $N^o_a(T)$ offline samples from arm $a$. In particular, here $\sum_a N_a(T) = T$. Then, from data processing inequality, we have 
\begin{equation}\label{eq:dp} 
\lrp{\E\lrs{N^o_2(T)} + \E\lrs{N_2(T)}} \KL({\mu}_2, {\nu}_2 ) \ge d(\mathbb{P}_\mu(E_T), \mathbb{P}_\nu(E_T) ), 
\end{equation}
for any event $E_T$ that belongs to $\mathcal F_T$. Here $\mathbb{P}_\mu$ and $\mathbb{P}_\nu$ denote the probabilities when samples are from bandit instance $\mu$ and $\nu$, respectively. We refer the reader to Appendix~\ref{app:LB} and \cite[Chapter 4]{lattimore2020bandit} for details of the probability space and the filtration $\mathcal F_T$.

Consider the event  
\[E_T = \lrset{ N_1(T) \le T-\sqrt{T}}.\] 

Then,
\[ \mathbb{P}_\mu(E_T) = \mathbb{P}_\mu\lrp{T-N_1(T) \ge \sqrt{T}} = \mathbb{P}_\mu\lrp{ \sum\limits_{a\ne 1} N_a(T) \ge \sqrt{T} } \le \frac{\sum\limits_{a\ne 1} \E_\mu{N_a(T)}}{\sqrt{T}} \xrightarrow{T\rightarrow \infty} 0. \]

As earlier, here again the expectation $\E_\mu$ is the expectation with respect to all the randomness in the system, when the interaction is with the bandit instance $\mu$. The above limit follows because the algorithm is consistent. Similarly, 
\[ 
    \mathbb{P}_\nu(E^c_T) = \mathbb{P}_\nu\lrp{N_1(T) \ge T-\sqrt{T}} \le \frac{ \sum\limits_{a\ne 2}\E_\nu{N_a(T)}}{T-\sqrt{T}} \xrightarrow{T\rightarrow \infty} 0.
\]

Next, consider

\begin{align*}
    &\lim\limits_{T\rightarrow \infty }\frac{d(\mathbb{P}_\mu(E_T), \mathbb{P}_\nu(E_T) )}{\log T} \\
    &= \lim\limits_{T\rightarrow \infty }\frac{1}{\log T} \lrp{ \mathbb{P}_\mu(E_T)\log\frac{\mathbb{P}_\mu(E_T)}{\mathbb{P}_\nu(E_T)} + (1-\mathbb{P}_\mu(E_T))\log\frac{1-\mathbb{P}_\mu(E_T)}{1-\mathbb{P}_\nu(E_T)}  } \\
    &= \lim\limits_{T\rightarrow \infty } \frac{1}{\log T}\lrp{\log \frac{1}{\mathbb{P}_\nu(E^c_T)} }\\
    &\ge \lim\limits_{T\rightarrow \infty } \frac{1}{\log T}\lrp{\log \frac{T-\sqrt{T}}{\sum\limits_{a\ne 2}\E_\nu(N_a(T))} }\\
    &= \lim\limits_{T\rightarrow \infty}\lrp{1 + \frac{1}{\log T}\log\lrp{1-\frac{1}{\sqrt{T}}} - \frac{1}{\log T}\log\lrp{\sum\limits_{a\ne 2} \E_\nu (N_a(T)) } }\\
    &=1.
\end{align*}

Substituting the above in~\eqref{eq:dp}, we get 
\[ \lim\limits_{T\rightarrow \infty} \frac{\lrp{\E{N^o_2(T)} + \E{N_2(T)}} \KL({\mu}_2, {\nu}_2 ) }{\log T} \ge \lim \limits_{T\rightarrow \infty} \frac{d(\mathbb{P}_\mu(E_T), \mathbb{P}_\nu(E_T) )}{\log T} \ge 1.  \] 

That is, 
\[ \lim\limits_{T\rightarrow\infty} \frac{\E N^o_2(T) + \E N_2(T)}{\log T} \ge \frac{1}{\inf\limits_{\nu_2 \ge \mu_1}\KL(\mu_2, \nu_2) } = \frac{1}{\KL(\mu_2, \mu_1)}.\]
\end{proof}

Observe that the above result highlights a possible gain of having offline data available as the lower bound on the online samples in the o-o setting is at most that in the purely-online setting. 

\subsection{Algorithm}
We now discuss the algorithm \texttt{o-o UCB}, a natural adaptation of the purely-online algorithm \texttt{UCB1} of \cite{auer2002finite} for the regret-minimization MAB problem to the o-o framework. We refer the reader to \cite{auer2002finite, garivier2011kl, cappe2013kullback, agrawal2021regret} for optimal UCB algorithms for the purely online problem under different assumptions. For simplicity, in this section, we assume that the arms have Gaussian distributions with unit variance, i.e., $\mathcal S$ is the collection of unit-variance Gaussian distributions.

The algorithm has access to $N^o_a$ offline samples from each arm. At each time, the algorithm computes an index for each arm using the available samples, and pulls the arm with the maximum index. If the arm chosen is  sub-optimal, it incurs a regret. For each arm $a\in [K]$, let $N_a(t)$ denote the online samples generated from that arm till time $t$. 
Next, let $\hat{\mu}_a(t)$ denote the empirical mean constructed using $N_a(t) + {N}^o_a$ samples from arm $a$. The UCB index for arm $a$ at time $t$, denoted by $U_a(t)$, is given by 
\begin{align*} \label{eq:ucb_index}
    U_a(t) := \hat{\mu}_a(t) + \sqrt{\frac{4 \log t}{N^o_a + N_a(t)}}.
\end{align*}

Also, recall that here, $\sum_a N_a(t) = t$ denotes the total online samples. Then, we have the following guarantee on the expected number of pulls of a sub-optimal arm (and hence, on regret) of the proposed algorithm.

\begin{theorem}
For $\mu\in\mathcal S^K$ with $a^*(\mu) = 1$, \texttt{o-o UCB} has the following bound on the number of pulls of sub-optimal arm $a$ till time $T$:
\[ \E\lrs{N_a(T) }  \le  \E\lrs{\max\lrset{1, \frac{8\log T}{\Delta^2_a} -  {N}^o_a + 1}} + o(\log T). \]
\end{theorem}
The analysis follows from the proof in  \cite[Theorem 1]{auer2002finite}. In particular, by the choice of index for \texttt{o-o UCB}, using Hoeffding's inequality, one can see that the probabilities of the events in \cite[Equations (7), (8)]{auer2002finite} are at bounded by $\frac{1}{t^4}$. Choosing
\[ \ell = \max\lrset{1, \frac{8\log T}{\Delta^2_a} + 1 -  N^o_a  } \]
in their analysis, we get the bound.

We point out that the algorithm can be extended beyond Gaussian to SPEF using analysis of \cite{cappe2013kullback} and using the index proposed therein.

\section{Additional experimental results}
\label{sec:additional_exps}

In this section, we present empirical evidence showing the importance of availability of the offline data over purely offline learning. We generated offline data from two policies: (a) a policy which uniformly samples all the arms (well explored offline data), and (b) a policy that uniformly samples all the arms except the best arm. The latter policy doesn't pull the best arm at all (poorly explored offline data). We run the o-o experiments for $3-$ as well as $10-$armed bandit for each of these offline policies. In each experiment, the rewards of the arms follow Gaussian distributions with variance $1$. The mean reward of all the arms, except that of the best arm, is $0.4$. The mean reward of the best arm is $0.5$. Figure~\ref{fig:online_offline_appendix} presents the expected stopping time of our algorithm, as we vary the amount of offline data, $\delta$ is set to $10^{-3}$. 

In Section~\ref{app:KL-LUCB}, we present additional experimental results where the arms follow  Bernoulli distributions (Figures~\ref{fig:lucb_tas1} and~\ref{fig:lucb_tas2}). The experiments in that section are for a larger value of $\delta$ ($=0.05$) and demonstrate that the gain of access to the offline samples is seen even for relatively large values of $\delta$. 

Here are two important takeaways from these results: 

(a) The number of online samples decreases as the amount of offline data increases. 

(b) Even if the offline data is of poor quality (\emph{e.g.,} data generated by the second policy above), it  helps reduce the number of online rounds. Note that learning algorithms that solely rely on offline data wouldn't have worked in such cases.

\begin{figure}[htb]
\begin{center}
    \includegraphics[scale=0.35]{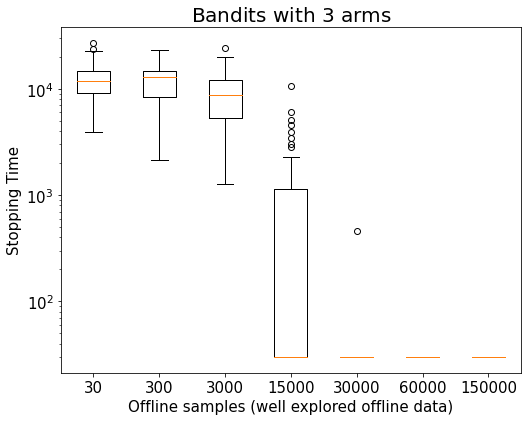}\hfill
    \includegraphics[scale=0.35]{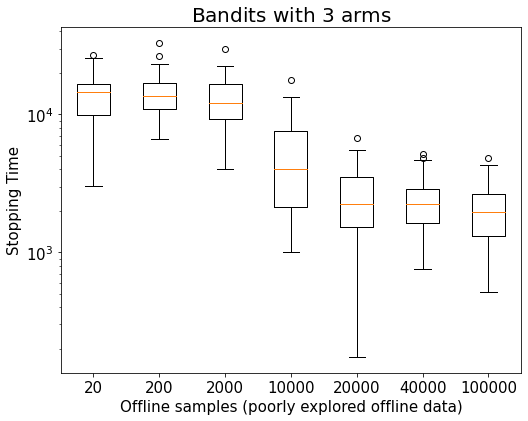}
    \includegraphics[scale=0.35]{bai_10_arms_corrected_well_explored.png}\hfill
    \includegraphics[scale=0.35]{bai_10_arms_corrected_poorly_explored.png}
      \caption{Stopping time of Algorithm~\ref{alg:tas_batch} with varying number of offline samples. The offline samples are collected using two policies that are described in Section~\ref{sec:exps}. The first 2 plots are for MAB instance with $3$ arms and the last 2 are for $10$ arms. The rewards of the arms follows Gaussian distribution with variance $1$. The mean reward of all the arms, except the best arm is $0.4$. The mean reward of the best arm is $0.5$. We chose $\delta = 10^{-3}$ for these experiments. The quantile plots are obtained by repeating each experiment $50$ times.}
    \label{fig:online_offline_appendix}
\end{center}
\end{figure}

\vfill

\end{document}